\documentclass[10pt,twocolumn,letterpaper]{article}

 \usepackage[pagenumbers]{cvpr} % To force page numbers, e.g. for an arXiv version
\usepackage{mathtools}
\usepackage{amsthm}
\usepackage{amsmath}
\usepackage{multirow}
\usepackage{dsfont}
\usepackage{colortbl}
\newtheorem{prop}{Proposition}
\usepackage{comment}
\usepackage{booktabs}

\DeclareMathOperator{\argmaxG}{arg\,max}

\definecolor{cvprblue}{rgb}{0.21,0.49,0.74}
\usepackage[pagebackref,breaklinks,colorlinks,citecolor=cvprblue]{hyperref}
\usepackage[accsupp]{axessibility}
\usepackage{colortbl}
\usepackage{booktabs}

\usepackage{algorithm}
\usepackage{algorithmic}
\usepackage{comment}
\def\paperID{14124} % *** Enter the Paper ID here
\def\confName{CVPR}
\def\confYear{2024}

\title{CDMAD: Class-Distribution-Mismatch-Aware Debiasing \\for
Class-Imbalanced Semi-Supervised Learning}

\author{%
  Hyuck Lee \ \ \ \ \ \ \ \ \ ~~~ Heeyoung Kim\\
  Department of Industrial and Systems Engineering, KAIST\\
  Daejeon 34141, Republic of Korea \\
  \texttt{\{dlgur0921, heeyoungkim\}@kaist.ac.kr} \\
  }

\begin{document}
\maketitle
\begin{abstract}
Pseudo-label-based semi-supervised learning (SSL) algorithms trained on a class-imbalanced set face two cascading challenges: 1) Classifiers tend to be biased towards majority classes, and 2) Biased pseudo-labels are used for training. It is difficult to appropriately re-balance the classifiers in SSL because the class distribution of an unlabeled set is often unknown and could be mismatched with that of a labeled set. We propose a novel class-imbalanced SSL algorithm called class-distribution-mismatch-aware debiasing (CDMAD). For each iteration of training, CDMAD first assesses the classifier's biased degree towards each class by calculating the logits on an image without any patterns (e.g., solid color image), which can be considered irrelevant to the training set. CDMAD then refines biased pseudo-labels of the base SSL algorithm by ensuring the classifier's neutrality. 
CDMAD uses these refined pseudo-labels during the training of the base SSL algorithm to improve the quality of the representations. In the test phase, CDMAD similarly refines biased class predictions on test samples. CDMAD can be seen as an extension of post-hoc logit adjustment to address a challenge of incorporating the unknown class distribution of the unlabeled set for re-balancing the biased classifier under class distribution mismatch. CDMAD ensures Fisher consistency for the balanced error. Extensive experiments verify the effectiveness of CDMAD.
\end{abstract}    
\vspace{-0.2in}
\section{Introduction}
\label{intro}
Classifiers trained on a class-imbalanced set suffer from being biased toward the majority classes. 
Under semi-supervised learning (SSL) settings, classifiers of pseudo-label-based algorithms tend to be further biased because of the use of biased pseudo-labels for training. The use of biased pseudo-labels also decreases the quality of representations. 
This problem becomes more serious when the class distributions of the labeled and unlabeled sets differ significantly. In fact, recent SSL algorithms, such as ReMixMatch \cite{berthelot2019remixmatch} and CoMatch \cite{li2021comatch}, rely on the assumption that the class distribution of the unlabeled set is the same as that of the labeled set and cannot consider a potential class distribution mismatch between the labeled and unlabeled sets.

Recently, many class imbalanced SSL (CISSL) algorithms \cite{wei2021crest,lee2021abc,fan2022CoSSL,NEURIPS2020_a7968b43,oh2022daso,lai2022smoothed} have been proposed. However, \citet{wei2021crest,lee2021abc,fan2022CoSSL} assumed that the class distribution of the unlabeled set is known and the same as that of the labeled set, although the class distribution of the unlabeled set can be unknown in practice (e.g., STL-10 \cite{coates2011analysis}) and training sets comprising labeled and unlabeled sets collected from different periods are likely to have a class distribution mismatch. 
\citet{NEURIPS2020_a7968b43,oh2022daso,lai2022smoothed} did not make an assumption of the same class distributions for labeled and unlabeled sets in the main training stage.
However, after the main training stage, they additionally used the re-balancing technique of Classifier Re-training (cRT) \cite{kang2019decoupling} or post-hoc logit-adjustment (LA) \cite{menon2020long}, which were proposed for fully supervised class-imbalanced learning. When using cRT for CISSL, there are disadvantages that the classifier cannot be learned interactively with representations, and only the labeled set is used for training the classifier \cite{lee2021abc}. Using LA for CISSL may not re-balance the classifier to an appropriate degree when the class distribution of the unlabeled set is unknown and differs from that of the labeled set, because LA can not consider the unknown class distribution of the unlabeled set. 

We propose a CISSL algorithm, class-distribution-mismatch-aware debiasing (CDMAD), which effectively mitigates class imbalance in SSL even under severe class distribution mismatch between labeled and unlabeled sets. The key idea of CDMAD is to consider the classifier's biased degree towards each class for refining both the biased pseudo-labels of the base SSL algorithm and class predictions on test samples. 
To measure the classifier's biased degree, we utilize the class prediction on an input that is reasonably assumed to be irrelevant to the training set.  

\begin{figure}
	\begin{center}
		\begin{tabular}{cc}
			 \hspace{-0.175in}\includegraphics[width=3.6cm, height=3.6cm]{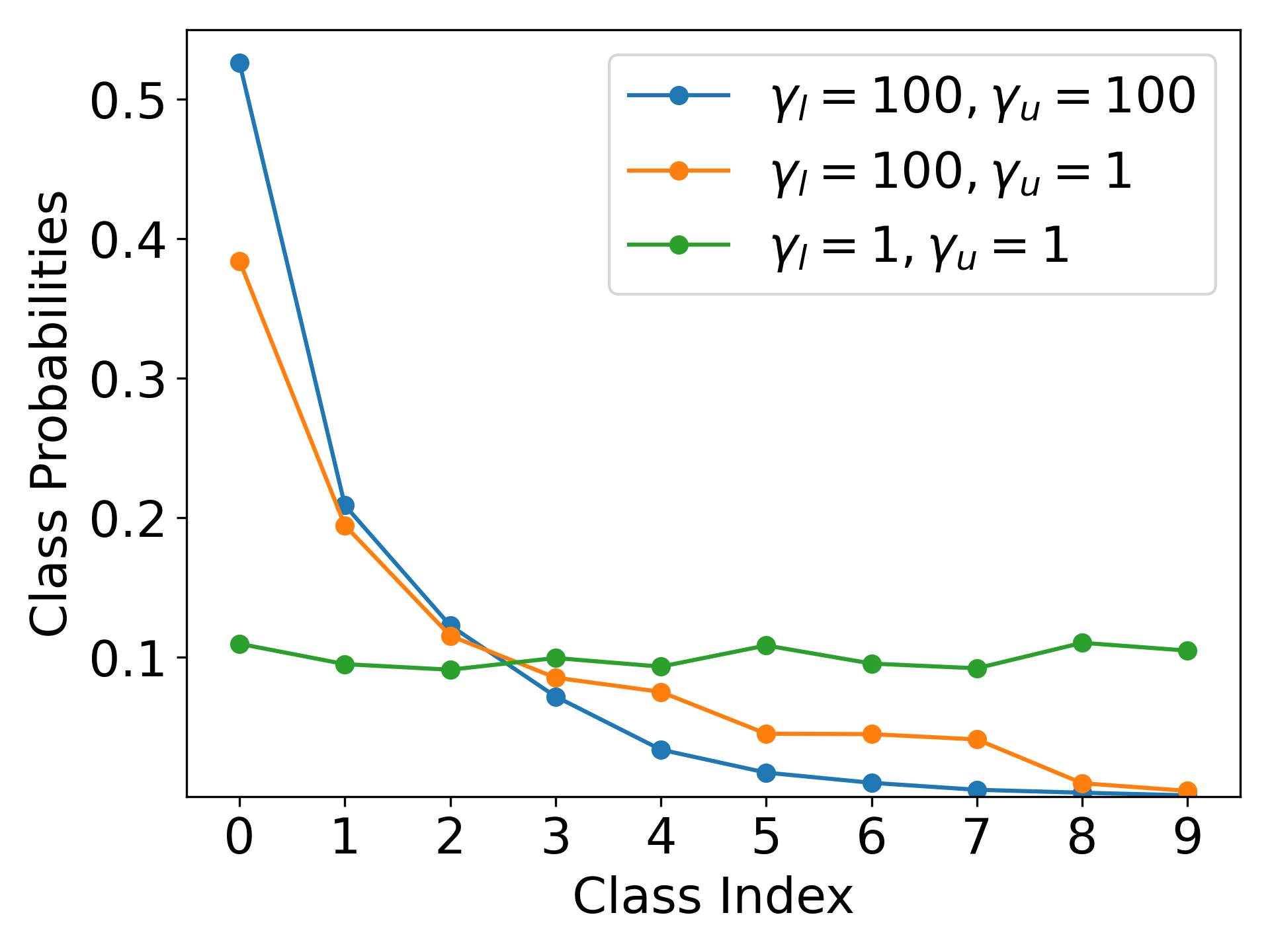}& \includegraphics[width=3.6cm, height=3.6cm]{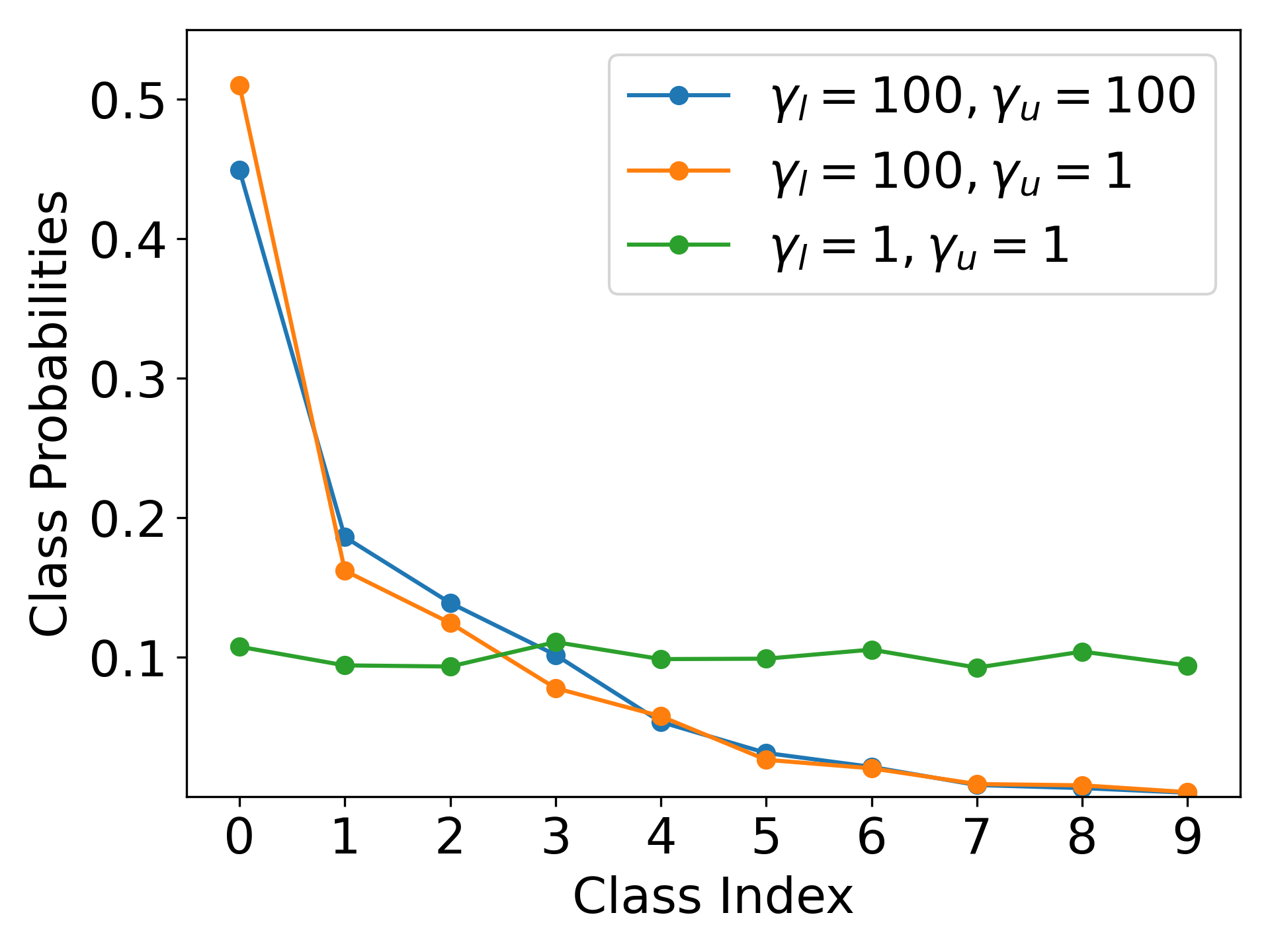} \\\ 
			  \footnotesize{(a) FixMatch}  &\hspace{-0.15in}
			\footnotesize{(b) ReMixMatch}
		\end{tabular}
	\end{center}
 \vspace{-0.2in}
	\caption{Class probabilities on an image without any patterns.}
	\label{prior}
 \vspace{-0.25in}
\end{figure}

In general, a trained classifier predicts a class of a new sample based on the learned features. Therefore, for an image irrelevant to the learned features, the predicted class probabilities are expected to be uniform across classes. However, this may not be true when the training set is class-imbalanced because the classifier tends to be biased towards the majority classes. \cref{prior} illustrates the class probabilities predicted on an image without any patterns (white image) using FixMatch \cite{sohn2020fixmatch} and ReMixMatch \cite{berthelot2019remixmatch}, base SSL algorithms of the recent CISSL studies, trained on CIFAR-10 under $\gamma_l=\gamma_u=1$ (class-balanced set), $\gamma_l=100$ and $\gamma_u=1$ (class-imbalanced set%, severe class distribution mismatch between the labeled and unlabeled sets
), and $\gamma_l=\gamma_u=100$ (class-imbalanced set), where $\gamma_l$ and $\gamma_u$ denote the class imbalanced ratios for the labeled and unlabeled sets (formally defined in \cref{setup}), respectively. The classifiers trained on the class-imbalanced sets produced highly nonuniform class probabilities for the white image, whereas they produced nearly uniform class probabilities for the same input when trained on the class-balanced set. Here, it may be reasonable to assume that the solid color image does not have the features learned from the training set. Then, the class probabilities for the solid color image can be thought of as predicted based solely on the classifier's biased degree towards each class, regardless of the learned features. 

Motivated by the above finding, CDMAD measures the classifier's biased degree by calculating logits on a solid color image for each iteration of training. Then, CDMAD refines the biased pseudo-labels of a base SSL algorithm by adjusting for the measured bias of the classifier in the logits for unlabeled samples. The refined pseudo-labels are used to train the base SSL algorithm, which leads to the mitigation of class imbalance and improved quality of the representations. After training is completed, CDMAD similarly refines the biased class predictions on test samples by adjusting for the measured bias of the classifier in the logits for test samples. CDMAD can appropriately re-balance the classifier even under severe class distribution mismatch between labeled and unlabeled sets because the class distributions of both labeled and unlabeled sets can be implicitly considered when measuring the classifier's biased degree. In \cref{LA}, we analyze that CDMAD can be viewed as an extension of LA, incorporating awareness of class distribution mismatch. Similar to LA, CDMAD is Fisher consistent for minimizing the balanced error \citep{menon2013statistical}.

Experimental results on four benchmark datasets verify that CDMAD outperforms baseline CISSL algorithms in both scenarios where the class distributions of the labeled and unlabeled sets either match or mismatch. Furthermore, through qualitative analysis and an ablation study, we demonstrate the effectiveness of each component of CDMAD. Unlike previous CISSL studies, CDMAD does not require additional parameters or training stages in comparison to the base SSL algorithm. Additionally, it can be implemented by simply adding a few lines of code into the existing code of the base SSL algorithms as presented in Appendix A. The code for the CDMAD is available at https://github.com/LeeHyuck/CDMAD.
\vspace{-0.05in}
\section{Related Works}
%\vspace{-0.025in}
%In this section, we review CISSL studies. Balanced SSL and fully supervised class-imbalanced learning (CIL) studies are reviewed in Appendix B. Base SSL algorithms (FixMatch \citep{sohn2020fixmatch} and ReMixMatch \citep{berthelot2019remixmatch}) of the proposed algorithm and LA are described in \cref{fixremix} and \cref{LA}, respectively.
%\textbf{Class-imbalanced semi-supervised learning (CISSL)} 
%CISSL algorithms aim to mitigate class imbalance when SSL algorithms are trained on a class-imbalanced set. %has started to get a lot of attention with theoretical founding of \citep{yang2020rethinking} that SSL enables classifiers to learn more accurate decision boundaries in class-imbalanced settings. 
CReST \citep{wei2021crest} uses unlabeled samples predicted as the minority classes more frequently than those predicted as the majority classes for iterative self-training. ABC \citep{lee2021abc} and CoSSL \citep{fan2022CoSSL} use an auxiliary classifier and train the classifier to be balanced. CoSSL generates pseudo-labels for base SSL algorithms using the balanced classifier. These studies assume that the class distribution of the unlabeled set is known and same as that of the labeled set. DARP \citep{NEURIPS2020_a7968b43} and DASO \citep{oh2022daso} refine biased pseudo-labels by iteratively solving a convex optimization problem and blending semantic pseudo-labels and linear pseudo-labels, respectively. SAW \citep{lai2022smoothed} mitigates class imbalance using smoothed reweighting based on the number of pseudo-labels belonging to each class. %These studies do not assume that the unlabeled set has the same known class distribution as the labeled set for the main training stage. However, 
These studies additionally use CIL techniques, such as cRT \citep{kang2019decoupling} and LA \citep{menon2020long}, after the main training stage.
%These studies additionally use CIL techniques after the main training stage such as cRT \citep{kang2019decoupling}, which uses only labeled set for training classifier and LA \citep{menon2020long}, which implicitly assume that the class distributions of the labeled and unlabeled sets. %, to further mitigate the class-imbalance.
 %These CISSL algorithms may face a challenge in using unlabeled set for training classifier or appropriately mitigating class imbalance when the class distribution of the unlabeled set significantly differs from that of the labeled set. 
 Adsh \citep{guo2022class} and InPL \citep{yu2023inpl} use pseudo-labels based on class-dependent confidence thresholds and energy score threshold, respectively. DebiasPL \citep{wang2022debiased} debiases pseudo-labels by mitigating the classifier response bias based on counterfactual reasoning. UDAL \citep{lazarow2023unifying} unifies distribution alignment technique \citep{berthelot2019remixmatch} and logit-adjusted loss \citep{menon2020long} to progressively mitigate class-imbalance. L2AC \citep{wang2023imbalanced} trains a bias adaptive classifier composed of a bias attractor and a linear classifier with bi-level optimization. ACR \cite{wei2023towards} dynamically refines pseudo-labels using an adaptive consistency regularizer that estimates the true class distribution of unlabeled set.
\definecolor{Gray}{gray}{0.9}
\vspace{-0.075in}
\section{Methodology}
\label{method}
\vspace{-0.025in}
\subsection{Problem setup}
\label{setup}

Suppose that we have a training set with labeled set $\mathcal{X}=\left\lbrace\left(x_{n},y_{n}\right):  n\in\left(1,\ldots,N\right)\right\rbrace$ and unlabeled set $\mathcal{U}= \left\lbrace\left(u_{m}\right):  m\in\left(1,\ldots, M\right)\right\rbrace$, where $x_{n}\in\mathbb{R}^{d}$ and $y_{n}\in [C] =\left\lbrace1,\ldots, C\right\rbrace$ denote the $n$th labeled sample and corresponding label, respectively, and $u_{m}\in\mathbb{R}^{d}$ denotes the $m$th unlabeled sample. %Generally $N < M$ because it is laborious to annotate labels of a lot of samples.
We denote the number of labeled and unlabeled samples of class $c$ as $N_{c}$ and $M_{c}$, respectively, i.e., $\sum_{c=1}^{C}N_{c}=N$ and $\sum_{c=1}^{C}M_{c}=M$, where $M_{c}$ is challenging to know in a realistic scenario. The $C$ classes are sorted in descending order according to the cardinality of labeled samples, i.e., $N_{1}\geq\cdots\geq N_{C}$. The ratio of the class imbalance of labeled and unlabeled sets are $\gamma_l = \frac{N_{1}}{N_{C}}$ and $\gamma_u = \frac{M_{1}}{M_{C}}$, respectively, where $\gamma_l\gg 1$ or $\gamma_u\gg 1$ in the class-imbalanced training set. When $M_{c}$ is unknown, $\gamma_u$ will be also unknown and can differ from $\gamma_l$. That is, the class distribution of the unlabeled set can be mismatched with that of the labeled set. 
For each iteration of training, we sample minibatches $\mathcal{MX} = \left\lbrace\left(x^m_{b},y^m_{b}\right):  b\in \left(1,\ldots, B\right)\right\rbrace \subset \mathcal{X}$ and $\mathcal{MU}= \left\lbrace\left(u^m_{b}\right):  b\in\left(1,\ldots, \mu B\right)\right\rbrace \subset \mathcal{U}$ from the training set, where $B$ denotes the minibatch size and $\mu$ denotes the relative size of $\mathcal{MU}$ to $\mathcal{MX}$. Using $\mathcal{MX}$ and $\mathcal{MU}$ for training, we aim to learn a classifier $f_{\theta}:\mathbb{R}^{d}\rightarrow\left\lbrace1,\ldots,C\right\rbrace$ that effectively classifies samples in a test set $\mathcal{X}^{test} = \left\lbrace\left(x^{test}_{k},y^{test}_{k}\right):  k\in\left(1,\ldots,K\right)\right\rbrace$, where $\theta$ denotes parameters of base SSL algorithm. We denote the output logits of $f_{\theta}$ on an input as $g_{\theta}\left(\cdot\right)\in\mathbb{R}^{C}$, i.e, $f_{\theta}\left(\cdot\right)=\argmaxG_cg_{\theta}\left(\cdot\right)_c$, where $\left(\cdot\right)_c$ denotes the $c$th element. %and $\phi$ denotes the softmax activation function.

\subsection{Base SSL algorithms}
\label{fixremix}

The proposed algorithm uses FixMatch \citep{sohn2020fixmatch} or ReMixMatch \citep{berthelot2019remixmatch} as its base SSL algorithm, following other CISSL studies. FixMatch and ReMixMatch use hard or sharpened pseudo-labels for entropy minimization and strong data augmentation techniques \citep{devries2017improved,cubuk2020randaugment} for consistency regularization. Specifically, FixMatch first predicts the class probability of weakly augmented unlabeled data point $\alpha\left(u^m_b\right)$ as $q_b=P_{\theta}\left(y|\alpha\left(u^m_b\right)\right)$ and then generates hard pseudo-label $\hat{q_b}$=$\argmaxG_c\left(q_{b,c}\right)$, where $P_{\theta}\left(y|\cdot\right)=\phi\left(g_{\theta}\left(\cdot\right)\right)$ for softmax activation function $\phi$. %, where $\alpha\left(\cdot\right)$ denotes weak data augmentation techniques, $P_{\theta}\left(y|\cdot\right)=\phi\left(g_{\theta}\left(\cdot\right)\right)$ for softmax activation function $\phi$, and $q_{b,c}$ denotes the $c$th element of $q_b$. FixMatch then conducts consistency regularization by encouraging the class prediction on strongly augmented data point  $\mathcal{A}\left(u^m_b\right)$ to be consistent with $\hat{q_b}$. %such as Cutout \citep{devries2017improved} and RandomAugment \citep{cubuk2020randaugment} 
For consistency regularization, FixMatch uses hard pseudo-label $\hat{q_b}$ only when $\max_c\left(q_{b,c}\right)\geq\tau$, where $\tau$ denotes a predefined confidence threshold, to improve the quality of the pseudo-labels used for training.

%where max$\left(q_b\right)$ is confidence of $\hat{q_b}$. 
ReMixMatch similarly produces $q_b$ and aligns the distribution of $q_b$ to the class distribution of the labeled set $P_{l}\left(y\right)$ as $\Tilde{q_b}=Normalize\left(q_b\times P_{l}\left(y\right)/q\left(y\right)\right)$, where $Normalize\left(x\right)_i=x_i/\sum_j x_j$ and $q\left(y\right)$ denotes the moving average of the class probabilities predicted over the last 128 unlabeled minibatches. %, and multiplication and division are applied element-wise.  
Then, ReMixMatch sharpens the pseudo-label as $\Bar{q_b}=Normalize\left(\Tilde{q_b}^{1/T}\right)$, where $1/T$ is the sharpening temperature, $0<T<1$. %After assigning the $\Bar{q_b}$ as label of $\alpha\left(u^m_b\right)$ and $\mathcal{A}\left(u^m_b\right)$
With the sharpened pseudo-label $\Bar{q_b}$, ReMixMatch conducts consistency regularization by encouraging the class prediction on $\mathcal{A}\left(u^m_b\right)$ to be consistent with $\Bar{q_b}$. ReMixMatch also conducts Mixup regularization and %which encourages the class prediction on interpolation of two augmented images to be consistent with interpolation of labels (or pseudo-labels) corresponding to each image. ReMixMatch also conducts consistency regularization by encouraging $P_{\theta}\left(y|\mathcal{A}\left(u^m_b\right)\right)$ to be consistent with $\Bar{q_b}$, and 
self-supervised learning by rotating unlabeled samples \citep{komodakis2018unsupervised}. Data augmentation techniques $\alpha\left(\cdot\right)$ and $\mathcal{A}\left(\cdot\right)$ are described in detail in Appendix C. % as described in \citet{berthelot2019remixmatch}. % $\mathcal{A}\left(u^m_b\right)$. 
We express the training losses of FixMatch $loss_{F}$ and ReMixMatch $loss_{R}$ on $\mathcal{MX}$ and $\mathcal{MU}$ as: 
\begin{equation}
\begin{aligned}
\label{eqfix}
	loss_{F}\left(\mathcal{MX},\mathcal{MU},\hat{q},\tau;\theta\right),
\end{aligned}
\end{equation}
\vspace{-0.1in}
\begin{equation}
\begin{aligned}
\label{eqremix}
    loss_{R}\left(\mathcal{MX},\mathcal{MU},\Bar{q};\theta\right),
\end{aligned}
\end{equation}
where $\hat{q}$ and $\bar{q}$ are concatenations of $\hat{q_b}$ and $\bar{q_b}$, $b=1,\ldots,$ $\mu B$, 
respectively. The losses are detailed in Appendix D.

 %Whereas FixMatch and ReMixMatch uses few techniques to improve quality of pseudo-labels, we did not use those techniques for the proposed algorithm. %both of FixMatch and ReMixMatch minimized entropy of the class prediction $P_{\theta}\left(\cdot\right)$ by generating hard pseudo-labels $\hat{q_b}$ or sharpened pseudo-labels $\Bar{q_b}$. On the other hand, 
 %For example, 
The proposed algorithm uses FixMatch or ReMixMatch as its base SSL algorithm with some modifications as follows: %detailed below.   
%	\item{We do not minimize entropy of class predictions by generating hard pseudo-labels or sharpened pseudo-labels because it may cause the classifier to be biased toward certain classes \citep{lee2021abc}.}
%\item{The proposed algorithm does not employ entropy minimization because hard or sharpened pseudo-labels may cause the classifier to be biased toward certain classes \citep{lee2021abc}.}
\textbf{1)} The proposed algorithm does not use hard or sharpened pseudo-labels because entropy minimization of class predictions may cause the classifier to be biased towards certain classes \citep{lee2021abc}.
 \textbf{2)} The proposed algorithm does not use confidence threshold $\tau$ for FixMatch, enabling the utilization of all unlabeled samples.
 A potential limitation of utilizing inaccurate pseudo-labels can be alleviated by refining them, as discussed in \cref{proposed}.
 %This may also have disadvantage of using inaccurate pseudo-labels, but the proposed algorithm reduces the effect of the disadvantage by refining the pseudo-labels.
  \textbf{3)} The proposed algorithm does not employ the distribution alignment technique for ReMixMatch when the class distribution of the unlabeled set is unknown. This is because the labeled and unlabeled sets can potentially have different class distributions while the distribution alignment technique aligns the distribution of pseudo-labels with the class distribution of the labeled set.
 This modification helps prevent the generation of low-quality pseudo-labels in situations where there is a severe class distribution mismatch between the labeled and unlabeled sets, as discussed in \cref{result}. 

\subsection{CDMAD}
\label{proposed}

\textbf{Refinement of pseudo-labels during training}

%To refine the pseudo-labels, we first separate the biased preferences from the logits on unlabeled samples $g_{\theta}\left(\alpha\left(u_b^{m}\right)\right)$ for $b=1,\ldots,\mu B$ each iteration of training, to refine pseudo-labels.
%To improve the quality of the pseudo-labels, 

%To refine pseudo-label $q_b$, CDMAD first obtains the logits on weakly augmented unlabeled sample $g_{\theta}\left(\alpha\left(u_b^{m}\right)\right)$ and logits on a non-image input $g_{\theta}\left(\mathcal{I}\right)$. Then, CDMAD separates the relative preference regardless of learned feature $g_{\theta}\left(\mathcal{I}\right)$ from $g_{\theta}\left(\alpha\left(u_b^{m}\right)\right)$ by simply substituting as follows:

To refine a pseudo-label $q_b$ generated by FixMatch or ReMixMatch, CDMAD first calculates the logits on a weakly augmented unlabeled sample, $g_{\theta}\left(\alpha\left(u_b^{m}\right)\right)$, 
and the logits on %a \textit{non-image} input
an image without any patterns $g_{\theta}\left(\mathcal{I}\right)$, where $\mathcal{I}$ denotes %a non-image input (image with (R,G,B) values outside the range 0-255)
an image without any patterns (solid color image). The logits on a solid color image $g_{\theta}\left(I\right)$ is considered the classifier's biased degree towards each class regardless of the learned features, as discussed in \cref{intro}.
Then, CDMAD adjusts for the classifier's biased degree, $g_{\theta}\left(\mathcal{I}\right)$, in the logits  $g_{\theta}\left(\alpha\left(u_b^{m}\right)\right)$, by simple subtraction as follows:
%Then, CDMAD removes the relative preference regardless of learned features,  $g_{\theta}\left(\mathcal{I}\right)$ from $g_{\theta}\left(\alpha\left(u_b^{m}\right)\right)$, by simple subtraction as follows:
\begin{equation}
\label{eq1}	
g_{\theta}^{*}\left(\alpha\left(u_b^{m}\right)\right)=g_{\theta}\left(\alpha\left(u_b^{m}\right)\right)-g_{\theta}\left(\mathcal{I}\right),
\end{equation}
where $g_{\theta}^{*}\left(\cdot\right)$ denotes the refined logits, %which can be thought as calculated based only on the learned features. 
which are considered to be calculated based only on the learned features. With $g_{\theta}^{*}\left(\alpha\left(u_b^{m}\right)\right)$, the refined pseudo-label $q_b^*$ is obtained as:
\begin{equation}
\label{eq22}
q_b^*=\phi\left(g_{\theta}^{*}\left(\alpha\left(u_b^{m}\right)\right)\right).%=Normalize\left(q_b/P_{\theta}\left(y|\mathcal{I}\right)\right).
\end{equation}

As noted in \cref{fixremix}, CDMAD does not use the distribution alignment technique for ReMixMatch. Instead, CDMAD adds the supervised loss for weakly augmented labeled sample $Sup(\mathcal{MX};\theta)$ into the training loss of ReMixMatch to enhance the classifier's familiarity with labeled samples. % $\alpha\left(\cdot\right)$. 
This can effectively improve the quality of pseudo-labels when the class distributions of labeled and unlabeled sets mismatch, as discussed in \cref{result}.
%With the prediction loss on weakly augmented labeled samples for ReMixMatch, 
The training losses for FixMatch and ReMixMatch with CDMAD, denoted by $loss_{F}^*$ and $loss_{R}^*$, respectively, are expressed as:
\vspace{-0.05in}
\begin{equation}
\label{eq2}	loss_{F}^*=loss_{F}\left(\mathcal{MX},\mathcal{MU},q^*,0;\theta\right),
\end{equation}
\vspace{-0.2in}
\begin{equation}
\label{eq3}
loss_{R}^*=loss_{R}\left(\mathcal{MX},\mathcal{MU},q^*;\theta\right)+Sup(\mathcal{MX};\theta),
\end{equation}
where $loss_{F}$ and $loss_{R}$ are from \cref{eqfix} and \cref{eqremix}, and $q^*$ is the concatenation of the $q^*_b$. %for $b=1,\ldots,\mu B$. %, and $CEloss=CrossEntropy\left(p^m_b,P_{\theta}\left(y|\alpha\left(x^m_b\right)\right)\right)$ denotes the prediction loss on weakly augmented labeled samples $\alpha\left(x^m_b\right)$, where $p^m_b$ is one-hot encoded $y^m_b$. 
\cref{fig:architecture} illustrates the pseudo-label refinement process. By using the refined pseudo-labels during the training of the base SSL algorithm, the quality of representations is improved%, as presented in Appendix J
.
%With using refined pseudo-labels, the proposed algorithm can learn improved quality of representations, as we will present in Section \ref{analysis}.
%\begin{figure}[htbp]

\begin{figure}[t]

 \begin{center}
        \hspace{-0.1in}\includegraphics[width=3.45in,height=2.3in]
        {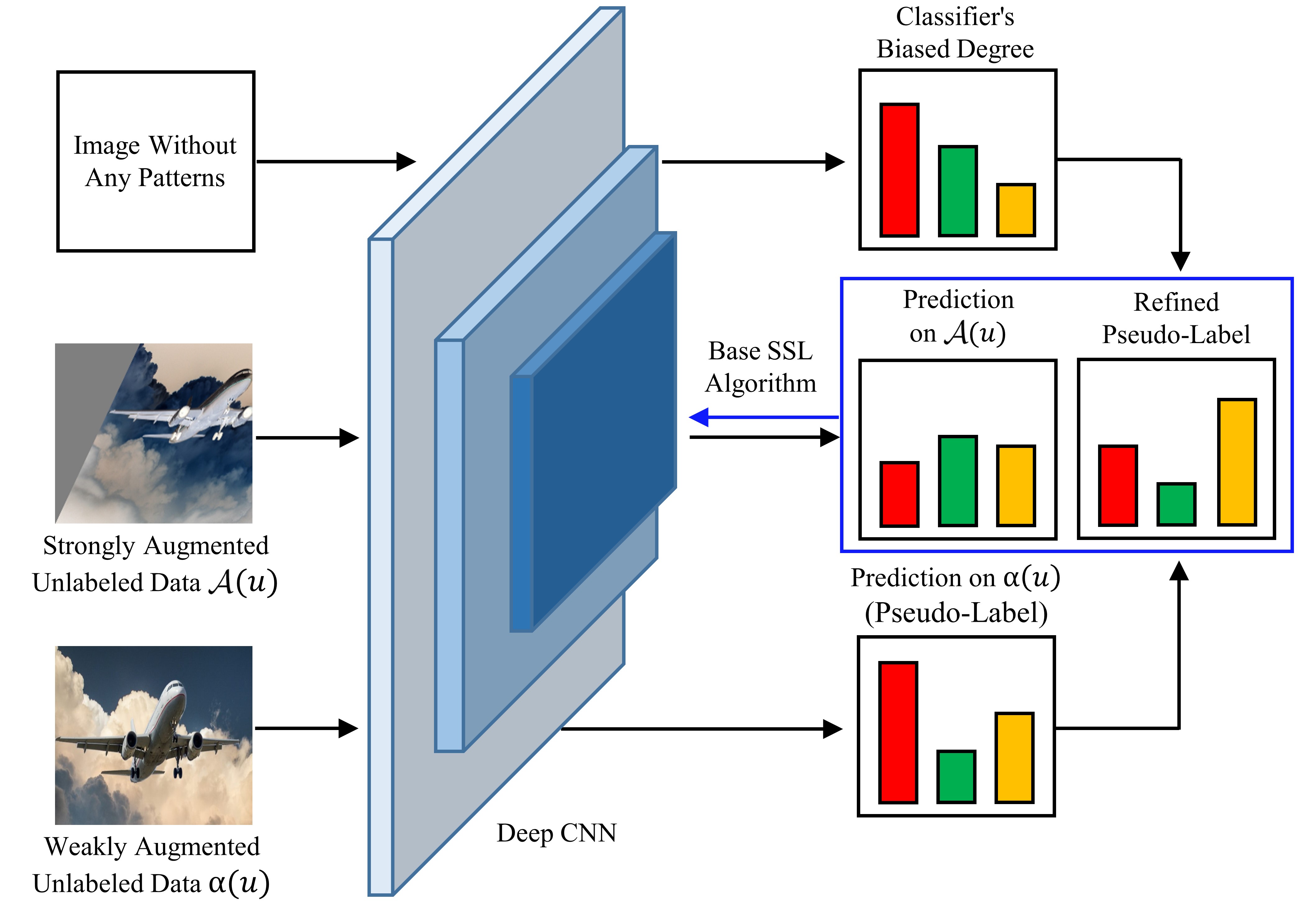}

	\end{center}
 \vspace{-0.1in}
\caption{%The process pseudo-labels refinement with CDMAD and consistency regularization. The consistency regularization loss is calculated from the predicted class probability on $\mathcal{A}\left(u\right)$ and the CDMAD pseudo-label in the blue box.
 Pseudo-label refinement process using CDMAD. %The refined pseudo-labels and class predictions on strongly augmented unlabeled samples are used for training the base SSL algorithm.
 }
	\label{fig:architecture}
\vspace{-0.2in}
\end{figure}

CDMAD can effectively refine the pseudo-labels even under severe class distribution mismatch between the labeled and unlabeled sets because classifier's biased degree $g_{\theta}\left(\mathcal{I}\right)$ is affected by the class distributions of both sets.
%CDMAD can effectively refine the pseudo-labels even under severe class distribution mismatch between labeled and unlabeled sets \textcolor{red}{because $g_{\theta}\left(\mathcal{I}\right)$ can be measured without directly using the information of the class distribution of the unlabeled set, which is often unavailable in practice.} 
Unlike the previous CISSL studies, CDMAD does not require additional parameters for an auxiliary classifier or additional training stages. CDMAD can be implemented by adding a few lines of code as presented in Appendix A. %In Appendix N, we qualitatively verify that the proposed algorithm can effectively refine biased pseudo-labels and improve the quality of representations using the refined pseudo-labels. 
%we qualitatively verify that the proposed algorithm can learn improved quality of representations with using refined pseudo-labels. 

%and $CEloss=CrossEntropy\left(p^m_b,P_{\theta}\left(y|\alpha\left(x^m_b\right)\right)\right)$ denotes the prediction loss on weakly augmented labeled samples, where $p^m_b$ is one-hot encoded $y^m_b$.
 %CDMAD increases training cost in a negligible degree because it adds only one forward propagation (obtain logits on non-image inputs) per each iteration of training. 
 %Note that ReMixMatch also uses the refined pseudo-labels for mix-up regularization in addition to the presented consistency regularization.

\textbf{Refinement of biased class predictions during testing}

%Even assuming that all biased pseudo-labels are perfectly refined during the training process, 
Even with all biased pseudo-labels perfectly refined during the training process, biased predictions may still be produced for the test samples because the training set is class-imbalanced.
To refine biased class predictions on test samples, CDMAD also adjusts the biased logits on $x^{test}_k$, for $k=1,\ldots,K$. CDMAD first calculates the classifier's biased degree $g_{\theta}\left(\mathcal{I}\right)$. Then, similar to \cref{eq2}, the logits for test samples, $g_{\theta}\left(x^{test}_k\right)$, for $k=1,\ldots,K$, are adjusted as:
\begin{equation}
\label{eq4}
g_{\theta}^{*}\left(x^{test}_k\right)=g_{\theta}\left(x^{test}_k\right)-g_{\theta}\left(\mathcal{I}\right).
\end{equation}
%\begin{equation}
%\label{eq4}
%P_{\theta}^{*}\left(y|x^{test}_k\right)=Normalize\left(P_{\theta}\left(y|x^{test}_k\right)/P_{\theta}\left(y|\mathcal{I}\right)\right).
%\end{equation}
    %With the refined logits $g_{\theta}^{*}\left(x^{test}_k\right)$, we can get the refined class probability on test sample $P_{\theta}^*\left(y|x^{test}_k\right)$ as follows:
%\begin{equation}
%\label{eq5}
%	P_{\theta}^*\left(y|x^{test}_k\right)=\phi\left(g_{\theta}^{*}\left(x^{test}_k\right)\right)=Normalize\left(\frac{P_{\theta}\left(y|x^{test}_k\right)}{P_{\theta}\left(y|\mathcal{I}\right)}\right).
%\end{equation}
With the adjusted logits $g_{\theta}^{*}\left(x^{test}_k\right)$, the refined class prediction $f_{\theta}^*\left(x^{test}_k\right)$ is obtained as follows: 
\begin{equation}
\begin{aligned}
\label{eqtest}
f_{\theta}^*\left(x^{test}_k\right)=\argmaxG_c g_{\theta}^{*}\left(x^{test}_k\right)_c\\ =\argmaxG_{y\in[C]}P_{\theta}\left(y|x^{test}_k\right)/P_{\theta}\left(y|\mathcal{I}\right).
\end{aligned}
\end{equation}
%\begin{equation}
%\label{eq5}
%f_{\theta}^*\left(x^{test}_k\right)=\argmaxG_c g_{\theta}^{*}\left(x^{test}_k\right)_c.
%\end{equation}
We illustrate the test process in Appendix E and pseudo code of the proposed algorithm in Appendix F. %Post-adjustment of biased predictions using 

\subsection{CDMAD as a CISSL extension of post-hoc logit-adjustment (LA)}
\label{LA}

CDMAD can be viewed as a CISSL extension of LA \citep{menon2020long} to take into account class distribution mismatch between labeled and unlabeled sets, where LA was originally introduced to re-balance a biased classifier in CIL. To re-balance the classifier, LA post-adjusts the logits on test samples $g_{\theta}\left(x^{test}_k\right)$ by simply subtracting the log of the estimate of the underlying class prior $P(y)$, denoted by $\pi$ (e.g., each class frequency on the training set), as follows:
\begin{equation}
\label{eqla}
g_{\theta}^{*}\left(x^{test}_k\right)=g_{\theta}\left(x^{test}_k\right)-\log \pi.
\end{equation}
%where $g_{\theta}\left(\cdot\right)\in\mathbb{R}^{C}$ denotes output logits calculated using $\theta$ and $g^*_{\theta}\left(\cdot\right)$ denotes the adjusted logits. 
%Re-balancing of LA is based on a firm statistical grounding : the optimal solution under the LA coincides with the Bayes-optimal solution for the balanced error \citep{chan1998learning}. LA also empirically improved classification performance under various CIL settings. %Because of the firm statistical grounding, empirical effectiveness and simpleness, recent CISSL algorithms \citep{wei2021crest,oh2022daso,lai2022smoothed} used LA under the setting that labeled and unlabeled sets share same class distribution. However, they did not use LA under the setting that class distribution of unlabeled set is unknown and severely differs with that of labeled set because LA may probably even decrease classification performance by re-balancing classifier to inappropriate degree. 
%However, under CISSL settings where class distribution of unlabeled set is unknown and severely differs with that of labeled set, LA may even decrease classification performance by re-balancing classifier to an inappropriate degree because class distribution of unlabeled set can not be considered in the term $P\left(y\right)$ in Eq (\ref{eqla}). %when class distribution of unlabeled set is unknown. %This approximated $P\left(y\right)$ and true $P\left(y\right)$ will significantly differ under severe class distribution mismatch setting, which may lead to inappropriate degree of re-balancing.
The adjustment in \cref{eqla} was proven to be Fisher consistent for minimizing the balanced error rate (BER), 
\vspace{-0.05in}
\begin{equation}
\label{ber}
%\mbox{BER}\left(s\right)=\frac{1}{C}\sum_{y\in[C]}P_{x|y}\left(y\neq\argmax_{y'\in[C]}s\left(y'|x\right)\right),
\mbox{BER}\left(f^*_{\theta}\right)=\frac{1}{C}\sum_{y\in[C]}P_{x|y}\left(y\neq f_{\theta}^*\left(x\right)\right).
\vspace{-0.05in}
\end{equation}
%where $s:\mathbb{R}^{d}\rightarrow\mathbb{R}^{C}$ is a function that predicts the class of input $x$. %The logit adjustment under Eq (\ref{eqla}) was proven to be the Bayes-optimal solution for minimizing the balanced error rate.
In addition, empirical evidence has demonstrated that LA can effectively enhance classification performance across various class-imbalanced learning scenarios.

However, in CISSL, where the class distribution of the unlabeled set is unknown and may substantially differ from that of the labeled set, LA may result in the classifier being re-balanced to an inappropriate degree, thereby leading to a decrease in classification performance. This limitation arises from the challenge of incorporating the unknown class distribution of the unlabeled set for re-balancing. Specifically, the class prior $P(y)$, which is often approximated as the class distribution of the labeled set $P_l\left(y\right)$, cannot consider the class distribution of the unlabeled set.

By comparing \cref{eq4} and \cref{eqla}, we can observe that CDMAD shares a similar form with LA. Specifically, given that $g_{\theta}\left(\mathcal{I}\right)+constant=\log P_{\theta}\left(y|\mathcal{I}\right)$, CDMAD can be seen as replacing the class frequencies $\pi$ in \cref{eqla} by $P_{\theta}\left(y\right|\mathcal{I})$, which can be considered an estimate of the \textit{classifier's prior} $P_{\theta}(y)$. (With this interpretation, CDMAD can be viewed as allowing the classifier to predict class probabilities based solely on the given input, without being affected by the classifier's prior $P_{\theta}(y)$.) 
This replacement allows CDMAD to implicitly incorporate the class distributions of both labeled and unlabeled sets, facilitating its awareness of class distribution mismatch between the two sets, as we will further discuss in \cref{analysis}. Furthermore, whereas LA is solely employed to refine biased class predictions on the test set, CDMAD is also employed to refine the biased pseudo-labels.
%\textcolor{red}{As like LA, the effectiveness of CDMAD can be explained in terms of the balanced error BER.}
Similar to LA, CDMAD ensures Fisher consistency for the balanced error.

%\textcolor{red}{\textbf{Proposition 1.} The adjustment by CDMAD in Eq.(\ref{eqtest}) is Fisher consistent for minimizing the balanced error under the Assumption 1. In other words, when the network parameters $\theta$ are trained with data points of the entire population, the adjusted class predictions minimize the balanced error.}
\begin{prop}
Given a solid color image $\mathcal{I}$ independent of class labels $y$, the refinement by CDMAD in \cref{eqtest} is Fisher consistent for minimizing %the balanced error.
the BER in \cref{ber}. 
%In other words, when the network parameters $\theta$ are trained with data points of the entire population, the adjusted class predictions minimize the balanced error.
\end{prop}
%\textcolor{red}{\textbf{Assumption 1.} Non-image input $\mathcal{I}$ is uninformative to class variable $y$ (i.e, $\mathcal{I}$ is a realization of random variable $z$ which is independent with y).}
%\textcolor{red}{\textit{\textbf{proof.}} Under the Assumption1, the refined class prediction for input $x$, 
% $f_{\theta}^*\left(x\right)= \argmaxG_{y\in[C]}P_{\theta}\left(y|x\right)/P_{\theta}\left(y|\mathcal{I}\right)$ becomes $\argmaxG_{y\in[C]}P_{\theta}\left(y|x\right)/P_{\theta}\left(y\right)$. Because deep neural network is a universal function approximator \citep{hornik1989multilayer,zhou2020universality}, when the network is trained with data points of the entire population, $P_{\theta}\left(y|x\right)$ becomes underlying class-probabilities $P\left(y|x\right)$ and $P_{\theta}\left(y\right)=\int P_{\theta}\left(y|x\right)P\left(x\right)dx$ becomes $\int P\left(y|x\right)P\left(x\right)dx=P\left(y\right)$. Then, $\argmaxG_{y\in[C]}P_{\theta}\left(y|x\right)/P_{\theta}\left(y\right)$ becomes $\argmaxG_{y\in[C]}P\left(y|x\right)/P\left(y\right)=\argmaxG_{y\in[C]}P\left(x|y\right)$. From Corollary 4 of \cite{koyejo2014consistent}, it is known that a function $s^*$ which minimizes the balanced error BER as $s^*\in\argmin_{s:\mathbb{R}^{d}\rightarrow\mathbb{R}^{C}}BER\left(s\right)$, satisfies $\argmaxG_{y\in[C]}s^*\left(y|x\right)=\argmaxG_{y\in[C]}P\left(x|y\right)$.} %Thus, CDMAD Fisher consistently minimize the balanced error rate under the assumption that $I$ is uninformative to $y$.}
 \begin{proof}
%Due to the universal approximation theorem \citep{hornik1989multilayer,zhou2020universality}, $P_\theta \left(y\right)$ becomes $P\left(y\right)$, and $P_\theta \left(y|I\right)$ becomes $P \left(y|I\right)$, under the population setting. By the assumption of Proposition 1, i.e., $P\left(y\right)=P \left(y|I\right)$, it follows that $P_\theta \left(y\right)=P_\theta \left(y|I\right)$. The refined class prediction for input $x$, $f_{\theta}^*\left(x\right)= \argmaxG_{y\in[C]}P_{\theta}\left(y|x\right)/P_{\theta}\left(y|\mathcal{I}\right)=\argmaxG_{y\in[C]}P_{\theta}\left(y|x\right)$ $/P_{\theta}\left(y\right)$.If the network is trained with the entire population, $P_{\theta}\left(y|x\right)$ becomes $P\left(y|x\right)$, and $P_{\theta}\left(y\right)=\int P_{\theta}\left(y|x\right)P\left(x\right)dx$ becomes $\int P\left(y|x\right)P\left(x\right)dx=P\left(y\right)$, due to the universal approximation theorem. Then, $f_{\theta}^*\left(x\right)$ $=\argmaxG_{y\in[C]}P_{\theta}\left(y|x\right)/P_{\theta}\left(y\right)$ becomes $\argmaxG_{y\in[C]}$ $P\left(y|x\right)/P\left(y\right)=\argmaxG_{y\in[C]}P\left(x|y\right)$, which minimizes the BER \citep{menon2020long,collell2016reviving}.
The refined class prediction for input $x$, $f_{\theta}^*\left(x\right)= \argmaxG_{y\in[C]}P_{\theta}\left(y|x\right)/P_{\theta}\left(y|\mathcal{I}\right)$. If the network is trained with the entire population, $P_{\theta}\left(y|x\right)$ becomes $P\left(y|x\right)$, and $P_{\theta}\left(y|\mathcal{I}\right)$ becomes $P\left(y|\mathcal{I}\right)$ due to the universal approximation theorem. Then, $f_{\theta}^*\left(x\right)$ $=\argmaxG_{y\in[C]}P_{\theta}\left(y|x\right)/P_{\theta}\left(y|\mathcal{I}\right)$ becomes $\argmaxG_{y\in[C]}P\left(y|x\right)/P\left(y|\mathcal{I}\right)$ and by the assumption of Proposition 1, i.e., $P\left(y\right)=P \left(y|I\right)$, it follows that $\argmaxG_{y\in[C]}P\left(y|x\right)/P\left(y|\mathcal{I}\right)=\argmaxG_{y\in[C]}P\left(y|x\right)/P\left(y\right)=\argmaxG_{y\in[C]}P\left(x|y\right)$, which minimizes the BER \citep{menon2020long,collell2016reviving}.
%According to \citet{menon2020long,collell2016reviving}, a function $f^*$ that satisfies $\argmaxG_{y\in[C]}f^*\left(y|x\right)=\argmaxG_{y\in[C]}P\left(x|y\right)$ minimizes the BER,  $f^*\in\argmin_{s:\mathbb{R}^{d}\rightarrow\mathbb{R}^{C}}BER\left(s\right)$.
\end{proof}
\vspace{-0.1in}
Fisher consistency is a desirable property for an estimator and implies that in the entire population setting,  %i.e., if the probability distribution that generates the data were available, then 
optimizing the estimator yields the best result. In our case, we prove that in the population setting, CDMAD minimizes the BER. 
%However, in practice, the entire population is not available, and thus the theory cannot be directly applied. Nevertheless, Proposition 1 makes us to be optimistic about that the algorithm is learned in the direction to minimize the BER. %Simultaneously, to investigate finite sample properties, we conducted extensive experiments in realistic CISSL scenarios, and demonstrated state-of-the-art performance of CDMAD.}
Although it is impractical to use data points of the entire population for training of network parameters under SSL setting, the above proof shows the possibility that CDMAD can effectively mitigate class imbalance by reducing BER. Through empirical results in section \ref{exp}, we verify that CDMAD does indeed mitigate the class imbalance.
\definecolor{Gray}{gray}{0.9}
%\vspace{-0.05in}
\section{Experiments}\label{exp}
%\vspace{-0.025in}
\subsection{Experimental setup}

%To verify the effectiveness of CDMAD, 
We conducted experiments on CIFAR-10-LT, CIFAR-100-LT \citep{cui2019class}, STL-10-LT \citep{NEURIPS2020_a7968b43}, and Small-ImageNet-127 \citep{fan2022CoSSL} under the settings considered in previous studies \citep{lai2022smoothed, fan2022CoSSL}. We used the balanced accuracy (bACC) \citep{huang2016learning} and geometric mean (GM) \citep{kubat1997addressing} to evaluate the classification performance on CIFAR-10-LT and STL-10-LT and only bACC to evaluate classification performance on CIFAR-100-LT and Small-ImageNet-127, following \citet{fan2022CoSSL}. For CIFAR-10-LT, CIFAR-100-LT, and STL-10-LT, we repeated the experiments three times and report the average and standard error for the performance measures. %We report the performance of the baseline algorithms reported in Tables in \citet{lai2022smoothed} and \citet{fan2022CoSSL} when it is reproducible; the performance measured using the uploaded code was reported otherwise. 
We used a white image to measure the $g_{\theta}\left(y\right)$. 
Performance measures, description of datasets and experimental setup, and baseline algorithms are detailed in Appendices G, H, and I, respectively.

\subsection{Experimental results}
\label{result}
\begin{table}[htbp]
%\vspace{-0.2in}
  \caption{Comparison of bACC/GM on CIFAR-10-LT% under $\gamma=\gamma_{l}=\gamma_{u}$ ($\gamma_{u}$ is assumed  to be known). %``-" indicates unreported values. %``SSL" represents whether an unlabeled set was used, and `
  %``RB" represents whether the class distribution is re-balanced.
  }
  \vspace{-0.05in}
  \label{cifar}
  \centering
  \resizebox{3.25in}{!}{%
  \begin{tabular}{ccccc}
    \toprule
    \multicolumn{5}{c}{CIFAR-10-LT ($\gamma=\gamma_l=\gamma_u$, $\gamma_u$ is assumed  to be known) }                   \\
    \midrule
    \multicolumn{2}{c}{Algorithm}   &$\gamma=50$&$\gamma=100$&$\gamma=150$\\
    \midrule
    \multicolumn{2}{c}{Vanilla }&$65.2$\tiny$\pm0.05$ \normalsize/ $61.1$\tiny$\pm0.09$&$58.8$\tiny$\pm0.13$ \normalsize/ $58.2$\tiny$\pm0.11$&$55.6$\tiny$\pm0.43$ \normalsize/ $44.0$\tiny$\pm0.98$\\
    \midrule
    \multicolumn{2}{c}{Re-sampling}&$64.3$\tiny$\pm0.48$ \normalsize/ $60.6$\tiny$\pm0.67$&$55.8$\tiny$\pm0.47$ \normalsize/ $45.1$\tiny$\pm0.30$&$52.2$\tiny$\pm0.05$ \normalsize/ $38.2$\tiny$\pm1.49$\\
    \multicolumn{2}{c}{LDAM-DRW }&$68.9$\tiny$\pm0.07$ \normalsize/ $67.0$\tiny$\pm0.08$&$62.8$\tiny$\pm0.17$ \normalsize/ $58.9$\tiny$\pm0.60$&$57.9$\tiny$\pm0.20$ \normalsize/ $50.4$\tiny$\pm0.30$\\
    \multicolumn{2}{c}{cRT }&$67.8$\tiny$\pm0.13$ \normalsize/ $66.3$\tiny$\pm0.15$&$63.2$\tiny$\pm0.45$ \normalsize/ $59.9$\tiny$\pm0.40$&$59.3$\tiny$\pm0.10$ \normalsize/ $54.6$\tiny$\pm0.72$\\
    \midrule
    \multicolumn{2}{c}{FixMatch}&$79.2$\tiny$\pm0.33$ \normalsize/ $77.8$\tiny$\pm0.36$&$71.5$\tiny$\pm0.72$ \normalsize/ $66.8$\tiny$\pm1.51$&$68.4$\tiny$\pm0.15$ \normalsize/ $59.9$\tiny$\pm0.43$ \\
    \multicolumn{2}{c}{/++DARP+cRT}&$85.8$\tiny$\pm0.43$ \normalsize/ $85.6$\tiny$\pm0.56$&$82.4$\tiny$\pm0.26$ \normalsize/ $81.8$\tiny$\pm0.17$&$79.6$\tiny$\pm0.42$ \normalsize/ $78.9$\tiny$\pm0.35$ \\
    \multicolumn{2}{c}{/+CReST+LA }&$85.6$\tiny$\pm0.36$ \normalsize/ $81.9$\tiny$\pm0.45$&$81.2$\tiny$\pm0.70$ \normalsize/ $74.5$\tiny$\pm0.99$&$71.9$\tiny$\pm2.24$ \normalsize/ $64.4$\tiny$\pm1.75$ \\
    \multicolumn{2}{c}{/+ABC }&$85.6$\tiny$\pm0.26$ \normalsize/ $85.2$\tiny$\pm0.29$&$81.1$\tiny$\pm1.14$ \normalsize/ $80.3$\tiny$\pm1.29$&$77.3$\tiny$\pm1.25$ \normalsize/ $75.6$\tiny$\pm1.65$ \\
    \multicolumn{2}{c}{/+CoSSL}&$86.8$\tiny$\pm0.30$ \normalsize/ $86.6$\tiny$\pm0.25$&$83.2$\tiny$\pm0.49$ \normalsize/ $82.7$\tiny$\pm0.60$&$80.3$\tiny$\pm0.55$ \normalsize/ $79.6$\tiny$\pm0.57$ \\
    \multicolumn{2}{c}{/+SAW+LA}&$86.2$\tiny$\pm0.15$ \normalsize/ $83.9$\tiny$\pm0.35$&$80.7$\tiny$\pm0.15$ \normalsize/ $77.5$\tiny$\pm0.21$&$73.7$\tiny$\pm0.06$ \normalsize/ $71.2$\tiny$\pm0.17$ \\
    \multicolumn{2}{c}{/+Adsh} &$83.4$\tiny$\pm0.06$\normalsize/ -&$76.5$\tiny$\pm0.35$\normalsize/ -&$71.5$\tiny$\pm0.30$\normalsize/ - \\
    \multicolumn{2}{c}{/+DebiasPL} &-/ -&$80.6$\tiny$\pm0.50$\normalsize/ -&-/ - \\
    \multicolumn{2}{c}{/+UDAL }&$86.5$\tiny$\pm0.29$ \normalsize/ -&$81.4$\tiny$\pm0.39$ \normalsize/ -&$77.9$\tiny$\pm0.33$ \normalsize/ - \\
    \multicolumn{2}{c}{/+L2AC }&-/ -&$82.1$\tiny$\pm0.57$ \normalsize/ $81.5$\tiny$\pm0.64$&$77.6$\tiny$\pm0.53$ \normalsize/ $75.8$\tiny$\pm0.71$ \\
    
    \rowcolor{Gray}
    \multicolumn{2}{c}{/+CDMAD}&\textbf{87.3}\tiny$\pm0.12$ \normalsize/ \textbf{87.0}\tiny$\pm0.15$&\textbf{83.6}\tiny$\pm0.46$ \normalsize/ \textbf{83.1}\tiny$\pm0.57$&\textbf{80.8}\tiny$\pm0.86$ \normalsize/ \textbf{79.9}\tiny$\pm1.07$ \\
    \midrule
    \multicolumn{2}{c}{ReMixMatch }&$81.5$\tiny$\pm0.26$ \normalsize/ $80.2$\tiny$\pm0.32$&$73.8$\tiny$\pm0.38$ \normalsize/ $69.5$\tiny$\pm0.84$&$69.9$\tiny$\pm0.47$ \normalsize/ $62.5$\tiny$\pm0.35$ \\
    
    \multicolumn{2}{c}{/+DARP+cRT }&$87.3$\tiny$\pm0.61$ \normalsize/ $87.0$\tiny$\pm0.11$&$83.5$\tiny$\pm0.07$ \normalsize/ $83.1$\tiny$\pm0.09$&$79.7$\tiny$\pm0.54$ \normalsize/ $78.9$\tiny$\pm0.49$ \\
    \multicolumn{2}{c}{/+CReST+LA }&$84.2$\tiny$\pm0.11$ \normalsize/ -&$81.3$\tiny$\pm0.34$ \normalsize/ -&$79.2$\tiny$\pm0.31$ \normalsize/ - \\
    \multicolumn{2}{c}{/+ABC }&$87.9$\tiny$\pm0.47$ \normalsize/ $87.6$\tiny$\pm0.51$&$84.5$\tiny$\pm0.32$ \normalsize/ $84.1$\tiny$\pm0.36$&$80.5$\tiny$\pm1.18$ \normalsize/ $79.5$\tiny$\pm1.36$ \\
    \multicolumn{2}{c}{/+CoSSL }&$87.7$\tiny$\pm0.21$ \normalsize/ $87.6$\tiny$\pm0.25$&$84.1$\tiny$\pm0.56$ \normalsize/ $83.7$\tiny$\pm0.66$ &$81.3$\tiny$\pm0.83$ \normalsize/ $80.5$\tiny$\pm0.76$ \\
    \multicolumn{2}{c}{/+SAW+cRT }&$87.6$\tiny$\pm0.21$ \normalsize/ $87.4$\tiny$\pm0.26$&85.4\tiny$\pm0.32$ \normalsize/ $83.9$\tiny$\pm0.21$&$79.9$\tiny$\pm0.15$ \normalsize/ $79.9$\tiny$\pm0.12$ \\
    
    \rowcolor{Gray}
   \multicolumn{2}{c}{/+CDMAD}&\textbf{88.3}\tiny$\pm0.35$ \normalsize/ \textbf{88.1}\tiny$\pm0.35$&\textbf{85.5}\tiny$\pm0.46$ \normalsize/ \textbf{85.3}\tiny$\pm0.44$&\textbf{82.5}\tiny$\pm0.23$ \normalsize/ \textbf{82.0}\tiny$\pm0.30$ \\
    \bottomrule
  \end{tabular}}
\vspace{-0.25in}
\end{table}
\vspace{-0.025in}
\begin{table*}[htbp]

  \caption{Comparison of bACC/GM on CIFAR-10-LT and STL-10-LT under $\gamma_l\neq\gamma_u$ ($\gamma_u$ is assumed to be unknown). ReMixMatch* denotes ReMixMatch with the estimated class distribution of the unlabeled set  \citep{NEURIPS2020_a7968b43}.}
  \vspace{-0.05in}
  \label{cifar2}
  \centering
  \resizebox{6.5in}{!}{%
  \begin{tabular}{ccccccc}
    \toprule
    &&\multicolumn{3}{c}{CIFAR-10-LT ($\gamma_l=100$, $\gamma_u$ is assumed to be unknown)} &\multicolumn{2}{c}{STL-10-LT ($\gamma_u=$Unknown)}                  \\
    \midrule
    \multicolumn{2}{c}{Algorithm}   &$\gamma_u=1$&$\gamma_u=50$&$\gamma_u=150$&$\gamma_{l}=10$ &$\gamma_{l}=20$\\
    \midrule
    \multicolumn{2}{c}{FixMatch}&$68.9$\tiny$\pm1.95$ \normalsize/ $42.8$\tiny$\pm8.11$&$73.9$\tiny$\pm0.25$ \normalsize/ $70.5$\tiny$\pm0.52$&$69.6$\tiny$\pm0.60$ \normalsize/ $62.6$\tiny$\pm1.11$&$72.9$\tiny$\pm0.09$\normalsize/ $69.6$\tiny$\pm0.01$&$63.4$\tiny$\pm0.21$\normalsize/ $52.6$\tiny$\pm0.09$ \\
    \multicolumn{2}{c}{FixMatch+DARP}&$85.4$\tiny$\pm0.55$ \normalsize/ $85.0$\tiny$\pm0.65$&$77.3$\tiny$\pm0.17$ \normalsize/ $75.5$\tiny$\pm0.21$ &$72.9$\tiny$\pm0.24$ \normalsize/ $69.5$\tiny$\pm0.18$&$77.8$\tiny$\pm0.33$\normalsize/ $76.5$\tiny$\pm0.40$&$69.9$\tiny$\pm1.77$\normalsize/ $65.4$\tiny$\pm3.07$ \\
    
    \multicolumn{2}{c}{FixMatch+DARP+LA}&$86.6$\tiny$\pm1.11$ \normalsize/ $86.2$\tiny$\pm1.15$&$82.3$\tiny$\pm0.32$ \normalsize/ $81.5$\tiny$\pm0.29$ &$78.9$\tiny$\pm0.23$ \normalsize/ $77.7$\tiny$\pm0.06$&$78.6$\tiny$\pm0.30$\normalsize/ $77.4$\tiny$\pm0.40$&$71.9$\tiny$\pm0.49$\normalsize/ $68.7$\tiny$\pm0.51$ \\
    \multicolumn{2}{c}{FixMatch+DARP+cRT}&$87.0$\tiny$\pm0.70$ \normalsize/ $86.8$\tiny$\pm0.67$&$82.7$\tiny$\pm0.21$ \normalsize/ $82.3$\tiny$\pm0.25$ &$80.7$\tiny$\pm0.44$ \normalsize/ $80.2$\tiny$\pm0.61$&$79.3$\tiny$\pm0.23$\normalsize/ $78.7$\tiny$\pm0.21$&$74.1$\tiny$\pm0.61$\normalsize/ $73.1$\tiny$\pm1.21$ \\
    
    \multicolumn{2}{c}{FixMatch+ABC}&$82.7$\tiny$\pm0.49$ \normalsize/ $81.9$\tiny$\pm0.68$&$82.7$\tiny$\pm0.64$ \normalsize/ $82.0$\tiny$\pm0.76$&$78.4$\tiny$\pm0.87$ \normalsize/ $77.2$\tiny$\pm1.07$&$79.1$\tiny$\pm0.46$\normalsize/ $78.1$\tiny$\pm0.57$&$73.8$\tiny$\pm0.15$\normalsize/ $72.1$\tiny$\pm0.15$ \\
    \multicolumn{2}{c}{FixMatch+SAW}&$81.2$\tiny$\pm0.68$ \normalsize/ $80.2$\tiny$\pm0.91$&$79.8$\tiny$\pm0.25$ \normalsize/ $79.1$\tiny$\pm0.32$&$74.5$\tiny$\pm0.97$\normalsize/ $72.5$\tiny$\pm1.37$&-/-&-/-\\
    \multicolumn{2}{c}{FixMatch+SAW+LA}&$84.5$\tiny$\pm0.68$ \normalsize/ $84.1$\tiny$\pm0.78$&$82.9$\tiny$\pm0.38$ \normalsize/ $82.6$\tiny$\pm0.38$&$79.1$\tiny$\pm0.81$ \normalsize/ $78.6$\tiny$\pm0.91$&-/-&-/- \\
    \multicolumn{2}{c}{FixMatch+SAW+cRT}&$84.6$\tiny$\pm0.23$ \normalsize/ $84.4$\tiny$\pm0.26$&$81.6$\tiny$\pm0.38$ \normalsize/ $81.3$\tiny$\pm0.32$&$77.6$\tiny$\pm0.40$ \normalsize/ $77.1$\tiny$\pm0.41$&-/-&-/- \\
    
    \rowcolor{Gray}
    
     \multicolumn{2}{c}{FixMatch+CDMAD}&\textbf{87.5}\tiny$\pm0.46$ \normalsize/ \textbf{87.1}\tiny$\pm0.50$&\textbf{85.7}\tiny$\pm0.36$ \normalsize/ \textbf{85.3}\tiny$\pm0.38$&\textbf{82.3}\tiny$\pm0.23$ \normalsize/ \textbf{81.8}\tiny$\pm0.29$&\textbf{79.9}\tiny$\pm0.23$\normalsize/ \textbf{78.9}\tiny$\pm0.38$&\textbf{75.2}\tiny$\pm0.40$\normalsize/ \textbf{73.5}\tiny$\pm0.31$ \\
    \midrule
    \multicolumn{2}{c}{ReMixMatch}&$48.3$\tiny$\pm0.14$ \normalsize/ $19.5$\tiny$\pm0.85$&$75.1$\tiny$\pm0.43$ \normalsize/ $71.9$\tiny$\pm0.77$&$72.5$\tiny$\pm0.10$ \normalsize/ $68.2$\tiny$\pm0.32$&$67.8$\tiny$\pm0.45$\normalsize/ $61.1$\tiny$\pm0.92$&$60.1$\tiny$\pm1.18$\normalsize/ $44.9$\tiny$\pm1.52$ \\
    \multicolumn{2}{c}{ReMixMatch*}&$85.0$\tiny$\pm1.35$ \normalsize/ $84.3$\tiny$\pm1.55$&$77.0$\tiny$\pm0.12$ \normalsize/ $74.7$\tiny$\pm0.04$&$72.8$\tiny$\pm0.10$ \normalsize/ $68.8$\tiny$\pm0.21$&$76.7$\tiny$\pm0.15$ \normalsize/ $73.9$\tiny$\pm0.32$ &$67.7$\tiny$\pm0.46$ \normalsize/ $60.3$\tiny$\pm0.76$ \\
    \multicolumn{2}{c}{ReMixMatch*+DARP}&$86.9$\tiny$\pm0.10$ \normalsize/ $86.4$\tiny$\pm0.15$&$77.4$\tiny$\pm0.22$ \normalsize/ $75.0$\tiny$\pm0.25$&$73.2$\tiny$\pm0.11$ \normalsize/ $69.2$\tiny$\pm0.31$&$79.4$\tiny$\pm0.07$\normalsize/ $78.2$\tiny$\pm0.10$&$70.9$\tiny$\pm0.44$\normalsize/ $67.0$\tiny$\pm1.62$ \\
    \multicolumn{2}{c}{ReMixMatch*+DARP+LA}&$81.8$\tiny$\pm0.45$ \normalsize/ $80.9$\tiny$\pm0.40$&$83.9$\tiny$\pm0.42$ \normalsize/ $83.4$\tiny$\pm0.45$&$81.1$\tiny$\pm0.20$ \normalsize/ $80.3$\tiny$\pm0.26$&$80.6$\tiny$\pm0.45$\normalsize/ $79.6$\tiny$\pm0.55$&$76.8$\tiny$\pm0.60$\normalsize/ $74.8$\tiny$\pm0.68$ \\
    \multicolumn{2}{c}{ReMixMatch*+DARP+cRT}&$88.7$\tiny$\pm0.25$ \normalsize/ $88.5$\tiny$\pm0.25$&$83.5$\tiny$\pm0.53$ \normalsize/ $83.1$\tiny$\pm0.51$&$80.9$\tiny$\pm0.25$ \normalsize/ $80.3$\tiny$\pm0.31$&$80.9$\tiny$\pm0.53$\normalsize/ $80.0$\tiny$\pm0.46$&$76.7$\tiny$\pm0.50$\normalsize/ $74.9$\tiny$\pm0.70$ \\
    \multicolumn{2}{c}{ReMixMatch+ABC}&$76.4$\tiny$\pm5.34$ \normalsize/ $74.8$\tiny$\pm6.05$&$85.2$\tiny$\pm0.20$ \normalsize/ $84.7$\tiny$\pm0.25$&$80.4$\tiny$\pm0.40$ \normalsize/ $80.0$\tiny$\pm0.44$&$76.8$\tiny$\pm0.52$\normalsize/ $74.8$\tiny$\pm0.64$&$71.2$\tiny$\pm1.37$\normalsize/ $67.4$\tiny$\pm1.89$ \\
    \multicolumn{2}{c}{ReMixMatch*+SAW}&$87.0$\tiny$\pm0.75$ \normalsize/ $86.4$\tiny$\pm0.85$&$80.6$\tiny$\pm1.57$ \normalsize/ $79.2$\tiny$\pm2.19$&$77.6$\tiny$\pm0.76$ \normalsize/ $76.0$\tiny$\pm0.93$&-/-&-/-\\

    \multicolumn{2}{c}{ReMixMatch*+SAW+LA}&$74.2$\tiny$\pm1.49$ \normalsize/ $65.1$\tiny$\pm2.36$&$84.8$\tiny$\pm1.07$ \normalsize/ $82.4$\tiny$\pm2.32$&$81.3$\tiny$\pm2.42$ \normalsize/ $80.9$\tiny$\pm2.47$&-/-&-/- \\
    \multicolumn{2}{c}{ReMixMatch*+SAW+cRT}&$88.8$\tiny$\pm0.79$ \normalsize/ $88.6$\tiny$\pm0.83$&$84.5$\tiny$\pm0.78$ \normalsize/ $83.6$\tiny$\pm1.27$&$82.4$\tiny$\pm0.10$ \normalsize/ $82.0$\tiny$\pm0.10$&-/-&-/- \\
     
    \rowcolor{Gray}
    \multicolumn{2}{c}{ReMixMatch+CDMAD}&\textbf{89.9}\tiny$\pm0.45$ \normalsize/ \textbf{89.6}\tiny$\pm0.46$&\textbf{86.9}\tiny$\pm0.21$ \normalsize/ \textbf{86.7}\tiny$\pm0.17$&\textbf{83.1}\tiny$\pm0.46$ \normalsize/ \textbf{82.7}\tiny$\pm0.50$&\textbf{83.0}\tiny$\pm0.38$\normalsize/ \textbf{82.1}\tiny$\pm0.35$&\textbf{81.9}\tiny$\pm0.32$\normalsize/ \textbf{80.9}\tiny$\pm0.44$ \\
    \bottomrule
  \end{tabular}}
\vspace{-0.15in}
\end{table*}

\cref{cifar} summarizes bACC and GM of the baseline algorithms and proposed algorithm on CIFAR-10-LT when $\gamma_u$ is assumed to be known and equal to $\gamma_l$. We can first observe that the vanilla algorithm (Deep CNN trained with cross-entropy loss) performed the worst. CIL (Re-sampling, LDAM-DRW, and cRT) mitigated class imbalance but did not significantly improve the classification performance compared to the vanilla algorithm. These results demonstrate the importance of using the unlabeled set. Compared to the vanilla algorithm, SSL algorithms (FixMatch and ReMixMatch) significantly improved the classification performance. However, their lower performance than that of the CISSL algorithms highlights the importance of mitigating class imbalance. By mitigating class imbalance and leveraging unlabeled data, CISSL algorithms achieved higher performance than the other algorithms. Overall, the proposed algorithm outperformed all other algorithms. 
This may be because CDMAD effectively refined biased pseudo-labels and class predictions on the test set by considering the classifier's biased degree.%, as illustrated in Appendix J. 
%We can also observe that CISSL algorithms combined with ReMixMatch achieved better performance than when combined with FixMatch in most settings. This seems to be because the mix-up regularization of ReMixMatch slightly mitigates class imbalance \citep{zhang2017mixup}.
\begin{table}[htbp]
%\caption{Classification performance of the baseline algorithms and the proposed algorithm under $\gamma_l=100$ and $\gamma_u=100$ (reversed).}
\caption{Comparison of bACC/GM under $\gamma_l=\gamma_u=100$ (reversed).}
\vspace{-0.125in}
  
  \centering 
  \resizebox{3.2in}{!}{%
 \begin{tabular}{cccccc}
    \toprule     
    \multicolumn{6}{c}{CIFAR-10-LT, $\gamma_l=100$, $\gamma_u=100$ (reversed)} \\
    \midrule
    Algorithm & ABC & SAW & SAW+LA & SAW+cRT & \textbf{CDMAD} \\
    
    \midrule
    
    FixMatch+ &69.5/66.8&72.3/68.7&74.1/72.0&75.5/73.9&\textbf{77.1}/\textbf{75.4}\\
    ReMixMatch+ &63.6/60.5&79.5/78.5&50.2/14.8&80.8/79.9&\textbf{81.7}/\textbf{81.0}\\
    \bottomrule
  \end{tabular}}
\label{table:reversed}
\vspace{-0.1in}
\end{table}
%\begin{table}[htbp]

\begin{table}

  \caption{Comparison of bACC on CIFAR-100-LT. %Experiment results are copied from CoSSL \citep{fan2022CoSSL}.
  }
  \vspace{-0.1in}
  \label{cifar3}
  \centering
  \resizebox{3in}{!}{%
  \begin{tabular}{ccccc}
    \toprule
    \multicolumn{5}{c}{CIFAR-100-LT ($\gamma=\gamma_l=\gamma_u$, $\gamma_u$ is assumed to be known)}                   \\
    \midrule
    \multicolumn{2}{c}{Algorithm}   &$\gamma=20$&$\gamma=50$&$\gamma=100$\\
    \midrule
    \multicolumn{2}{c}{FixMatch}&$49.6$\tiny$\pm0.78$&$42.1$\tiny$\pm0.33$&$37.6$\tiny$\pm0.48$  \\
    \multicolumn{2}{c}{FixMatch+DARP }&$50.8$\tiny$\pm0.77$&$43.1$\tiny$\pm0.54$&$38.3$\tiny$\pm0.47$ \\
    \multicolumn{2}{c}{FixMatch+DARP+cRT }&$51.4$\tiny$\pm0.68$&$44.9$\tiny$\pm0.54$&$40.4$\tiny$\pm0.78$ \\
    \multicolumn{2}{c}{FixMatch+CReST }&$51.8$\tiny$\pm0.12$&$44.9$\tiny$\pm0.50$&$40.1$\tiny$\pm0.65$ \\
    \multicolumn{2}{c}{FixMatch+CReST+LA }&$52.9$\tiny$\pm0.07$&$47.3$\tiny$\pm0.17$&$42.7$\tiny$\pm0.70$ \\
    \multicolumn{2}{c}{FixMatch+ABC }&$53.3$\tiny$\pm0.79$&$46.7$\tiny$\pm0.26$&$41.2$\tiny$\pm0.06$ \\
    \multicolumn{2}{c}{FixMatch+CoSSL }&$53.9$\tiny$\pm0.78$&$47.6$\tiny$\pm0.57$&$43.0$\tiny$\pm0.61$ \\
    \multicolumn{2}{c}{FixMatch+UDAL }&-&$48.0$\tiny$\pm0.56$&$43.7$\tiny$\pm0.41$ \\
    
    \rowcolor{Gray}
    \multicolumn{2}{c}{FixMatch+CDMAD}&\textbf{54.3}\tiny$\pm0.44$&\textbf{48.8}\tiny$\pm0.75$&\textbf{44.1}\tiny$\pm0.29$ \\
    \midrule
    \multicolumn{2}{c}{ReMixMatch}&$51.6$\tiny$\pm0.43$&$44.2$\tiny$\pm0.59$&$39.3$\tiny$\pm0.43$ \\
    \multicolumn{2}{c}{ReMixMatch+DARP }&$51.9$\tiny$\pm0.35$&$44.7$\tiny$\pm0.66$&$39.8$\tiny$\pm0.53$ \\
    \multicolumn{2}{c}{ReMixMatch+DARP+cRT }&$54.5$\tiny$\pm0.42$&$48.5$\tiny$\pm0.91$&$43.7$\tiny$\pm0.81$ \\
    \multicolumn{2}{c}{ReMixMatch+CReST }&$51.3$\tiny$\pm0.34$&$45.5$\tiny$\pm0.76$&$41.0$\tiny$\pm0.78$ \\
    \multicolumn{2}{c}{ReMixMatch+CReST+LA }&$51.9$\tiny$\pm0.60$&$46.6$\tiny$\pm1.14$&$41.7$\tiny$\pm0.69$ \\
    \multicolumn{2}{c}{ReMixMatch+ABC }&$55.6$\tiny$\pm0.35$&$47.9$\tiny$\pm0.10$&$42.2$\tiny$\pm0.12$ \\
    \multicolumn{2}{c}{ReMixMatch+CoSSL }&$55.8$\tiny$\pm0.62$&$48.9$\tiny$\pm0.61$&$44.1$\tiny$\pm0.59$ \\
    
    \rowcolor{Gray}
    \multicolumn{2}{c}{ReMixMatch+CDMAD}&\textbf{57.0}\tiny$\pm0.32$&\textbf{51.1}\tiny$\pm0.46$&\textbf{44.9}\tiny$\pm0.42$ \\
    \bottomrule
  \end{tabular}}
\vspace{-0.2in}
\end{table}

\cref{cifar2} summarizes bACC and GM of the baseline algorithms and proposed algorithm on CIFAR-10-LT and STL-10-LT when $\gamma_u$ is unknown and different from $\gamma_l$. %The proposed algorithm performed the best in all settings probably because CDMAD can appropriately re-balance the biased classifier by implicitly considering the class distributions of both labeled and unlabeled sets, as discussed in Section \ref{analysis}.
%both class distributions of labeled set and unlabeled set were automatically considered, 
ReMixMatch performed poorly when $\gamma_l$ and $\gamma_u$ differed significantly, probably because the distribution alignment technique employed in ReMixMatch significantly degraded the quality of pseudo-labels. By aligning the class distribution of the pseudo-labels with that of the unlabeled set estimated as in \citet{NEURIPS2020_a7968b43}, ReMixMatch* significantly improved the classification performance. However, the estimation of the class distribution of the unlabeled set becomes more time-consuming as the amount of unlabeled data increases. Furthermore, the estimation process requires more than 10 labeled samples for each class, making it unsuitable for datasets with a very small number of labeled samples, such as CIFAR-100-LT. In contrast, CDMAD does not rely on the estimated class distribution of the unlabeled set, making it more effective than baseline algorithms combined with ReMixMatch* for real-world scenarios. We can also observe that the LA decreased the performance of ReMixMatch*+DARP and ReMixMatch*+SAW when $\gamma_l$ and $\gamma_u$ differed significantly. This may be because the LA considers only the class distribution of the labeled set for re-balancing the classifier when the class distribution of the unlabeled set is unknown.  
%when class distribution of the unlabeled set is unknown. 
%These results show the importance of re-balancing the classifier by considering the class distribution of the unlabeled set. 
In Appendix K, we present further comparisons of LA and CDMAD under the settings that the class distributions of labeled and unlabeled sets mismatch. %For the experiments on STL-10, we did not report the performance of SAW+cRT and SAW+LA because they were not reproducible as in \citet{lai2022smoothed}. 
%we were not able to reproduce the results in \citep{lai2022smoothed} at all.%For the experiments on STL-10 using SAW, we were not able to reproduce the results reported in \citep{lai2022smoothed} at all.% We thought that there might be our problem in the process of executing the uploaded code, so we brought the performance of SAW on STL-10 reported in \citep{lai2022smoothed}, and we did not measure performance of SAW+cRT and SAW+LA.
%\textcolor{red}{Table \ref{cifar3} summarizes bACC of the baseline algorithms and the proposed algorithm on CIFAR-100-LT. The proposed algorithm outperformed the baseline algorithms for all three cases of $\gamma$. These results show that the proposed algorithm can be suitable for CISSL on data sets consist of many classes. Considering that most minority classes have only one labeled sample in the training set of CIFAR-100-LT when $\gamma=100$, the results also show that the proposed algorithm outperformed the baseline algorithms even when there are very few labeled samples belonging to minority classes. This may be because CDMAD effectively compensates for the lack of the labeled samples by relatively well refining the biased pseudo-labels.}

We also conducted experiments under the setting that the class distribution of the unlabeled set is imbalanced in the opposite direction to the labeled set. From \cref{table:reversed}, we can observe that CDMAD outperforms the baseline algorithms.
%\vspace{-0.025in}

\begin{table}

  \caption{Comparison of bACC on Small-ImageNet-127 (size $32\times32$ and $64\times64$, $\gamma_u$ is assumed to be known)}
  %on Small-ImageNet-127 (Down-sized to $32 \times 32$ and $64 \times 64$) under $\gamma_{l}=\gamma_{u}$. 
  %Experiment results are copied from CoSSL \citep{fan2022CoSSL}.
  \vspace{-0.05in}
  \label{imagenet}
  \centering
  \resizebox{2.9in}{!}{%
  \begin{tabular}{cccccc}
    \toprule
    \multicolumn{6}{c}{Small-ImageNet-127 ($\gamma=\gamma_l=\gamma_u$, $\gamma_u$ is assumed to be known)}                   \\
    \midrule
    \multicolumn{2}{c}{Algorithm}  && &$32\times32$&$64\times64$\\
    \midrule
    \multicolumn{2}{c}{FixMatch}&&&$29.7$&$42.3$\\
    \multicolumn{2}{c}{FixMatch+DARP }&&&$30.5$&$42.5$ \\
    \multicolumn{2}{c}{FixMatch+DARP+cRT }&&&$39.7$&$51.0$ \\
    \multicolumn{2}{c}{FixMatch+CReST }&&&$32.5$&$44.7$ \\
    \multicolumn{2}{c}{FixMatch+CReST+LA }&&&$40.9$&$55.9$ \\
    \multicolumn{2}{c}{FixMatch+ABC }&&&$46.9$&$56.1$ \\
    \multicolumn{2}{c}{FixMatch+CoSSL }&&&$43.7$&$53.8$ \\
    \rowcolor{Gray}
    \multicolumn{2}{c}{FixMatch+CDMAD}&&&\textbf{48.4}&\textbf{59.3}\\
    \bottomrule
  \end{tabular}}
\vspace{-0.25in}
\end{table}

\cref{cifar3} summarizes bACC of the baseline algorithms and the proposed algorithm on CIFAR-100-LT. The proposed algorithm outperformed baseline algorithms. % for all experiments on CIFAR-100-LT. 
These results demonstrate that the proposed algorithm is well-suited for CISSL on datasets with a large number of classes. %, considering that CIFAR-100 contains 100 classes. 
Moreover, given that %the most minor
several minority classes in the training set have only one labeled sample when $\gamma=100$, the results indicate that the proposed algorithm may outperform the baseline CISSL algorithms when the number of labeled samples from minority classes is extremely limited. This may be because CDMAD effectively compensates for the lack of labeled samples by well refining the biased pseudo-labels compared to the baseline algorithms.

\cref{imagenet} summarizes bACC of the baseline algorithms on Small-ImageNet-127. For both sizes of Small-ImageNet-127, CDMAD outperformed the baseline algorithms by a large margin. %larger margin compared to when tested on other datasets. 
%These results show that the proposed algorithm is suitable for CISSL on large-scale datasets. We thought that this might be because the proposed algorithm used a large number of unlabeled samples appropriately by refining the pseudo-labels. 
The effective use of unlabeled samples through appropriate refinement of the pseudo-labels may allow the proposed algorithm suitable for CISSL on large-scale datasets. 
Given that the test set of Small-ImageNet-127 is class-imbalanced, the results also show that CDMAD can be suitable for CISSL with a class-imbalanced test set. 
%We can also observe that the proposed algorithm performed better to a higher extent compared to baseline algorithms when the size of each image is 64 x 64 size. This result shows the possibility that the proposed algorithm may perform better when the algorithm is trained on the original ImageNet-127 or other data sets with large image size.

\begin{table}[htbp]
\vspace{-0.05in}
   \caption{Comparison of bACC/GM on CIFAR-10-LT with FreeMatch as the base SSL algorithm}

  \vspace{-0.1in}
  \label{freematch}
  \centering
  \resizebox{2.9in}{!}{%
  \begin{tabular}{ccc}
    \toprule
    \multicolumn{3}{c}{CIFAR-10-LT}                   \\
    \midrule
    Algorithm  &$\gamma_l=\gamma_u=100$&$\gamma_l=100$, $\gamma_u=1$\\
    \midrule
    FreeMatch&75.4/72.9&74.2/69.5\\
    FreeMatch+SAW+cRT&82.8/82.3&86.4/86.2 \\
    \rowcolor{Gray}
    FreeMatch+CDMAD&\textbf{84.8}/\textbf{84.4}&\textbf{89.0}/\textbf{88.7}\\
    \bottomrule
  \end{tabular}}
\vspace{-0.15in}
\end{table}

To verify that CDMAD can be also effectively combined with recent SSL algorithms, we conducted experiments by setting FreeMatch \cite{wang2023freematch} as the base SSL algorithm. % in place of FixMatch and ReMixMatch
% Experiments were conducted on CIFAR-10-LT under the settings that the class distributions of the labeled and unlabeled sets match ($\gamma_l=\gamma_u=100$) and significantly mismatch ($\gamma_l=100$ and $\gamma_u=1$). We compared the classification performance of FreeMatch+CDMAD with FreeMatch and FreeMatch+SAW+cRT. %Experimental results are summarized in \cref{freematch}. 
From \cref{freematch}, we can observe that CDMAD outperforms FreeMatch and FreeMatch+SAW+cRT. 

%CDMAD significantly improves classification performance of FreeMatch when the training set is class-imbalanced. In \cref{freematch}, we can also observe that FreeMatch+CDMAD outperforms FreeMatch+SAW+cRT. Considering that SAW+cRT is a strong CISSL baseline algorithm, the results verify the effectiveness of FreeMatch+CDMAD.
We also compared the classification performance of CDMAD with ACR \cite{wei2023towards}, a recent CISSL algorithm. From \cref{table:acr}, we can observe that CDMAD outperforms ACR.
\begin{table}[htbp]
\vspace{-0.1in}
  \caption{Comparison of bACC/GM on CIFAR-10-LT }
  \vspace{-0.1in}
  \centering 
  \resizebox{2.7in}{!}{%
 \begin{tabular}{ccc}
    \toprule     
    Algorithm/ CIFAR-10-LT &$\gamma_l=\gamma_u=100$&$\gamma_l=100$, $\gamma_u=1$ \\
    
    \midrule
    
    FixMatch+ACR &81.8/81.4&85.6/85.3\\
    \textbf{FixMatch+CDMAD} &\textbf{83.6/83.1}&\textbf{87.5/87.1} \\
    \bottomrule
      \vspace{-0.5in}
  \end{tabular}}
\label{table:acr}
%\vspace{-0.3in}
\end{table}
\begin{figure*}[htbp]

 \begin{center}
        \begin{tabular}{cccc}
			 \includegraphics[width=3.6cm, height=3.6cm]{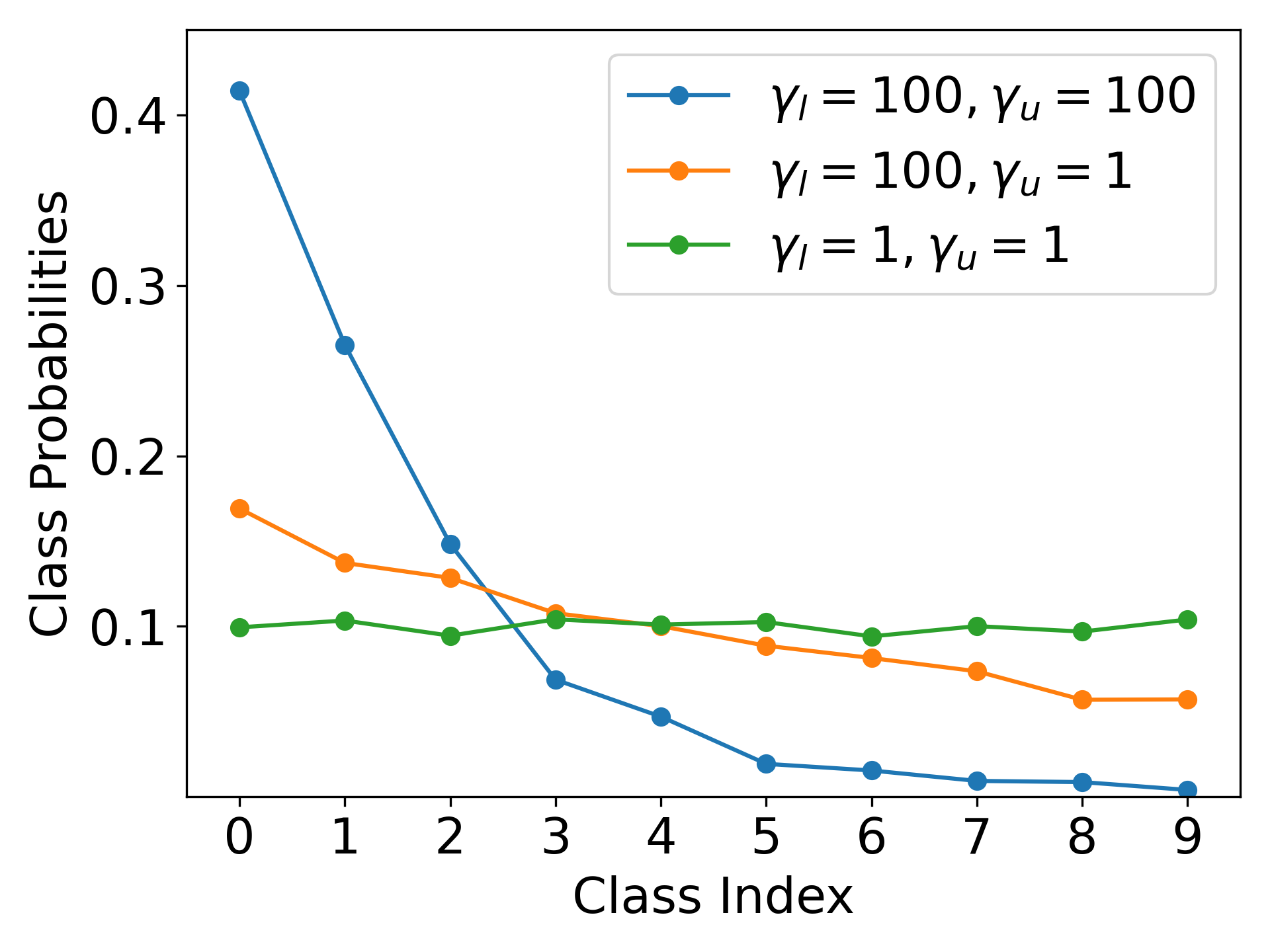}&\includegraphics[width=3.6cm, height=3.6cm]{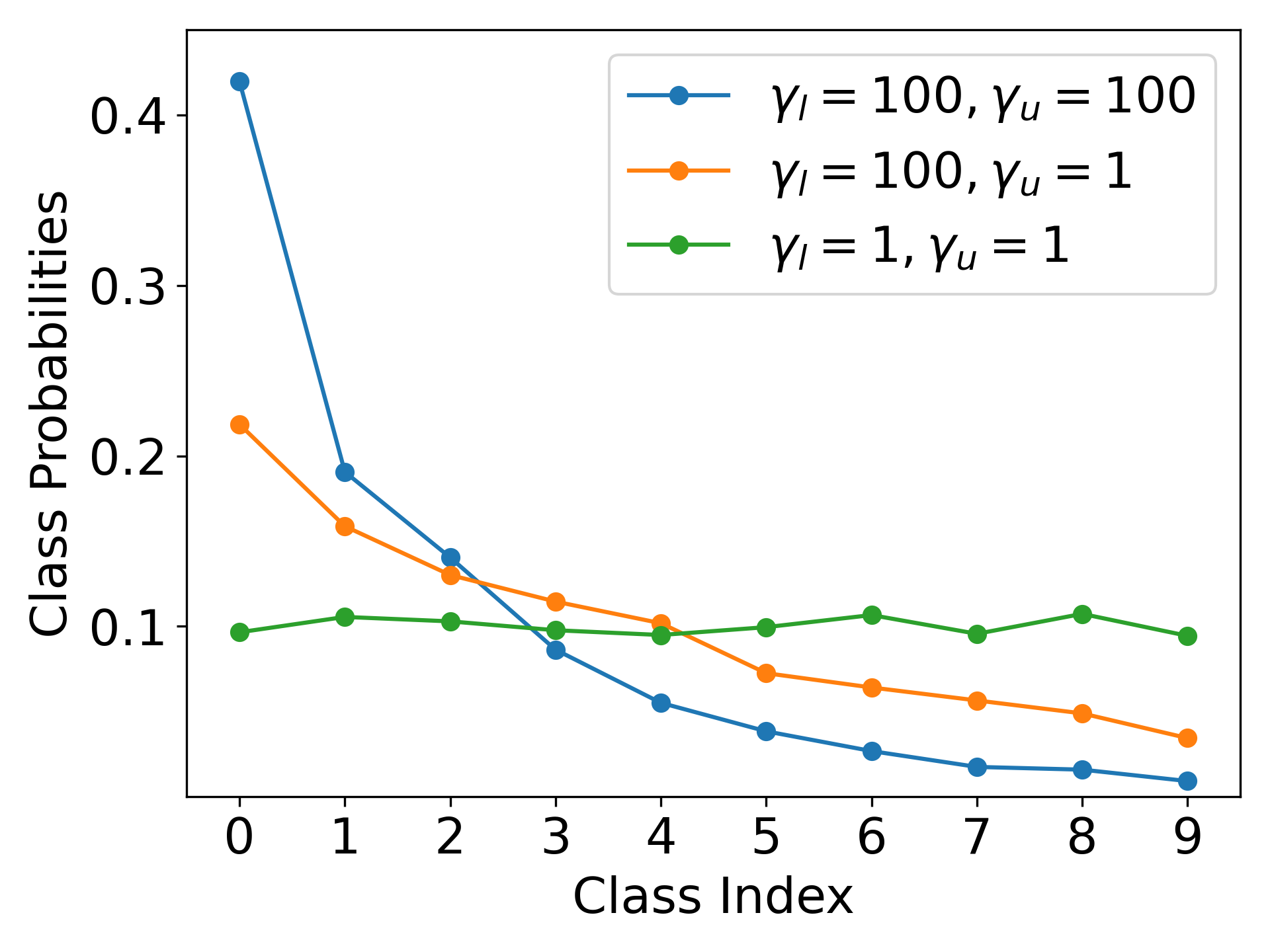}&\includegraphics[width=3.6cm, height=3.6cm]{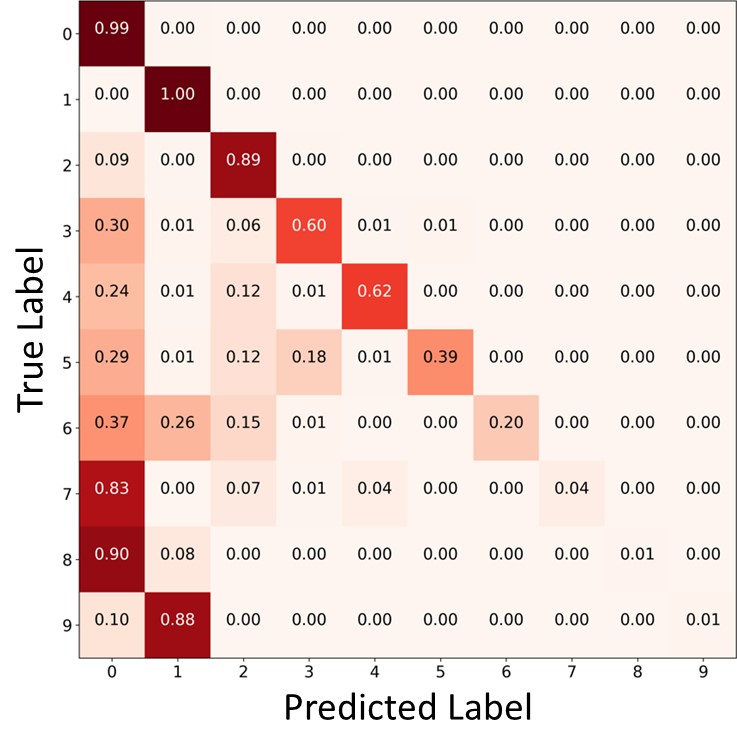}&\includegraphics[width=3.6cm, height=3.6cm]{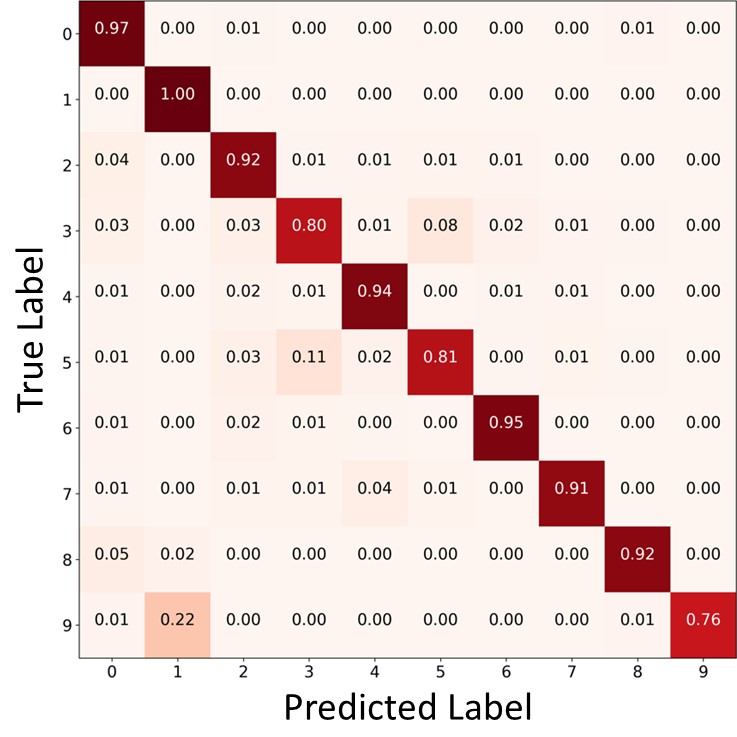} \\\
			  \footnotesize{(a) FixMatch+CDMAD} &\footnotesize{(b) ReMixMatch+CDMAD} &\footnotesize{(c) ReMixMatch} &\footnotesize{(d) ReMixMatch+CDMAD} 
		\end{tabular}
	\end{center}
    \vspace{-0.2in}
	\caption{(a) and (b) present the class probabilities predicted on %a non-image input
 a white image using the proposed algorithm. (c) and (d) present the confusion matrices of the class predictions on test samples.}
 %image using FixMatch+CDMAD and ReMixMatch+CDMAD trained on CIFAR-10-LT under the three settings of $\gamma_l$ and $\gamma_u$.} % $\gamma_{l}=100,\gamma_u=100$, $\gamma_{l}=100,\gamma_u=1$, and $\gamma_{l}=1,\gamma_u=1$.}
	\label{biaseddegree}
\vspace{-0.1in}
\end{figure*}
\begin{table*}[htbp]

  \caption{Ablation study for the proposed algorithm on CIFAR-10-LT under $\gamma_l=100$ and $\gamma_u=1$}
  %on CIFAR-10-LT under $\gamma_l=100$, $\gamma_u=1$. Each row represents bACC/GM of the proposed algorithm under the conditions specified in that row.}
  \vspace{-0.1in}
 \label{ablationtable}
  \centering 
  \resizebox{6.5in}{!}{%
  
  \begin{tabular}{lclc}
    \toprule     
    \textbf{Ablation study ($\gamma_l=100$, $\gamma_u=1$)}&\textbf{bACC/GM}&&\textbf{bACC/GM}\\
    \midrule
    \textbf{FixMatch+CDMAD}&$\mathbf{87.5/87.1}$&\textbf{ReMixMatch+CDMAD}&$\mathbf{89.9/89.6}$ \\
    \midrule
    Without CDMAD for refining pseudo-labels&$78.2/75.8$& Without CDMAD for refining pseudo-labels&$72.3/65.9$ \\
    Without CDMAD for test phase &$84.9/84.1$&Without CDMAD for test phase&$88.2/87.7$ \\
    With the use of hard pseudo-labels &$86.7/86.3$&With the use of sharpened pseudo-labels &$88.9/88.6$\\
    %With the replacement of white image by random image &$77.2/74.7$\\
    With the use of confidence threshold $\tau=0.95$ &$86.8/86.3$& With the use of distribution alignment technique &$80.4/78.5$\\
    \bottomrule
  \end{tabular}}
	\vspace{-0.2in}
\end{table*} 

We present additional experimental results in Appendix. Specifically, fine grained results (many/medium/few group performance) are summarized in Appendix L. %Experimental results under the setting that the class distribution of the unlabeled set is imbalanced in the opposite direction to the labeled set are summarized in Appendix M. 
In Appendix M, we compare the classification performance of CDMAD with DASO \citep{oh2022daso} whose classification performance was measured under different settings with the settings of ours.
%whose performances were reported in their respective papers under different settings from ours. %slightly different settings compared to ours.

%\vspace{-0.025in}
\subsection{Qualitative analyses}
\label{analysis}

%\vspace{-0.025in}
%To verify that the proposed algorithm can mitigate class imbalance to an appropriate degree even when $\gamma_u$ is unknown and different with $\gamma_l$ we qualitatively analyze CDMAD. We first analyze class priors approximated by the original LA \citep{menon2020long} and CDMAD trained on CIFAR-10-LT, which determines the degree of classifier re-balancing. Figure \ref{prior} presents the approximated priors under $\gamma_l=100$, $\gamma_u=100$ and $\gamma_l=100$, $\gamma_u=1$ (unknown $\gamma_u$).

%As we described in Section \ref{intro}, the original LA approximates the $P\left(y\right)$ by class distribution of labeled set $P\left(y\right)$. Therefore, the original LA can not approximate the $P\left(y\right)$ differently depending on the class distribution of unlabeled set and it could degrade classification performance by failing to re-balance the classifier to an appropriate degree when $\gamma_u$ greatly differs from $\gamma_l$, as we saw in the Table \ref{cifar2}. On the other hand, as we can observe in Figure \ref{prior} (a) and (b), CDMAD approximated $P\left(y\right)$ much less steeply when $\gamma_u$ is 1 than when $\gamma_u$ is 100 using $P_{\theta}\left(y|\mathcal{I}hite\right)$ as an approximated class prior. This verifies that CDMAD can change the degree of re-balancing classifier considering $\gamma_u$.

We argue that the CDMAD can implicitly consider the class distributions of both labeled and unlabeled sets when measuring the classifier's biased degree. To verify this argument, in \cref{biaseddegree} (a) and (b), we analyze the class probabilities predicted on a white image,  $P_{\theta}\left(y|\mathcal{I}\right)$, using FixMatch+CDMAD and ReMixMatch+CDMAD trained on CIFAR-10-LT under the three settings: 1) $\gamma_{l}=\gamma_u=100$, 2) $\gamma_{l}=100$ and $\gamma_u=1$, and 3) $\gamma_{l}=\gamma_u=1$.
 
We can observe that both FixMatch+CDMAD and ReMixMatch+CDMAD produced highly nonuniform class probabilities when they were trained under $\gamma_{l}=100$ and $\gamma_u=100$. In contrast, when trained with $\gamma_{l}=100$ and $\gamma_u=1$, both algorithms produced significantly more balanced class probabilities. These results show that the classifier's biased degree, $g_{\theta}\left(\mathcal{I}\right)$, depends on the class distribution of the unlabeled set. Moreover, the comparison of nearly uniform class probabilities produced under $\gamma_{l}=\gamma_u=1$  and the results under $\gamma_{l}=100$ and $\gamma_u=1$ shows that $g_{\theta}\left(\mathcal{I}\right)$ also depends on the class distribution of the labeled set. Based on the above findings, CDMAD can be considered as measuring the classifier's biased degree by implicitly incorporating the class distributions of both labeled and unlabeled sets. It is worth noting that under $\gamma_l=100$ and $\gamma_u=1$, FixMatch+CDMAD and ReMixMatch+CDMAD produced significantly more balanced class probabilities compared to FixMatch and ReMixMatch in \cref{prior}. This may be because the use of biased pseudo-labels generated by FixMatch and ReMixMatch for training exacerbated class imbalance, whereas CDMAD effectively refined the biased pseudo-labels, as discussed in Figure 6 of Appendix J. 

We also argue that the ability of CDMAD to implicitly incorporate the class distributions of both labeled and unlabeled sets enables it to effectively mitigate class imbalance even under severe class distribution mismatch. % mismatch between the labeled and unlabeled sets. 
To verify this argument, we present the confusion matrices of the class predictions on the test set of CIFAR-10 using ReMixMatch and ReMixMatch+CDMAD trained on CIFAR-10-LT under $\gamma_l=100$ and $\gamma_u=1$ in \cref{biaseddegree} (c) and (d). The value in the $i$th row and $j$th column represents the proportion of the $i$th class samples classified as the $j$th class.
We can observe that the class predictions of ReMixMatch in \cref{biaseddegree} (c) are biased towards the majority classes. Specifically, the data points in the minority classes (e.g., classes 7, 8 and 9) are often misclassified into the majority classes (e.g., classes 0 and 1). In contrast, ReMixMatch+CDMAD in \cref{biaseddegree} (d) made nearly balanced class predictions. Further qualitative analyses are presented in Appendix J. 

\subsection{Ablation study}
\label{ablation}
%\vspace{-0.05in}

%To investigate the effect of each element of the proposed algorithm when the class distribution 
%of the unlabeled set is unknown and 
%significantly differs from that of the labeled set, 
To investigate the effectiveness of each element of CDMAD, we conducted an ablation study using CIFAR-10-LT ($\gamma_l=100$ and $\gamma_u=1$, $\gamma_u$ is assumed to be unknown). Each row in \cref{ablationtable} represents the proposed algorithm under the condition specified in that row. The results are as follows: 1) Without the refinement of the biased pseudo-labels using CDMAD in the training phase, the classification performance significantly decreased. %This seems to be because the quality of representations is lowered by using many biased pseudo-labels for training. 
2) Without the refinement of the biased class predictions on the test set using CDMAD, the classification performance decreased. %This seems to be because the biased logits can not be adjusted in the test phase. 
3) With entropy minimization (using hard pseudo-labels and sharpened pseudo-labels) of the class predictions during training, the classification performance slightly decreased. 
%This is probably because minimizing the entropy of class predictions may train the classifier to be biased toward certain classes \citep{lee2021abc}. 
%4) \textcolor{red}{By using the logits predicted on \textcolor{red}{random image from a normal distribution} as the relative preferences, the classification performance significantly decreased. This may be because when the mean and standard deviation of the Normal distribution are similar to those of a specific class, the random image becomes relevant to the class, which makes relative preference as inappropriate for correcting bias. The result verifies the importance of finding an image irrelevant with training set. We intend to study about generating image irrelevant to given training set in the future.} 
4) For FixMatch+CDMAD, the use of confidence threshold $\tau=0.95$ slightly decreased the classification performance. 
%This seems to be because unlabeled samples with confidence lower than $\tau$ were not used for training. 
5) For ReMixMatch+CDMAD, the classification performance significantly decreased by using the distribution alignment technique instead of adding the supervised loss on $\alpha\left(x^m_b\right)$ for training. %The results in 1) to 6) verifies that all components of the proposed algorithm are effective. 
%These results show that each element of CDMAD contributes to performance improvement.
These results indicate that every element of CDMAD enhances performance.

%\textcolor{red}{In Appendix XX, we conducted additional experiments by replacing the white image used as $\mathcal{I}$ in Section \ref{exp} with various other images. Specifically, we conducted experiments with using other solid color image or an image consisting of random pixel values that are generated from uniform, Bernoulli, and normal distributions as $\mathcal{I}$. The results show that other solid color image can be also used as $\mathcal{I}$.}

\begin{table}[htbp]
\vspace{0 in}
  \caption{Experiments with the replacement of $\mathcal{I}$ by other inputs}
\vspace{-0.05 in}
  \label{replace2}
  \centering
  \resizebox{2.9in}{!}{%
  \begin{tabular}{cccc}
    \toprule
    \multicolumn{2}{c}{ReMixMatch+CDMAD}&\multicolumn{2}{c}{CIFAR-10-LT}                   \\
    \midrule
    \multicolumn{2}{c}{Input}   &$\gamma_l=\gamma_u=100$&$\gamma_l=100$, $\gamma_u=1$\\
    \midrule
    \multicolumn{2}{c}{Uniform}&$81.3$/ $80.7$&$85.3$/ $84.2$ \\
    \multicolumn{2}{c}{Bernoulli}&$82.5$/ $82.0$&$83.6$/ $82.8$ \\
    \multicolumn{2}{c}{Normal}&$78.4$/ $77.5$&$84.0$/ $83.2$ \\
    \multicolumn{2}{c}{Black}&$84.8$/ $84.5$&$89.3$/ $89.0$ \\
    \multicolumn{2}{c}{Red}&$84.8$/ $84.6$&$90.1$/ $89.9$ \\
    \multicolumn{2}{c}{Green}&$84.9$/ $84.6$&$89.3$/ $88.9$ \\
    \multicolumn{2}{c}{Blue}&$84.9$/ $84.7$&$90.2$/ $89.9$ \\
    \multicolumn{2}{c}{Gray}&$85.1$/ $84.9$&$89.6$/ $89.3$ \\
    \multicolumn{2}{c}{White}&$85.5$/ $85.3$&$89.9$/ $89.6$ \\
    \bottomrule
  \end{tabular}}
  \vspace{-0.25 in}
\end{table}

To explore whether the classifier's biased degree can be measured using other images rather than the white image, we conducted experiments by replacing $\mathcal{I}$ with other solid color images or an image consisting of random pixel values that are generated from uniform, Bernoulli, and normal distributions. %Specifically, we used black, red, green, blue, and gray color images as solid color images. To generate an image consisting of random pixel values, we randomly sampled each of the R, G, and B values for each pixel from three distributions: 1) discrete uniform distribution over the integers from 0 to 255, 2) Bernoulli distribution with the same probabilities for the outcomes 0 and 255, and 3) normal distribution with mean 128 and standard deviation 16. 
Experimental results are summarized in \cref{replace2}. From the table, we can observe that the classification performance of ReMixMatch+CDMAD decreased when $\mathcal{I}$ was replaced by images consisting of random pixels. This may be because the parameters of the distributions (e.g., mean and standard deviation of a normal distribution) used to generate random pixels may be related to specific classes. In contrast, the classification performance of ReMixMatch+CDMAD did not significantly change when $\mathcal{I}$ was replaced by images of other solid color images. These results show that other solid color images can also be used to measure the classifier's biased degree.

\begin{table}[htbp]
\vspace{-0.125in}
  \caption{Experiments with replacing $\mathcal{I}$ by non-image input}
  \vspace{-0.075in}
  \centering 
  \resizebox{3.3in}{!}{%
 \begin{tabular}{cccc}
    \toprule     
    Algorithm&CIFAR-10-LT &$\gamma_l=\gamma_u=100$&$\gamma_l=100$, $\gamma_u=1$ \\
    \midrule
    
    \multirow{2}{*}{FixMatch+CDMAD} &White image&83.6/83.1&87.5/87.1\\
   % \cline{2-4}
    &Non-image&84.0/83.6&87.4/87.0\\
    \midrule
    \multirow{2}{*}{ReMixMatch+CDMAD} &White image&85.5/85.3&89.9/89.6\\
   % \cline{2-4}
     &Non-image&85.6/85.4&89.8/89.6\\
    
    \bottomrule
  \end{tabular}}
\label{table:nonimage}
\vspace{-0.125in}
\end{table}

However, the assumption that a solid color image is non-informative for the class labels may fail when the classification of images is related to their color. To address this concern, we additionally considered an image with pixel values that are outside the range [0, 255] (actually, \textit{non-image input}) to replace the solid color image $\mathcal{I}$. For example, in the case of CIFAR-10-LT, the maximum value of each (R,G,B) channel becomes (2.06, 2.12, 2.11) after input normalization. In this case, if we generate an input with every pixel's (R,G,B) values set to (3, 3, 3), the input would not be associated with a specific class because it is actually not an image. By replacing $\mathcal{I}$ with this non-image input, we conducted experiments and presented the results in \cref{table:nonimage}. We observe that CDMAD with the non-image input performs comparably with the white image. These results verify that the non-image input can be effectively used for measuring the classifier's biased degree, overcoming the challenge of  finding data that are non-informative for the class labels. 

\vspace{-0.05in}
\section{Conclusion}
\vspace{-0.025in}
%We proposed CDMAD which utilizes classifier's biased degree to appropriately mitigate the class imbalance in SSL under class distribution mismatch.
We proposed CDMAD, which considers the classifier's biased degree towards each class to appropriately mitigate the class imbalance in SSL even under severe class distribution mismatch between the labeled and unlabeled sets. Using CDMAD, we refined biased pseudo-labels as well as biased class predictions on test samples. Experiments on four benchmark datasets show that the proposed algorithm outperforms the existing CISSL algorithms. Moreover, the qualitative analysis and ablation study on the proposed algorithm demonstrate the effectiveness of each component of CDMAD. In this paper, we used a solid color image to measure the classifier's biased degree, which lacks a firm theoretical basis. In future research, we plan to establish a theoretical foundation for utilizing the solid color image to measure the classifier's biased degree. 

\textbf{Acknowledgements} This research was supported by the National Research Foundation of Korea (NRF) grant funded by the Korea government (MSIT) (2023R1A2C2005453, RS-2023-00218913).

{
    \small
    \bibliographystyle{ieeenat_fullname}
    \bibliography{main}
}

\newpage
\appendix

\definecolor{Gray}{gray}{0.9}

%\maketitlesupplementary
\appendix
\def\thesection{\Alph{section}}

\section{Core part of the code for CDMAD}
\label{examplecode}
\begin{figure}[htbp]
\vspace{-0.1 in}
	\begin{center}
        \includegraphics[width=3.3in,height=1.25in]{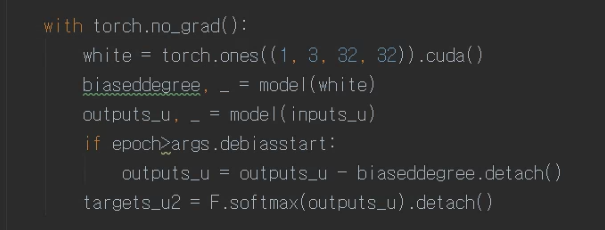}
      
	\end{center}
 \vspace{-0.1 in}
	\caption{Code for refining pseudo-labels using CDMAD}
	\label{fig:examplecode}
 \vspace{-0.1 in}
\end{figure}
\cref{fig:examplecode} presents a core part of the code for CDMAD to refine the biased pseudo-labels of the base SSL algorithm. As we can see in \cref{fig:examplecode}, CDMAD is very easy to implement. We simply need to calculate the logits for an image without any patterns (solid color image) and then subtract them from the logits for unlabeled samples. Biased class predictions on test samples are refined in a similar way. %All we need to do is to predict class probabilities on an image without any patterns (solid color image) and divide them from the soft pseudo-labels. Biased class predictions on test samples are refined in a similar way.

\section{Further related works}
\label{related} 
\textbf{Semi-supervised learning (SSL)} %has been actively studied for a relatively long time.%Many SSL algorithms use unlabeled data to place decision boundaries in low-density regions. 
algorithms use unlabeled data for training when labeled samples are insufficient. Entropy minimization \citep{NIPS2004_96f2b50b} encourages the class predictions on unlabeled samples to be confident by directly minimizing entropy or using pseudo-labels \citep{lee2013pseudo}. Consistency regularization \citep{park2018adversarial,miyato2018virtual,NIPS2017_68053af2} encourages the class predictions on two augmented versions of an unlabeled sample to be consistent. FixMatch \citep{sohn2020fixmatch} and ReMixMatch \citep{berthelot2019remixmatch} conduct entropy minimization and consistency regularization using strong data augmentation techniques \citep{devries2017improved,cubuk2020randaugment}. ReMixMatch also conducts Mixup regularization %which were used in previous SSL algorithms 
\citep{ijcai2019-504,NEURIPS2019_1cd138d0} and self-supervised learning with rotation \citep{gidaris2018unsupervised}. CoMatch \citep{li2021comatch} proposed graph-based contrastive learning using embedding and pseudo-label graphs. Recently, curriculum pseudo-labeling that considers the learning status for each class was proposed by FlexMatch \citep{zhang2021flexmatch} and extended in Adsh \cite{guo2022class}, SoftMatch \cite{chen2023softmatch} and FreeMatch \cite{wang2023freematch}.
%We describe FixMatch \citep{sohn2020fixmatch} and ReMixMatch \citep{berthelot2019remixmatch} in more detail in Section \ref{fixremix} because of their effectiveness and relevance with existing CISSL algorithms.

\textbf{Class-imbalanced learning (CIL)} algorithms mitigate class imbalance to improve classification performance for minority classes. Resampling techniques \citep{barandela2003restricted,chawla2002smote,japkowicz2000class,he2009learning} balance the number of each class samples, and reweighting techniques \citep{NIPS2013_9aa42b31,NIPS2017_147ebe63,huang2016learning,cui2019class,jamal2020rethinking} balance the loss for each class. %It is known that these techniques exhibit drawbacks such as overfitting, information loss, and unstable training \citep{NEURIPS2019_621461af,an2021why}. To alleviate these drawbacks, 
%\citet{cui2019class} proposed a reweighted loss based on an effective number of samples, and \citet{ren2018learning} and \citet{jamal2020rethinking} proposed meta-learning based reweighting. 
\citet{NEURIPS2019_621461af} and \citet{NEURIPS2020_2ba61cc3} proposed losses that minimize a generalization error bound, and \citet{kim2020m2m,yin2018feature} transferred knowledge from the data of the majority classes to the data of minority classes. \citet{kang2019decoupling} decoupled representation and classifier learning. \citet{menon2020long} proposed post-hoc logit-adjustment and loss, which is Fisher consistent for minimizing the balanced error. Recently, CIL algorithms based on contrastive learning \citep{kang2020exploring,kang2020exploring,jiang2021self,wang2021contrastive,cui2021parametric,li2022targeted} and multi-expert learning \citep{zhou2020bbn,xiang2020learning,cai2021ace,wang2021longtailed,li2022trustworthy,zhang2022self}  received considerable attention.

\section{Data augmentation techniques}
\label{augment} 
\begin{figure*}[htbp]
\vspace{-0.05 in}	
 \begin{center}
		\begin{tabular}{cccc}
	\hspace{-0.2 in}		 \includegraphics[width=4cm, height=4cm]{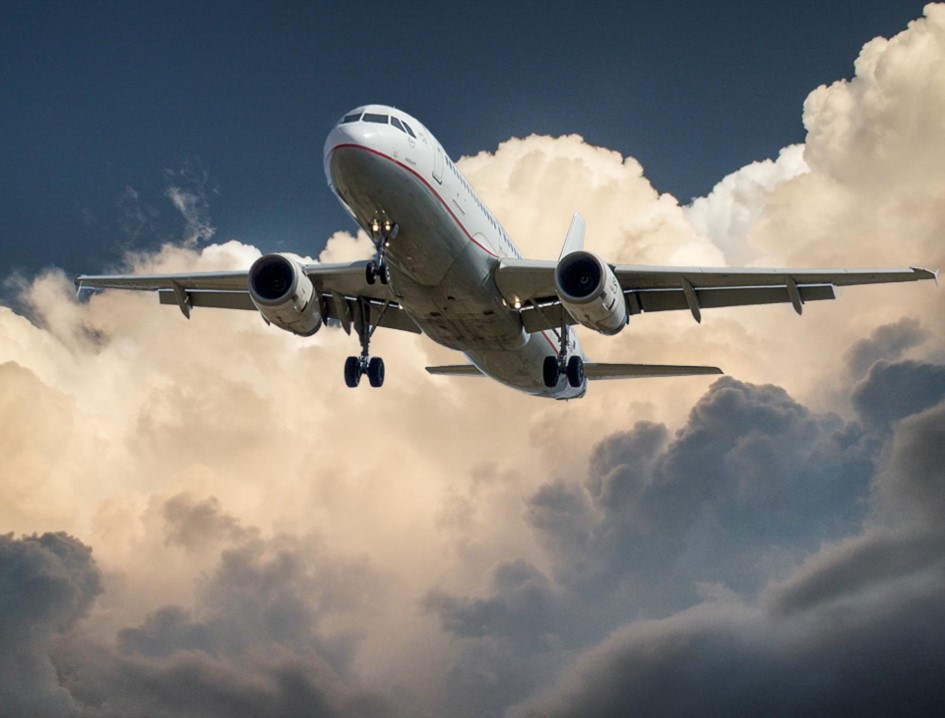}& \includegraphics[width=4cm, height=4cm]{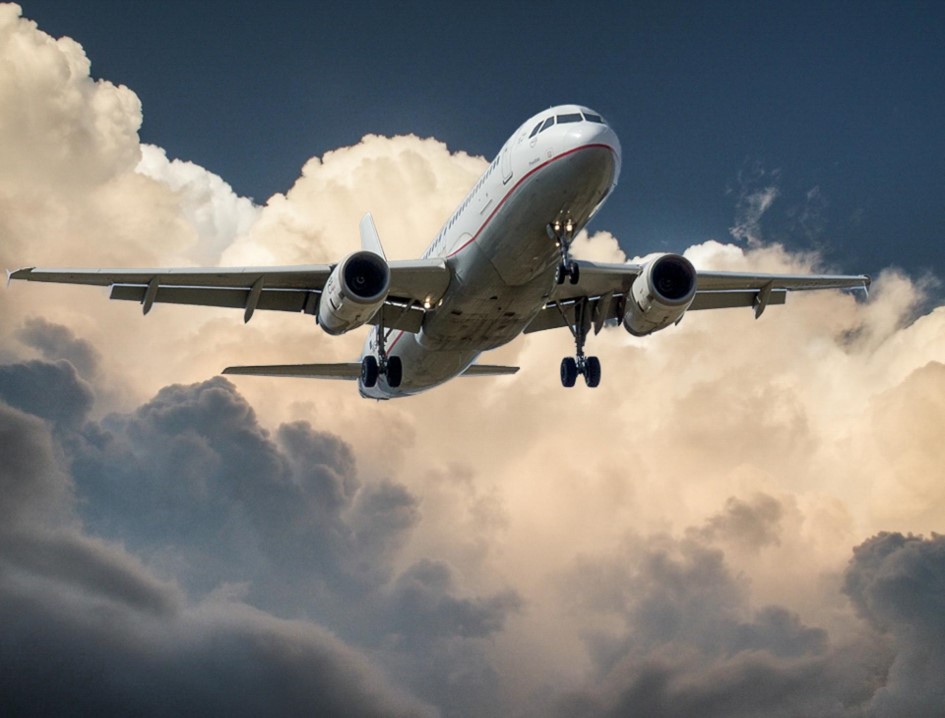}&\includegraphics[width=4cm, height=4cm]{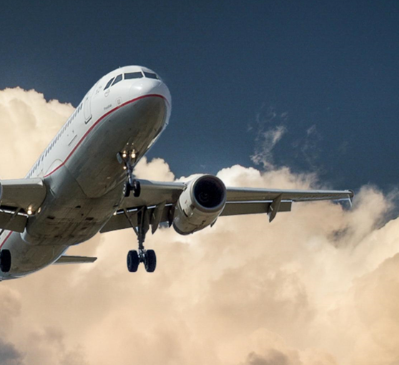}& \includegraphics[width=4cm, height=4cm]{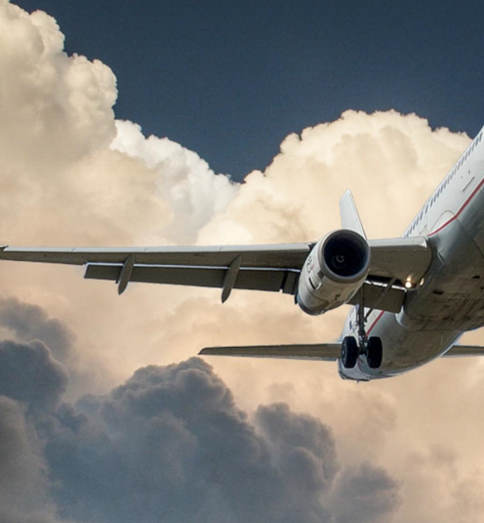} \\\
	\hspace{-0.2 in}		  \footnotesize{(a) Random horizontal flipping 1} &\footnotesize{(b) Random horizontal flipping 2}&\footnotesize{(c) Random cropping 1} &\footnotesize{(d) Random cropping 2}\\\
\hspace{-0.2 in}     \includegraphics[width=4cm, height=4cm]{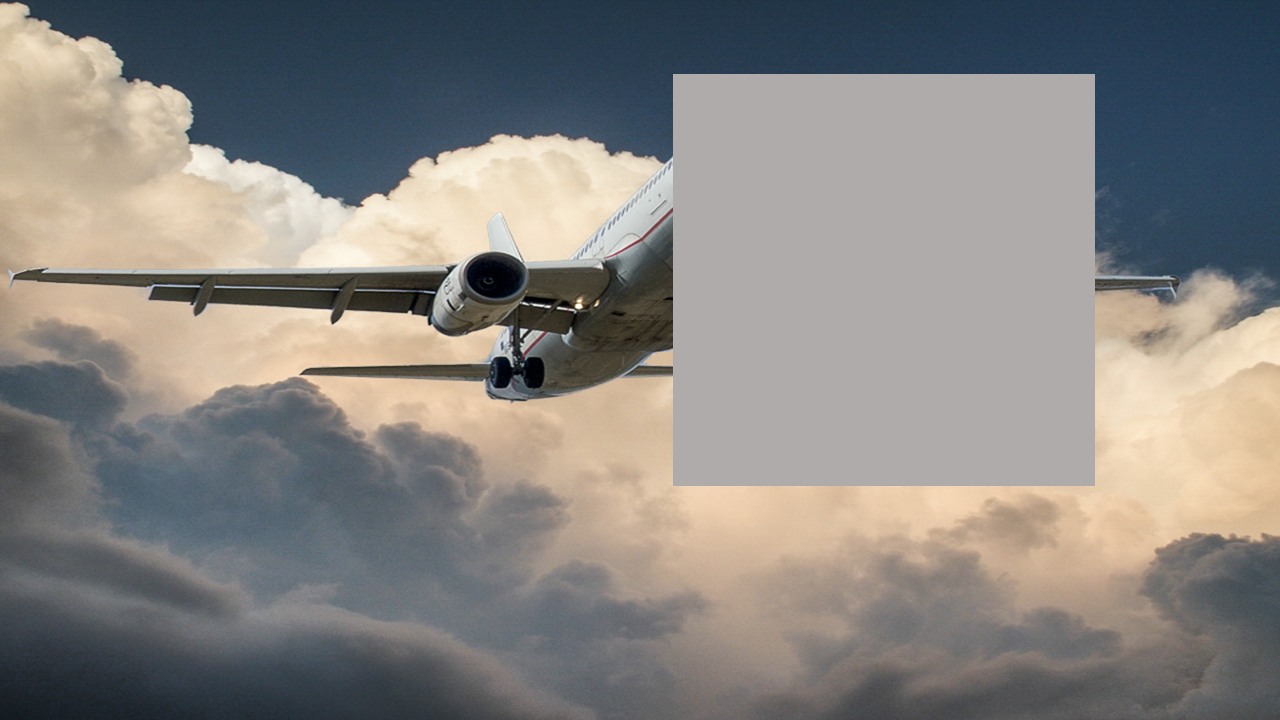}& \includegraphics[width=4cm, height=4cm]{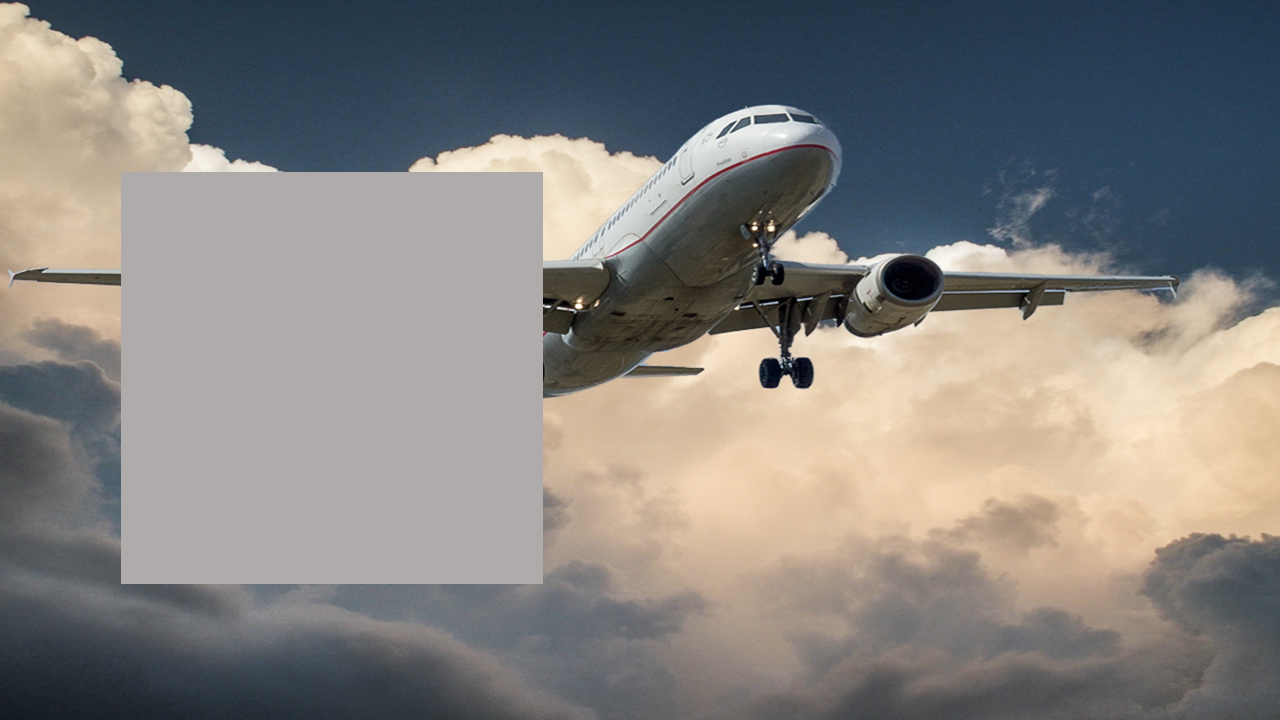}&\includegraphics[width=4cm, height=4cm]{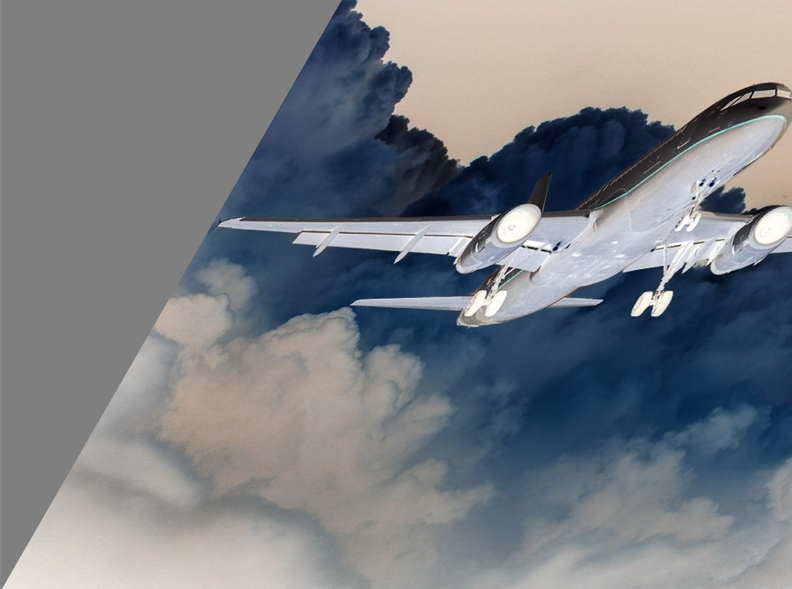}& \includegraphics[width=4cm, height=4cm]{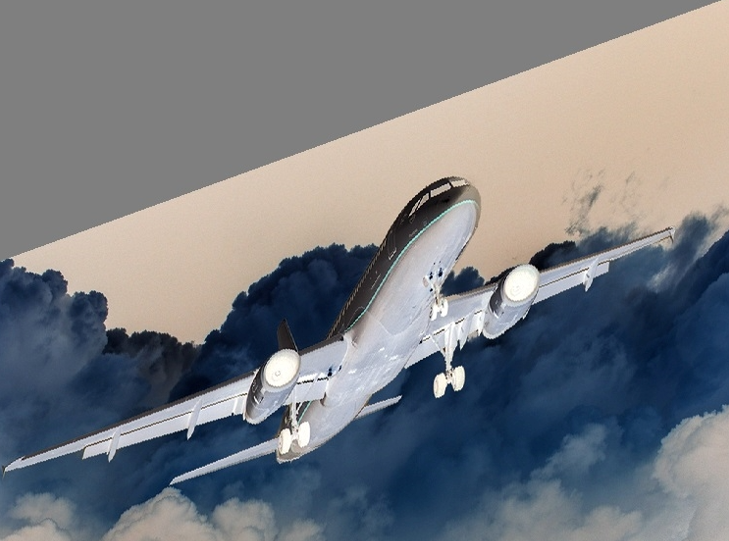} \\\
 \hspace{-0.2 in}    \footnotesize{(e) Cutout 1} &\footnotesize{(f) Cutout 2}&\footnotesize{(g) RandomAugment 1} &\footnotesize{(h) RandomAugment 2}
		\end{tabular}
	\end{center}
 \vspace{-0.05 in}
	\caption{Example images augmented using each data augmentation technique}
	\label{fig:augmentations}
\vspace{-0.05 in}
\end{figure*}
CDMAD uses data augmentation techniques utilized in FixMatch, ReMixMatch, and previous CISSL algorithms. Specifically, CDMAD uses random horizontal flipping and random cropping as weak data augmentation techniques and  uses Cutout \citep{devries2017improved} and RandomAugment \citep{cubuk2020randaugment} as strong data augmantation techniques. Random horizontal flipping and cropping flips and crops images, respectively. We implemented these weak data augmentation techniques using torchvision.transforms library. Cutout randomly masks out the square region of the image during training, which prevents the network from focusing on non-general features. The purpose of RandomAugment is to teach the network invariances. RandomAugment is a data augmentation technique that automatically searches for improved augmentation policies, where the search space of the policy consists of many sub-policies, one of which is randomly chosen for each data point at each iteration. A sub-policy is composed of basic data-augmentation techniques, such as shearing, rotation, and translation. We implemented Cutout and RandomAugment using the code from https://github.com/ildoonet/pytorch-randaugment. Example images augmented using each data augmentation technique are presented in \cref{fig:augmentations}.

\section{Training losses of FixMatch \citep{sohn2020fixmatch} and ReMixMatch \citep{berthelot2019remixmatch}}
\label{lossfixremix} 

Training losses of FixMatch \citep{sohn2020fixmatch} and ReMixMatch \citep{berthelot2019remixmatch} on a minibatch for labeled set $\mathcal{MX}$ and a minibatch for unlabeled set $\mathcal{MU}$ can be expressed as follows:
\begin{equation}
\begin{aligned}
\label{eqfixx}	loss_{F}\left(\mathcal{MX},\mathcal{MU},\hat{q},\tau;\theta\right)=Con(\mathcal{MU},\hat{q},\tau;\theta)\\+Sup(\mathcal{MX};\theta),
\end{aligned}
\end{equation}
\begin{equation}
\begin{aligned}
\label{eqremixx}
    loss_{R}\left(\mathcal{MX},\mathcal{MU},\Bar{q};\theta\right)=Mix(\mathcal{MX},\mathcal{MU},\bar{q};\theta)\\+Con(\mathcal{MU},\bar{q};\theta)+Rot(\mathcal{MU},r;\theta),
\end{aligned}
\end{equation}

where $\hat{q}$ and $\bar{q}$ denote the concatenations of $\hat{q_b}$ and $\bar{q_b}$, $b=1,\ldots,\mu B$, 
respectively, $Con(\mathcal{MU},\hat{q},\tau;\theta)$ and $Con(\mathcal{MU},\bar{q};\theta)$ denote the consistency regularization loss with and without the confidence threshold $\tau$, respectively, $Sup(\mathcal{MX};\theta)$ denotes the supervised loss for weakly augmented labeled data points, $Mix(\mathcal{MX},\mathcal{MU},\bar{q};\theta)$ denotes the mix-up regularization loss, and $Rot(\mathcal{MU},r;\theta)$ denotes the rotation loss with the rotated degree $r$. 

Each loss term in Eq (1) and (2) of the main paper is detailed as follows:
\begin{equation}
\begin{aligned}
\label{eqcon}
 Con(\mathcal{MU},\hat{q},\tau;\theta)=\\ \frac{1}{\mu B}\sum_{u^m_b\in\mathcal{MU}}{\mathbf{I}( \max(\hat{q_b})\ge\tau)}\mathbf{H}(P_{\theta}(y|\mathcal{A}(u^m_{b})),\hat{q_{b}}),
\end{aligned}
\end{equation}
\begin{equation}
\begin{aligned}
\label{eqcon2}
 Con(\mathcal{MU},\bar{q};\theta)= \frac{1}{\mu B}\sum_{u^m_b\in\mathcal{MU}}\mathbf{H}(P_{\theta}(y|\mathcal{A}(u^m_{b})),\bar{q_{b}}),
 \end{aligned}
 \end{equation}
\begin{equation}
\label{eqsub}
Sup(\mathcal{MX};\theta)=\frac{1}{B}\sum_{x^m_b\in\mathcal{MX}}\mathbf{H}(P_{\theta}(y|\alpha(x^m_{b})),p^m_b),
\end{equation}
\begin{equation}
\begin{aligned}
\label{eqmix}
Mix(\mathcal{MX},\mathcal{MU},\bar{q};\theta)= \frac{1}{B}\sum_{mx^m_b\in\mathcal{MX^\prime}}\mathbf{H}(P_{\theta}(y|mx^m_{b}),mp^m_b)+\\ \frac{1}{\mu B}\sum_{mu^m_b\in\mathcal{MU^\prime}}\mathbf{H}(P_{\theta}(y|mu^m_{b}),\bar{mq_{b}}),
\end{aligned}
\end{equation}
\begin{equation}
\label{eqrot}
Rot(\mathcal{MU},r;\theta)=\frac{1}{\mu B}\sum_{u^m_b\in\mathcal{MU}}\mathbf{H}(P_{\theta^\prime}(d|\mathcal{R}(u^m_{b},r)),r), 
\end{equation}
where $\mathbf{H}(\cdot,\cdot)$ denotes the cross-entropy loss, $p^m_b$ is one-hot encoded $y_b^m$, $\mathcal{MX^\prime}$ and $\mathcal{MU^\prime}$ are generated by mixup operation with strongly augmented $\mathcal{MX}$ and $\mathcal{MU}$, respectively, $mx^m_b$ denotes a mixed-labeled image, $mp^m_b$ denotes a mixed label, $mu^m_b$ denotes a mixed-unlabeled image, $\bar{mq_b}$ denotes a  mixed pseudo-label, $\mathcal{R}(u^m_{b},r)$ denotes the rotated $u^m_{b}$ with degree $r$, and $P_{\theta^\prime}(\hat{r}|\mathcal{R}(u^m_{b},r))$ denotes the prediction of rotated degree $r$ using network parameters $\theta^\prime$ that mostly overlap with $\theta$.

\section{Illustration of refining biased class predictions on test samples using CDMAD}
\label{testphasestructure}
\vspace{-0 in}
\begin{figure}[htbp]
	\begin{center}
 \vspace{0 in}
    \hspace{0.1 in}    \includegraphics[width=3.3in,height=2.28in]{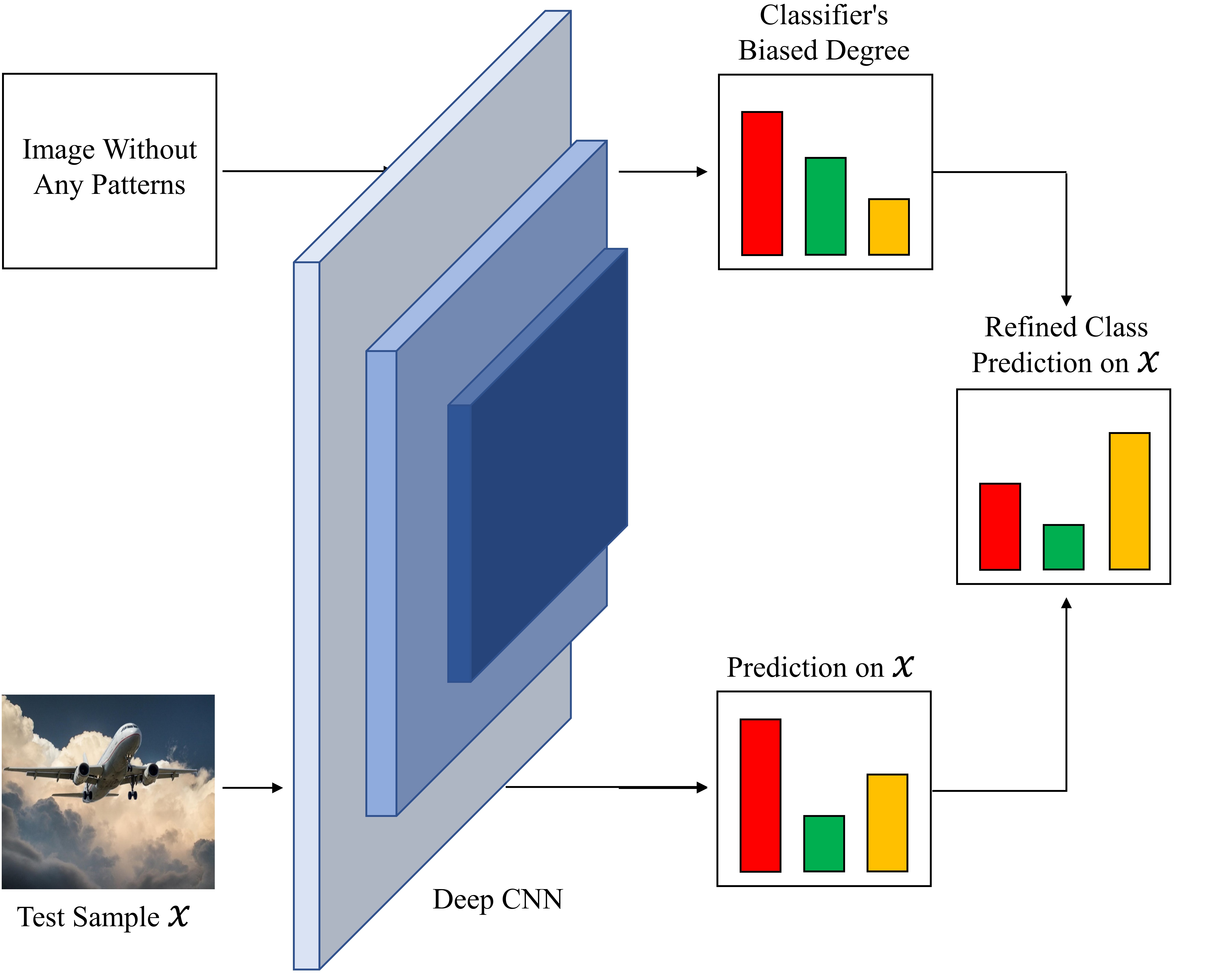}
	\end{center}
 \vspace{-0.15 in}
	\caption{%The process pseudo-labels refinement with CDMAD and consistency regularization. The consistency regularization loss is calculated from the predicted class probability on $\mathcal{A}\left(u\right)$ and the CDMAD pseudo-label in the blue box.
 Refinement of biased class predictions on test samples using CDMAD}
	\label{fig:architecturetest}
\vspace{-0.15 in}
\end{figure}
\cref{fig:architecturetest} presents refinement process of the biased class predictions on test samples using the CDMAD.
\vspace{-0 in}

\section{Pseudo code of the proposed algorithm}
\label{pseudocode}
The pseudo code that describes both training and test phases of the proposed algorithm is presented in \cref{alg}.
\begin{algorithm*}[htbp]
   \caption{Pseudo code of the proposed algorithm}
   \label{alg}
\begin{algorithmic}
   \STATE {\bfseries Input:} Labeled set $\mathcal{X}$, unlabeled set $\mathcal{U}$, test set $\mathcal{X}^{test}$, network parameters $\theta$
   \STATE {\bfseries Output:} Refined class predictions on test samples $f^*_{\theta}\left(x^{test}_{k}\right)$ for $k=1,\ldots,K$
   \WHILE{training}
    \STATE Generate minibatches $\mathcal{MX} = \left\lbrace\left(x^m_{b},y^m_{b}\right):  b\in \left(1,\ldots, B\right)\right\rbrace \subset \mathcal{X}$ and $\mathcal{MU}= \left\lbrace\left(u^m_{b}\right):  b\in\left(1,\ldots, \mu B\right)\right\rbrace \subset \mathcal{U}$
    \STATE Produce logits for a solid color image $g_{\theta}\left(\mathcal{I}\right)$
    \STATE Produce logits for weakly augmented unlabeled samples $g_{\theta}\left(\alpha\left(u^m_{b}\right)\right)$ for $b=1,\ldots,\mu B$
    \STATE Obtain refined logits $g^*_{\theta}\left(\alpha\left(u^m_{b}\right)\right)=g_{\theta}\left(\alpha\left(u^m_{b}\right)\right)-g_{\theta}\left(\mathcal{I}\right)$ for $b=1,\ldots,\mu B$
    \STATE Obtain refined pseudo-labels $q_b^*=\phi\left(g^*_{\theta}\left(\alpha\left(u^m_{b}\right)\right)\right)$ for $b=1,\ldots,\mu B$
    \IF{Base SSL==`FixMatch'}
   \STATE $loss_{F}^*=loss_{F}\left(\mathcal{MX},\mathcal{MU},q^*,0;\theta\right)$
   \STATE $\Delta\boldsymbol{\theta}\propto\nabla_{\boldsymbol{{\theta}}}loss_{F}^*$, $\quad\boldsymbol{\theta}\gets\boldsymbol{\theta}+\Delta\boldsymbol{\theta}$
   \ENDIF
   \IF{Base SSL==`ReMixMatch'}
   \STATE Produce class probabilities on wealy augmented labeled samples $P_{\theta}\left(y|\alpha\left(x^m_{b}\right)\right)$ for $b=1,\ldots, B$
   \STATE $CEloss=CrossEntropy\left(p^m_b,P_{\theta}\left(y|\alpha\left(x^m_b\right)\right)\right)$
   \STATE $loss_{R}^*=loss_{R}\left(\mathcal{MX},\mathcal{MU},q^*;\theta\right)+CEloss$
   \STATE $\Delta\boldsymbol{\theta}\propto\nabla_{\boldsymbol{{\theta}}}loss_{R}^*$, $\quad\boldsymbol{\theta}\gets\boldsymbol{\theta}+\Delta\boldsymbol{\theta}$
   \ENDIF   
   \ENDWHILE
   \STATE Produce logits for a solid color image $g_{\theta}\left(\mathcal{I}\right)$
   \STATE Produce logits for test samples $g_{\theta}\left(x^{test}_{k}\right)$ for $k=1,\ldots,K$
   \STATE Obtain refined logits $g^*_{\theta}\left(x^{test}_{k}\right)=g_{\theta}\left(x^{test}_{k}\right)-g_{\theta}\left(\mathcal{I}\right)$ for $k=1,\ldots,K$
   \STATE Obtain refined class predictions $f^*_{\theta}\left(x^{test}_{k}\right)=\argmaxG_cg^*_{\theta}\left(x^{test}_{k}\right)_c$ for $k=1,\ldots,K$
\end{algorithmic}
\end{algorithm*}

\section{Performance measures}
\label{measures} 
Following previous CISSL studies, we used balanced accuracy (bACC) \citep{huang2016learning}, geometric mean (GM) \citep{kubat1997addressing} as performance measures for the experiments in Section 4.2. Each performance measure is detailed as follows. \textbf{Balanced accuracy (bACC)} is the average of per-class accuracy. When the test set is class-balanced, bACC equals to the overall test accuracy. bACC is also referred to as the  averaged class recall in previous CISSL studies \citep{wei2021crest} and \citep{fan2022CoSSL}. \textbf{Geometric mean (GM)} is obtained by multiplying the $C$th root of per-class accuracy, where $C$ denotes the number of classes. GM equals to the overall test accuracy when all classes have the same per-class accuracy. %These two performance measures allow the classification performance for classes which have small amount of test samples to be well reflected in the overall classification performance.
\section{Further details about datasets and experimental setup}
\label{furtherdetail}

\textbf{CIFAR-10-LT and CIFAR-100-LT} are long-tailed datasets artificially generated from CIFAR-10 and CIFAR-100 \citep{krizhevsky2009learning}, respectively, 
with $N_{k}=N_{1}\times\left(N_{C}/N_{1}\right)^{\frac{k-1}{C-1}}$ and $M_{k}=M_{1}\times\left(M_{C}/M_{1}\right)^{\frac{k-1}{C-1}}$.
For CIFAR-10-LT, we assumed that $\gamma_u$ is known and equal to $\gamma_l$ while varying both $\gamma_l$ and $\gamma_u$ as $50$, $100$ and $150$. We then assumed that $\gamma_u$ is unknown and different from $\gamma_l$ while setting $\gamma_l$ to $100$ and varying $\gamma_u$ as $1$, $50$ and $150$. 
We set $N_1=1500$ and $M_1=3000$. 
For CIFAR-100-LT, we assumed that $\gamma_u$ is known and equal to $\gamma_l$ while varying both $\gamma_l$ and $\gamma_u$ as $20$, $50$ and $100$. We set $N_1=150$ and 
\textbf{STL-10-LT} is a long-tailed dataset created 
from STL-10 \citep{coates2011analysis}, where the number of labeled samples exponentially decreases from $N_{1}$ to $N_{C}$. 
We conducted experiments with unknown $\gamma_u$ while varying $\gamma_l$ as $10$ and $20$. We set $N_1$ to $450$ and used all $100,000$ unlabeled samples.
\textbf{Small-ImageNet-127} is a down-sampled version of ImageNet-127 \citep{huh2016makes}, created by grouping ImageNet \citep{russakovsky2015imagenet} into 127 classes based on WordNet hierarchy. The training set of ImageNet-127 consists of a total of 1,281,167 images and is imbalanced with the class imbalanced ratio of 286. \citet{fan2022CoSSL} created two versions of this dataset by down-sampling the images to 32$\times$32 and 64$\times$64, and randomly selected 10\% of the training samples of each class as a labeled set and used the remaining as an unlabeled set. We conducted experiments on both versions under the assumption that $\gamma_u$ is known and equal to $\gamma_l$. Similar to \citet{wei2021crest,fan2022CoSSL}, we conducted experiments using only FixMatch because of an excessive training cost. 
%because the experiments on this dataset required an excessive training cost. 
The test set of Small-ImageNet-127 is also class-imbalanced.

We used the Adam optimizer \citep{DBLP:journals/corr/KingmaB14}. %to train the proposed algorithm. 
We used the exponential moving average (EMA) of the network parameters for each iteration %with a decay parameter of $0.999$ to evaluate the classification performance on the test set
to evaluate the classification performance. We used Wide ResNet-28-2 \citep{zagoruyko2016wide} as a deep CNN for CIFAR-10-LT, CIFAR-100-LT, and STL-10-LT, and ResNet-50 \citep{he2016deep} for Small-ImageNet-127. %We set (R,G,B) values of the non-image input $\mathcal{I}$ as (510,510,510).

For the experiments using FixMatch, we set the minibatch size $B$ to 32, relative size of the unlabeled to labeled minibatches $\mu$ to 2, and learning rate of the optimizer to $1.5*10^{-3}$. We trained  FixMatch for 500 epochs, where 1 epoch= 500 iterations. For the experiments using ReMixMatch, we set the minibatch size $B$ to $64$, relative size of the unlabeled to labeled minibatches $\mu$ to 2, and learning rate of the optimizer to $2*10^{-3}$. We trained ReMixMatch for 300 epochs. For the experiments on CIFAR-100, we set the weight decay parameter of L2 regularization (for EMA parameters) to 0.08 because CIFAR-100 has significantly many classes compared to the total number of training samples. For the experiments on CIFAR-10, STL-10, and Small-ImageNet-127, we set the weight decay parameter of L2 regularization to 0.04 when the number of training samples is smaller than $3*10^4$, while we set it to 0.01 and 0.02 for FixMatch and ReMixMatch, respectively, when the number of training samples is larger than $3*10^4$, because L2 regularization becomes ineffective as the number of training samples increases. We confirmed that the training of the proposed algorithm took less time than the baseline CISSL algorithms. We used random cropping and horizontal flipping for weak data augmentation and Cutout \citep{devries2017improved} and RandomAugment \citep{cubuk2020randaugment} for strong data augmentation. These augmentation techniques are detailed in \cref{augment}. To use CDMAD after network parameters are stabilized, we trained naive ReMixMatch and FixMatch for first 100 epochs, and subsequently used CDMAD to refine pseudo-labels, similar to DARP \citep{NEURIPS2020_a7968b43}. We conducted experiments using the GPU server Nvidia Tesla-V100 and 3090ti and used the Python library PyTorch 1.11.0 and 1.12.1. Our experiment results can be reproduced using the code in the supplementary material. 

\section{Description of baseline algorithms}
\label{baseline}
The classification performance of the CDMAD was compared with those of the following algorithms: \textbf{1. vanilla algorithm} - Deep CNN trained with cross-entropy loss, \textbf{2. CIL algorithms} - Re-sampling \citep{japkowicz2000class}, LDAM-DRW \citep{NEURIPS2019_621461af}, and cRT \citep{kang2019decoupling}, \textbf{3. SSL algorithms} - FixMatch \citep{sohn2020fixmatch} and ReMixMatch \citep{berthelot2019remixmatch}, and \textbf{4. CISSL algorithms} - DARP, DARP+LA, DARP+cRT \citep{NEURIPS2020_a7968b43}, CReST, CReST+LA \citep{wei2021crest}, ABC \citep{lee2021abc}, CoSSL \citep{fan2022CoSSL}, DASO \cite{oh2022daso}, SAW, SAW+LA and SAW+cRT \citep{lai2022smoothed} combined with FixMatch and ReMixMatch. Adsh \cite{guo2022class}, DebiasPL \cite{wang2022debiased}, UDAL \cite{lazarow2023unifying} and L2AC \cite{wang2023imbalanced} combined with FixMatch. We report the performance of the baseline algorithms reported in Tables of \citet{lai2022smoothed} and \citet{fan2022CoSSL} when it is reproducible; the performance measured using the uploaded code was reported otherwise. 

\section{Further qualitative analysis}
\label{furtherqualitative} 
\subsection{Case of $\gamma_l=\gamma_u$}
In Table 1 of Section 4.2, CDMAD performed better than the baseline CISSL algorithms when the class distributions of the labeled and unlabeled sets are assumed to be the same. To verify whether the pseudo-labels and class predictions on test samples refined by CDMAD contributed to its superior performance, we conducted two types of comparison: 1) pseudo-labels refined by CDMAD vs. true labels of unlabeled samples, and 2) class predictions refined by CDMAD vs. true labels of test samples. These results are also compared to those from FixMatch and ReMixMatch. 

%which summarizes the experimental results under the assumption that the class distribution of the unlabeled set is known and matched with that of the labeled set, the proposed algorithm achieved better performance than the baseline CISSL algorithms because CDMAD effectively refined biased pseudo-labels and adjusted biased class predictions. To show quality of the refined pseudo-labels and adjusted class predictions,} 
First, \cref{fig:confusion100100pseudo} compares the confusion matrices of pseudo-labels generated by (a) FixMatch, (b) FixMatch+CDMAD, (c) ReMixMatch, and (d) ReMixMatch+CDMAD trained on CIFAR-10-LT under $\gamma_l=100$ and $\gamma_u=100$. The value in the $i$th row and $j$th column represents the proportion of the $i$th class samples classified as the $j$th class. We can observe that the pseudo-labels of FixMatch and ReMixMatch are biased toward the majority classes. Specifically, the data points in the minority classes (e.g., classes 8 and 9) are often misclassified into the majority classes (e.g. classes 0 and 1). In contrast, \cref{fig:confusion100100pseudo} (b) and \cref{fig:confusion100100pseudo} (d) show that FixMatch+CDMAD and ReMixMatch+CDMAD made nearly balanced class predictions.

\begin{figure*}[htbp]

 \begin{center}
		\begin{tabular}{cccc}
			% \hspace{-0.1 in}\includegraphics[width=3.8cm, height=3.8cm]{img/1001naiveremixconfusion.jpg}&\hspace{0.0cm} \includegraphics[width=3.8cm, height=3.8cm]{img/1001remixcdmadconfusion.jpg} \\\
    \includegraphics[width=3.8cm, height=3.8cm]{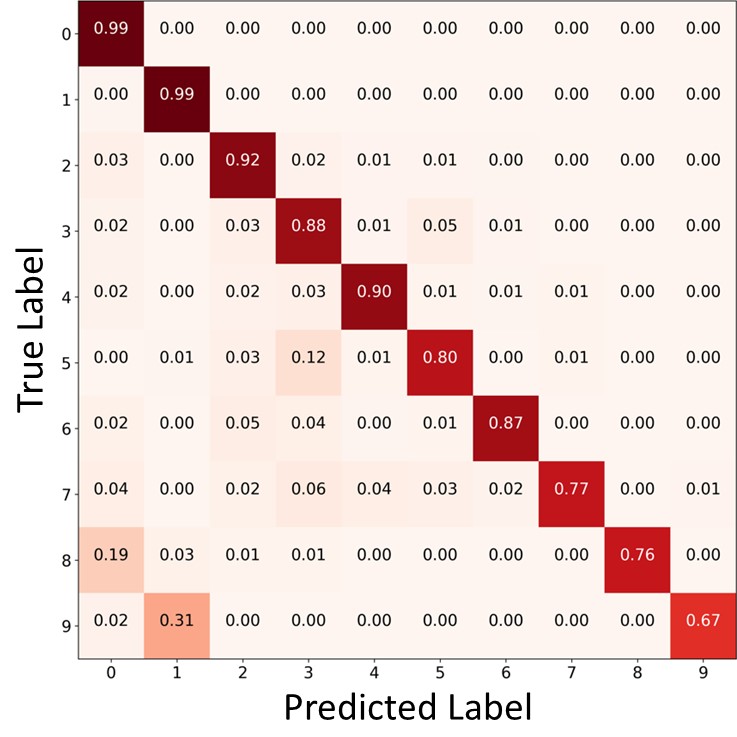}&\hspace{0.0cm} \includegraphics[width=3.8cm, height=3.8cm]{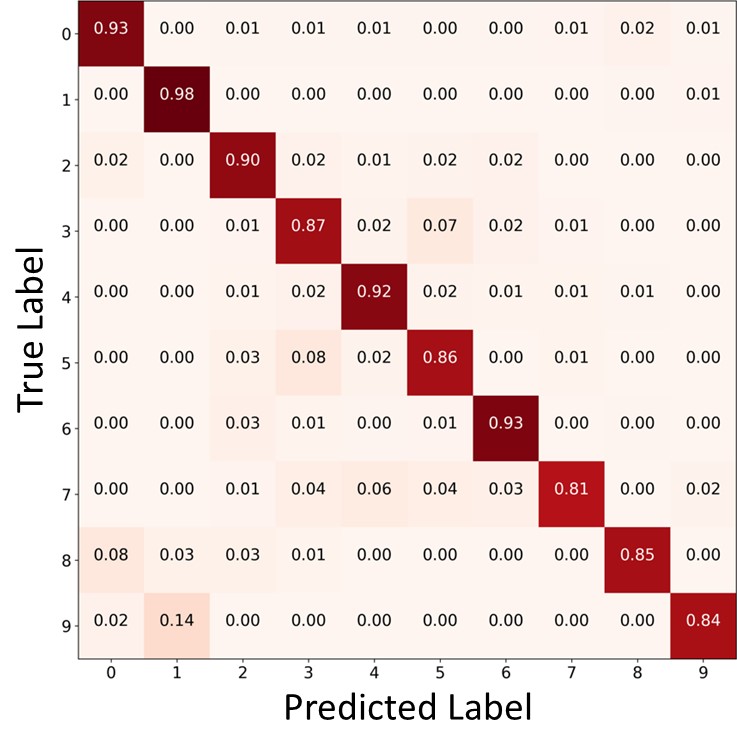}&\includegraphics[width=3.8cm, height=3.8cm]{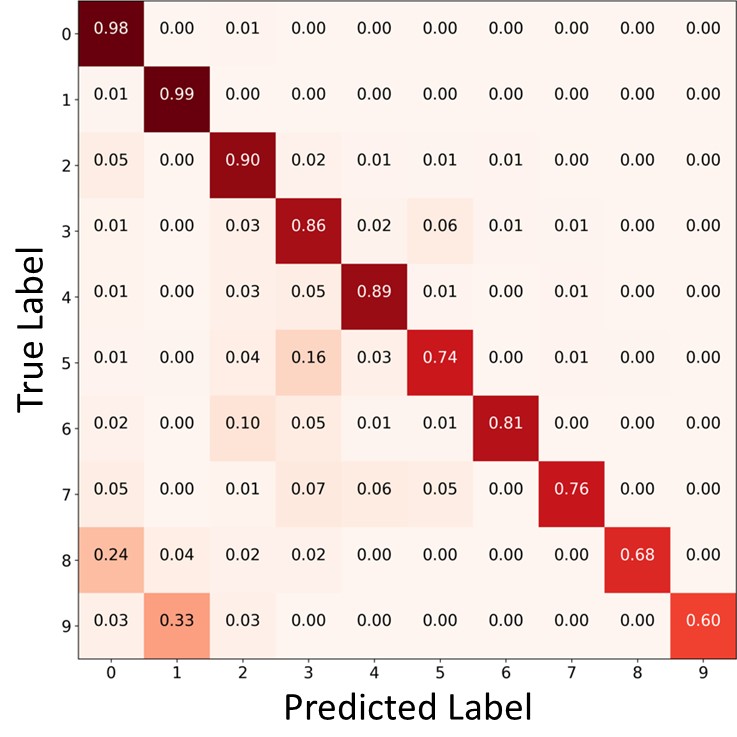}&\includegraphics[width=3.8cm, height=3.8cm]{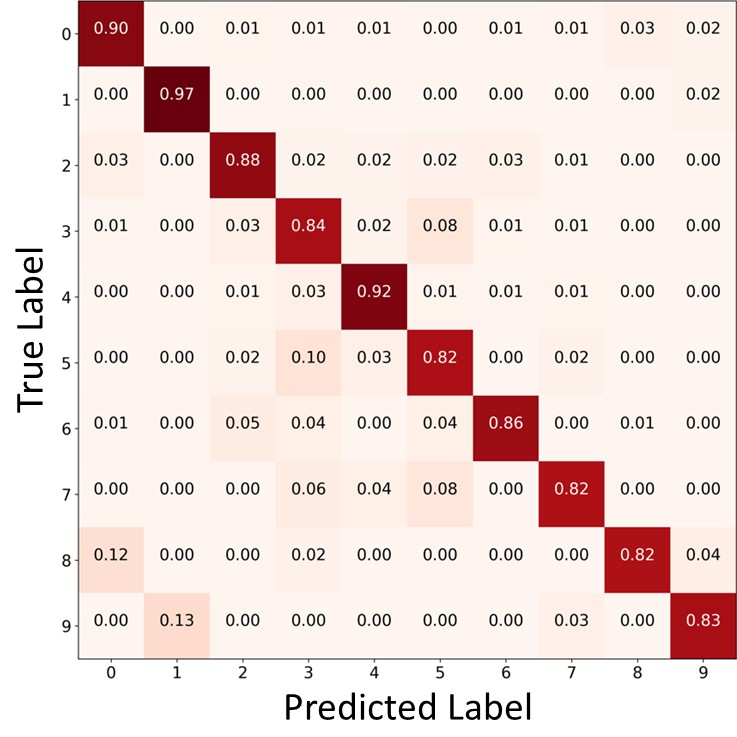} \\\
			 \footnotesize{(a) FixMatch} &\hspace{0.cm}
			\footnotesize{(b) FixMatch+CDMAD}& \footnotesize{(c) ReMixMatch}& \footnotesize{(d) ReMixMatch+CDMAD}
		\end{tabular}
	\end{center}

	\caption{Confusion matrices of pseudo-labels generated by (a) FixMatch, (b) FixMatch+CDMAD, (c) ReMixMatch, and (d) ReMixMatch+CDMAD trained on CIFAR-10-LT under $\gamma_l=100$ and $\gamma_u=100$.}
	\label{fig:confusion100100pseudo}

\end{figure*}
\begin{figure*}[htbp]
	
 \begin{center}
		\begin{tabular}{cccc}
			% \hspace{-0.1 in}\includegraphics[width=3.8cm, height=3.8cm]{img/1001naiveremixconfusion.jpg}&\hspace{0.0cm} \includegraphics[width=3.8cm, height=3.8cm]{img/1001remixcdmadconfusion.jpg} \\\
    \includegraphics[width=3.8cm, height=3.8cm]{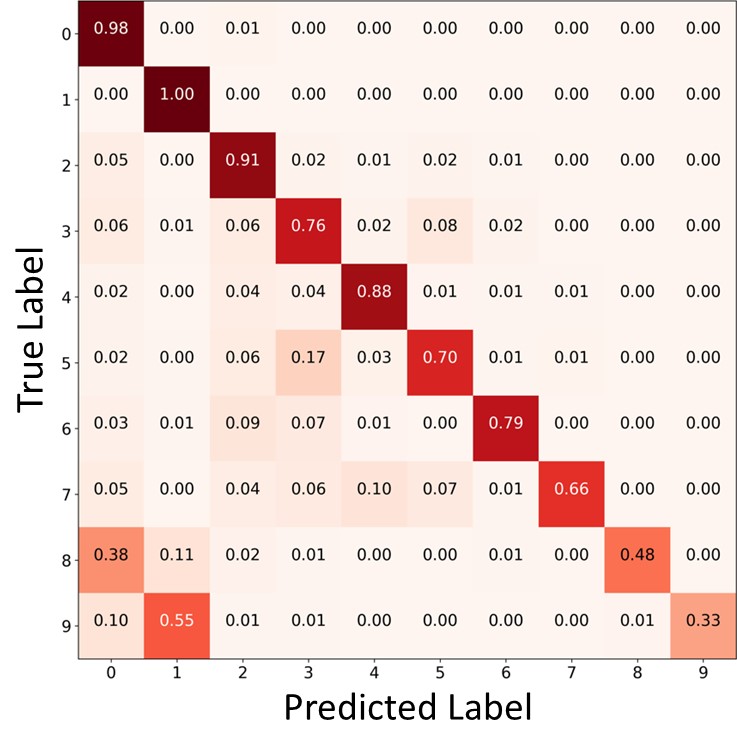}&\includegraphics[width=3.8cm, height=3.8cm]{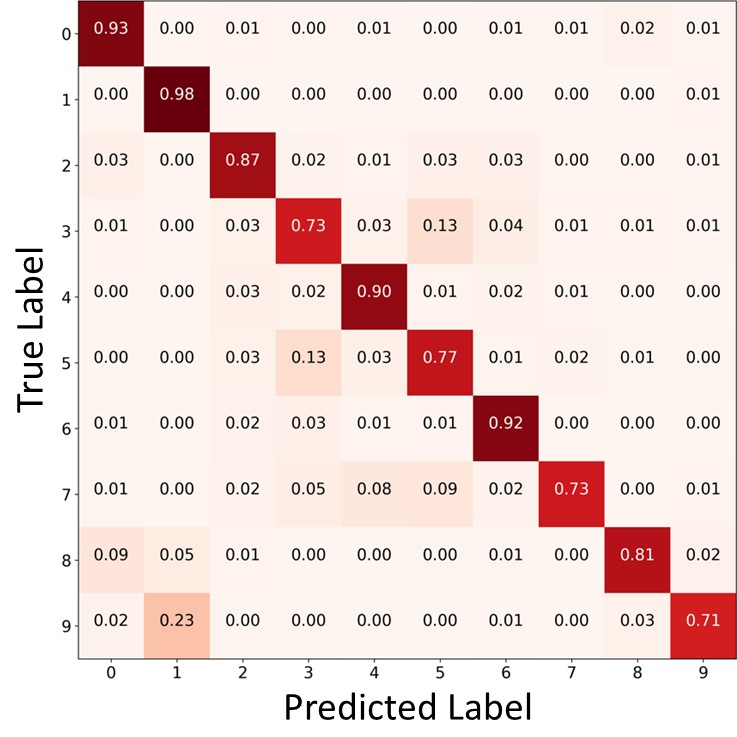}&\includegraphics[width=3.8cm, height=3.8cm]{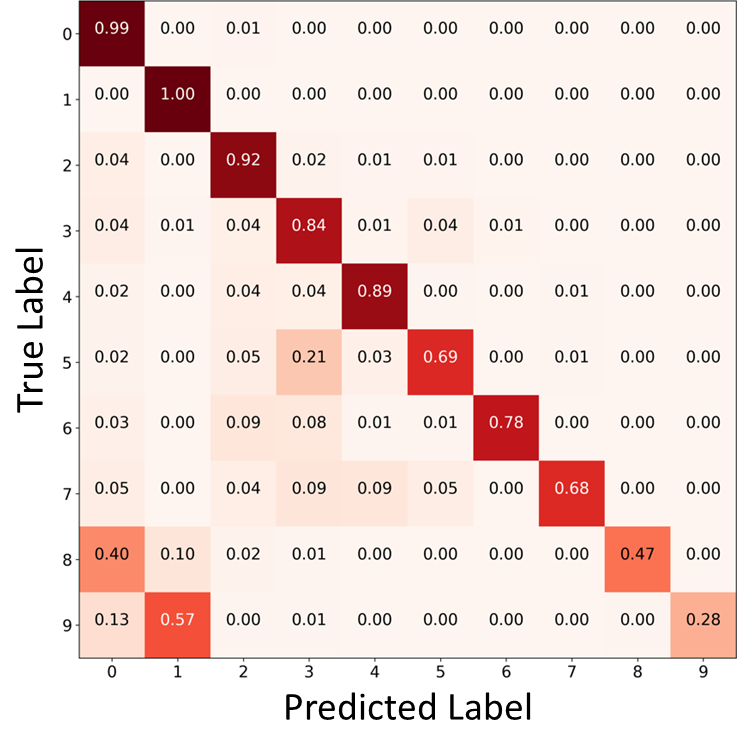}&\includegraphics[width=3.8cm, height=3.8cm]{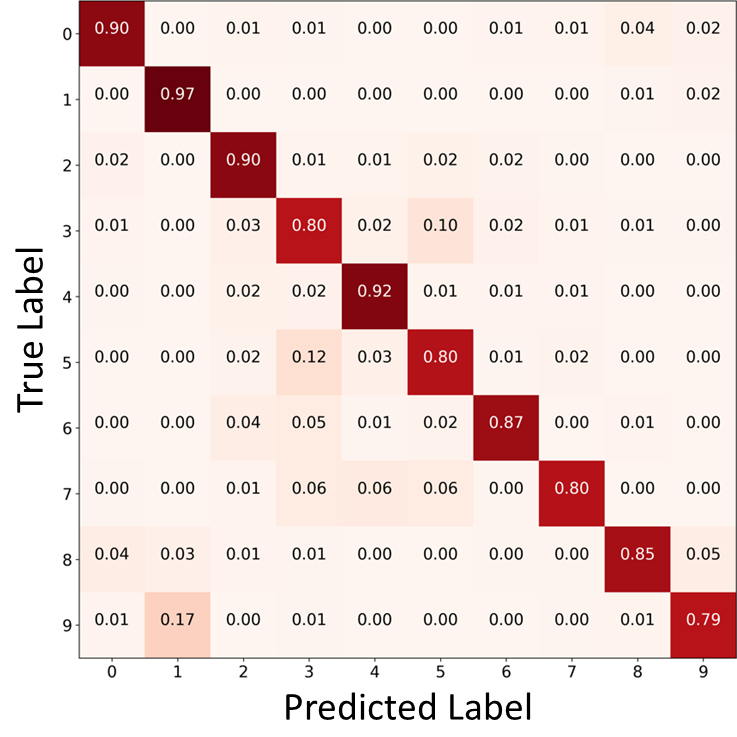} \\\
			  \footnotesize{(a) FixMatch} &
			\footnotesize{(b) FixMatch+CDMAD}&\footnotesize{(c) ReMixMatch} &
			\footnotesize{(d) ReMixMatch+CDMAD}
		\end{tabular}
	\end{center}

	\caption{Confusion matrices of the class predictions on the test set of CIFAR-10 using (a) FixMatch, (b) FixMatch+CDMAD, (c) ReMixMatch, and (d) ReMixMatch+CDMAD trained on CIFAR-10-LT under $\gamma_l=100$ and $\gamma_u=100$.}
	\label{fig:confusion100100}
\end{figure*}

Second, \cref{fig:confusion100100} compares the confusion matrices of the class predictions on the test set of CIFAR-10 using (a) FixMatch, (b) FixMatch+CDMAD, (c) ReMixMatch, and (d) ReMixMatch+CDMAD trained on  CIFAR-10-LT under $\gamma_l=100$ and $\gamma_u=100$. 
Similar to \cref{fig:confusion100100pseudo}, FixMatch+CDMAD and  ReMixMatch+CDMAD made more balanced predictions across classes.

%Third, \cref{tsne} compares the representations 

\subsection{Case of $\gamma_l\neq\gamma_u$}
In Table 2 of Section 4.2, the proposed algorithm performed better than the baseline algorithms when the class distribution of the unlabeled set is assumed to be unknown and actually differs with that of the labeled set. To verify whether the pseudo-labels and class predictions refined by CDMAD contributed to its superior performance, we conducted three types of comparison: 1) pseudo-labels refined by CDMAD vs. true labels of unlabeled samples, 2) representations learned with unrefined pseudo-labels vs. representations learned with pseudo-labels refined by CDMAD, and 3) class predictions refined by CDMAD vs. true labels of test samples. These results are also compared to those from FixMatch and ReMixMatch.

First, \cref{fig:confusion1001pseudo} compares the confusion matrices of pseudo-labels generated by (a) FixMatch, (b) FixMatch+CDMAD, (c) ReMixMatch, and (d) ReMixMatch+CDMAD trained on CIFAR-10-LT under $\gamma_l=100$ and $\gamma_u=1$. The value in the $i$th row and $j$th column represents the proportion of the $i$th class samples classified as the $j$th class. We can observe that the pseudo-labels of FixMatch and ReMixMatch are biased toward the majority classes. Specifically, the data points in the minority classes (e.g., classes 7, 8 and 9) are often misclassified into the majority classes (e.g. classes 0 and 1). In contrast, \cref{fig:confusion1001pseudo} (b) and \cref{fig:confusion1001pseudo} (d) show that FixMatch+CDMAD and ReMixMatch+CDMAD made nearly balanced class predictions.

\begin{figure*}[htbp]

 \begin{center}
		\begin{tabular}{cccc}
			% \hspace{-0.1 in}\includegraphics[width=3.8cm, height=3.8cm]{img/1001naiveremixconfusion.jpg}&\hspace{0.0cm} \includegraphics[width=3.8cm, height=3.8cm]{img/1001remixcdmadconfusion.jpg} \\\
    \includegraphics[width=3.8cm, height=3.8cm]{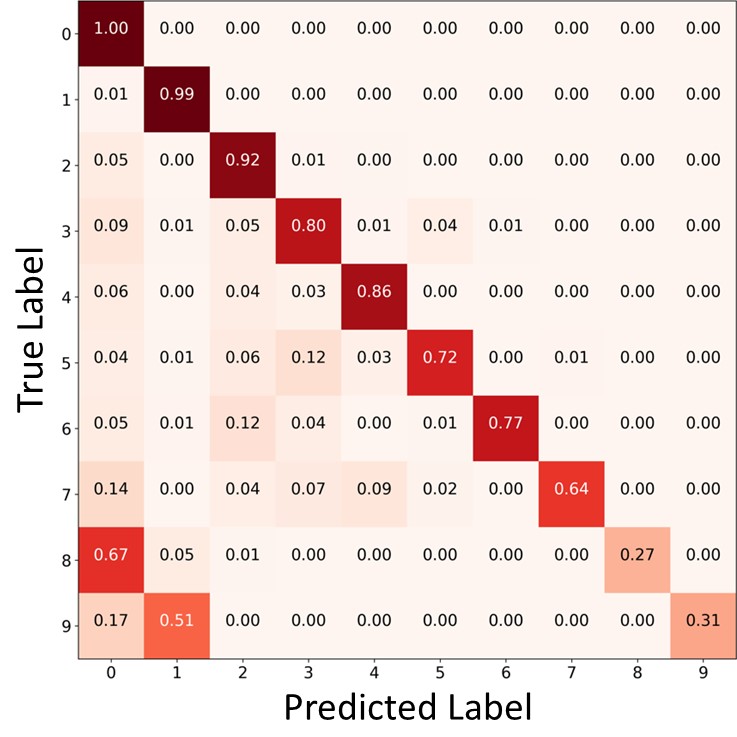}&\hspace{0.0cm} \includegraphics[width=3.8cm, height=3.8cm]{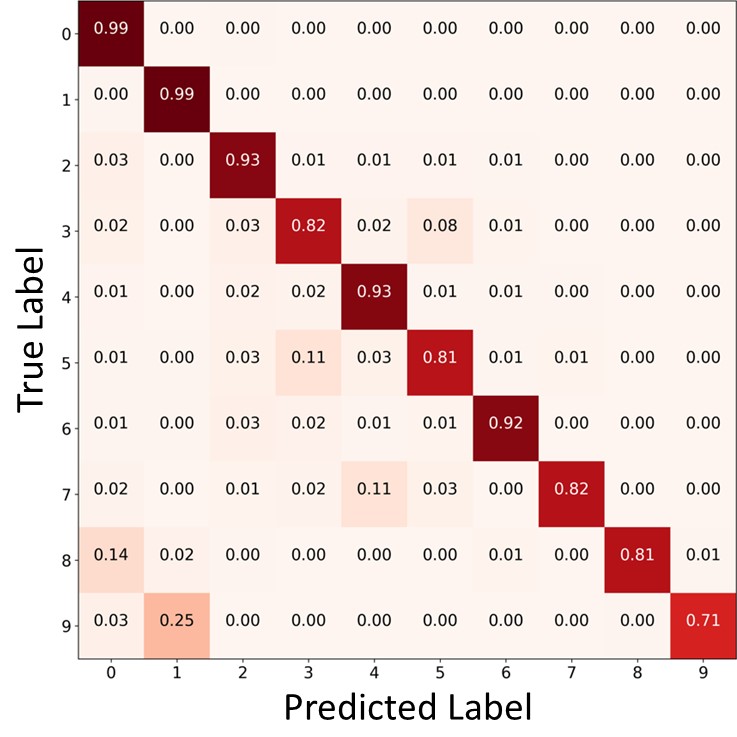}&\includegraphics[width=3.8cm, height=3.8cm]{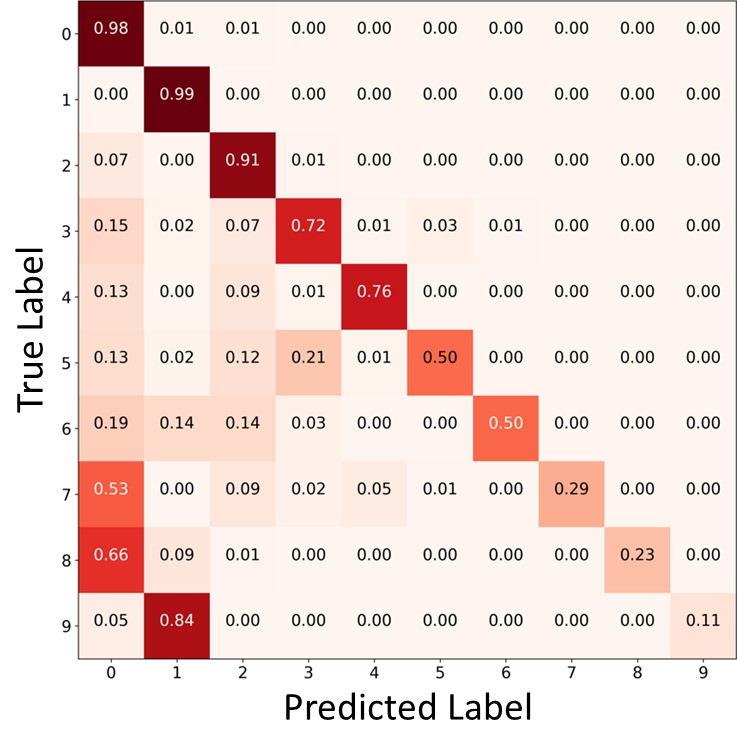}&\includegraphics[width=3.8cm, height=3.8cm]{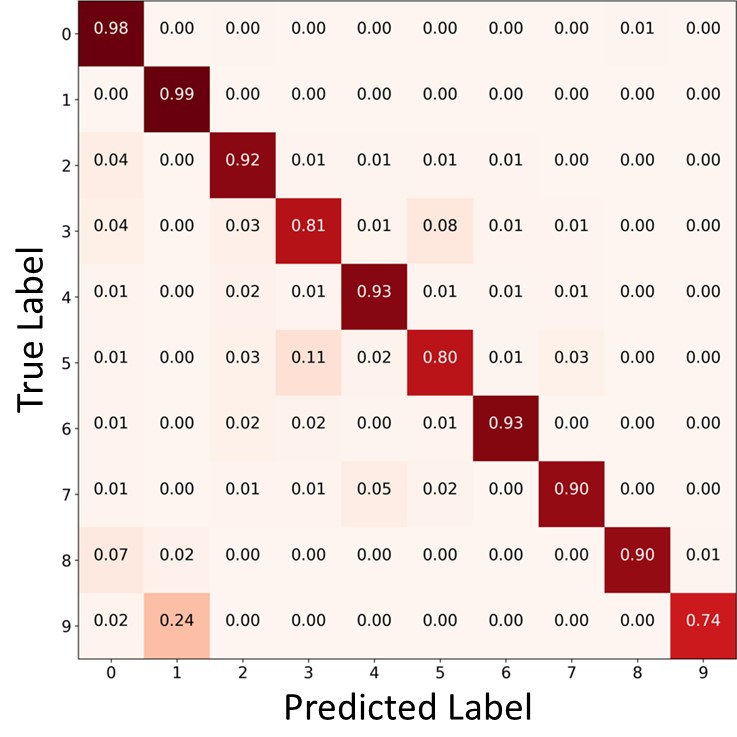} \\\
			 \footnotesize{(a) FixMatch} &\hspace{0.cm}
			\footnotesize{(b) FixMatch+CDMAD}& \footnotesize{(c) ReMixMatch}& \footnotesize{(d) ReMixMatch+CDMAD}
		\end{tabular}
	\end{center}

	\caption{Confusion matrices of pseudo-labels generated by (a) FixMatch, (b) FixMatch+CDMAD, (c) ReMixMatch, and (d) ReMixMatch+CDMAD trained on CIFAR-10-LT under $\gamma_l=100$ and $\gamma_u=1$.}
	\label{fig:confusion1001pseudo}

\end{figure*}

\begin{figure*}[htbp]

	\begin{center}
        \begin{tabular}{cccc}
			\includegraphics[width=3.8cm, height=3.8cm]{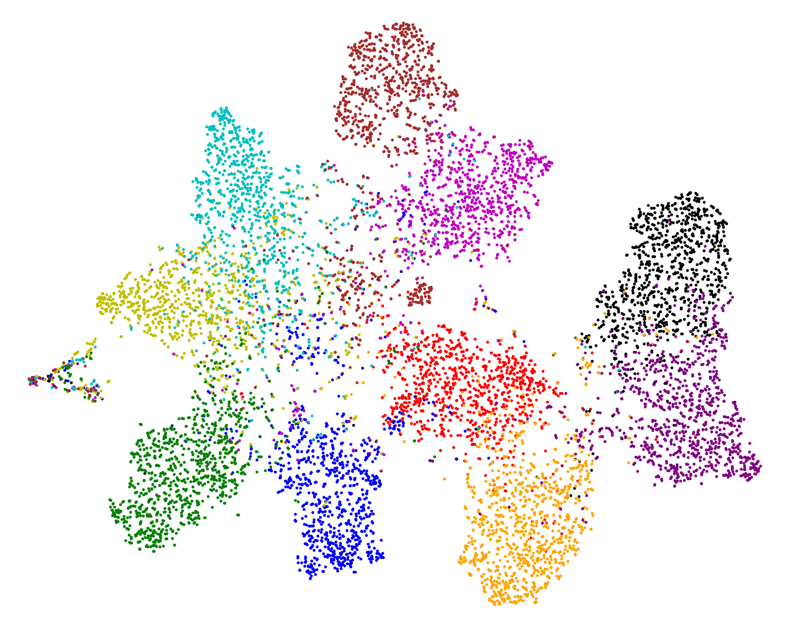}&\includegraphics[width=3.8cm, height=3.8cm]{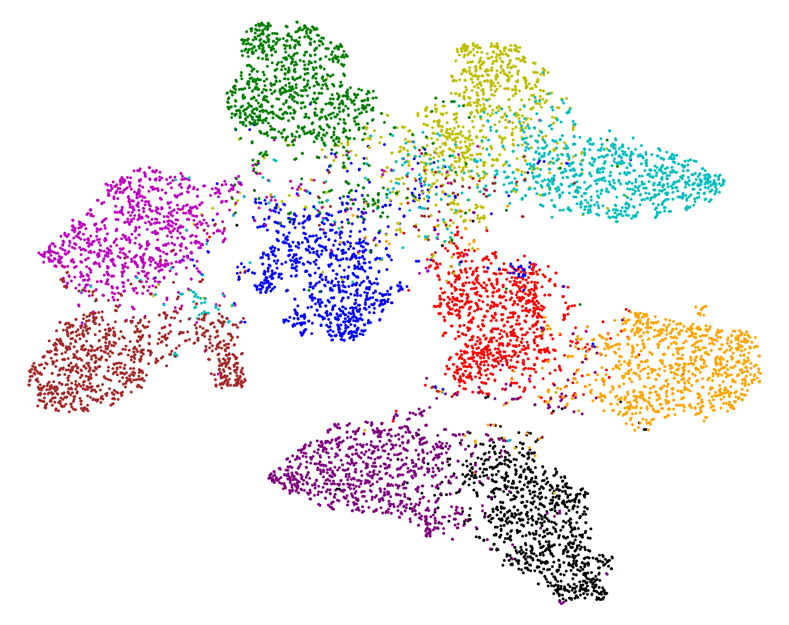}&\includegraphics[width=3.8cm, height=3.8cm]{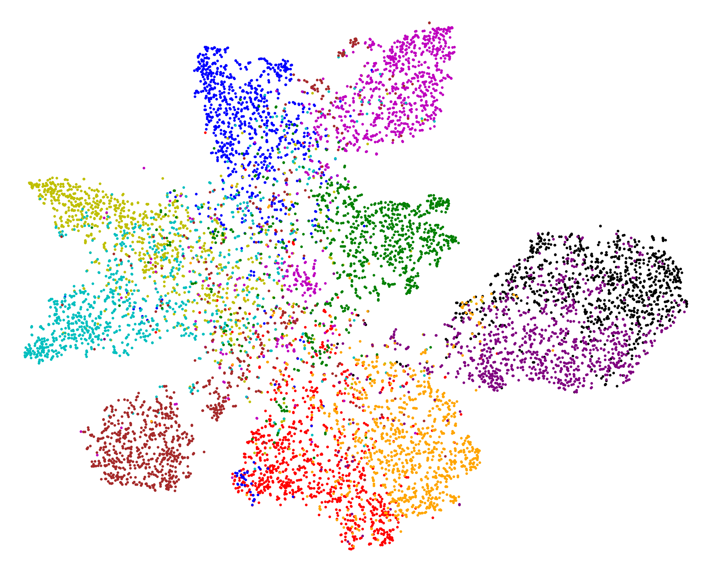}&\includegraphics[width=3.8cm, height=3.8cm]{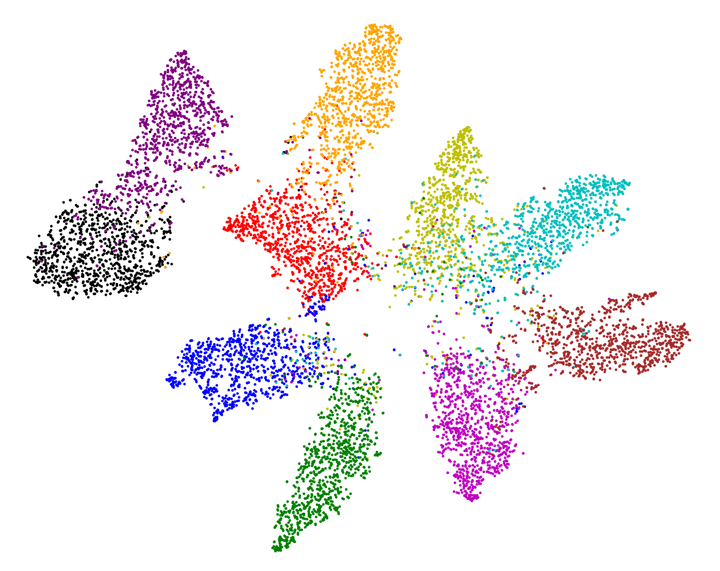} \\\
			   \footnotesize{(a) FixMatch} &\footnotesize{(b) FixMatch+CDMAD} &\footnotesize{(c) ReMixMatch} &\footnotesize{(d) ReMixMatch+CDMAD} 
		\end{tabular}
	\end{center}

	\caption{t-SNE of representations obtained for the test set of CIFAR-10 using (a) FixMatch, (b) FixMatch+CDMAD, (c) ReMixMatch, and (d) ReMixMatch+CDMAD trained on CIFAR-10-LT under $\gamma_l=100$ and  $\gamma_u=1$.}
	\label{tsne}

\end{figure*}

Second, \cref{tsne} compares t-distributed stochastic neighbor embedding (t-SNE) \citep{van2008visualizing} of representations obtained for the test set of CIFAR-10 using FixMatch, FixMatch+CDMAD, ReMixMatch, and ReMixMatch+CDMAD trained on CIFAR-10 with $\gamma_l=100$ and $\gamma_u=1$ (unknown $\gamma_u$), where different colors indicate different classes in CIFAR-10. We can observe that the representations obtained using FixMatch+CDMAD and ReMixMatch+CDMAD are separated into classes with clearer boundaries compared the those from FixMatch and ReMixMatchFrom in \cref{tsne} (a) and \cref{tsne} (c).
%In particular, in the left-middle part of \cref{tsne}(a), there is a cluster (XXX) consist of representations of samples belonging to various classes. 
This is probably because CDMAD appropriately refined the biased pseudo-labels and used them for training, whereas FixMatch and ReMixMatch failed to learn the representations properly because they used the biased pseudo-labels for training. 
%It seems that the representations were not properly learned because biased pseudo-labels of FixMatch and ReMixMatch were used for training. On the other hand, From the \cref{tsne} (b) and (d), we can observe that representations of FixMatch+CDMAD, ReMixMatch+CDMAD is overally separable compared to representations of FixMatch (a) and ReMixMatch (c). This seems to be because CDMAD appropriately refined the biased pseudo-labels and then used them for training. 
These results demonstrate that the quality of representations can be improved by using well refined pseudo-labels ( \cref{fig:confusion1001pseudo} (b) and \cref{fig:confusion1001pseudo} (d)) for training. 

Third, \cref{fig:confusion1001} compares the confusion matrices of the class predictions on the test set of CIFAR-10 using (a) FixMatch and (b) FixMatch+CDMAD trained on CIFAR-10-LT under $\gamma_l=100$ and $\gamma_u=1$. 
Similar to \cref{fig:confusion1001pseudo}, FixMatch+CDMAD made more balanced predictions across classes compared to the other algorithms. (Note that the results using ReMixMatch and ReMixMatch+CDMAD are presented in Section 4.3.)
%The confusion matrices of the class predictions on the test set of CIFAR-10 using (a) ReMixMatch and (b) ReMixMatch+CDMAD trained on a CIFAR-10-LT under $\gamma_l=100$, $\gamma_u=1$ are presented in Section \ref{analysis} of the paper.

 \begin{figure}[htbp]

 \begin{center}
		\begin{tabular}{cc}
			% \hspace{-0.1 in}\includegraphics[width=3.8cm, height=3.8cm]{img/1001naiveremixconfusion.jpg}&\hspace{0.0cm} \includegraphics[width=3.8cm, height=3.8cm]{img/1001remixcdmadconfusion.jpg} \\\
   \includegraphics[width=3.8cm, height=3.8cm]{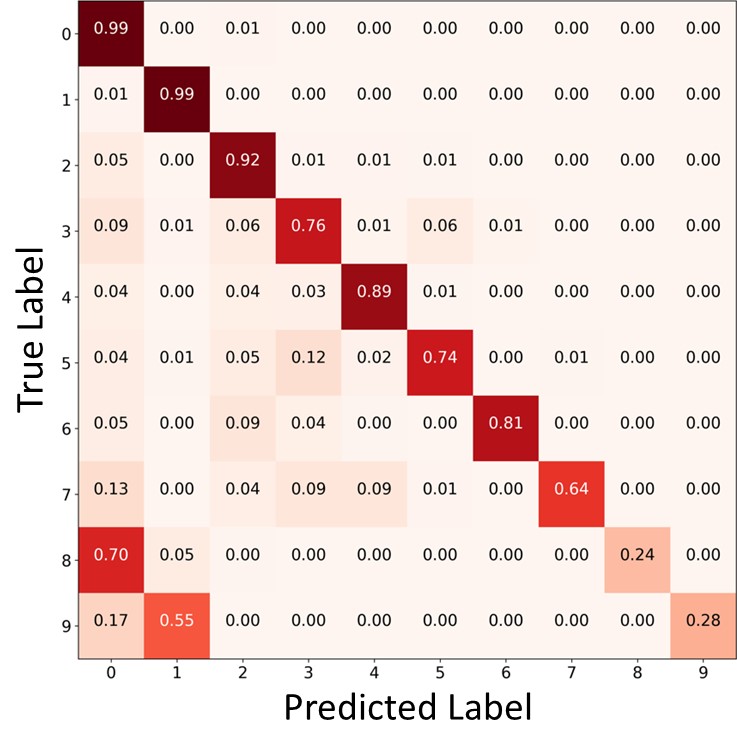}&\includegraphics[width=3.8cm, height=3.8cm]{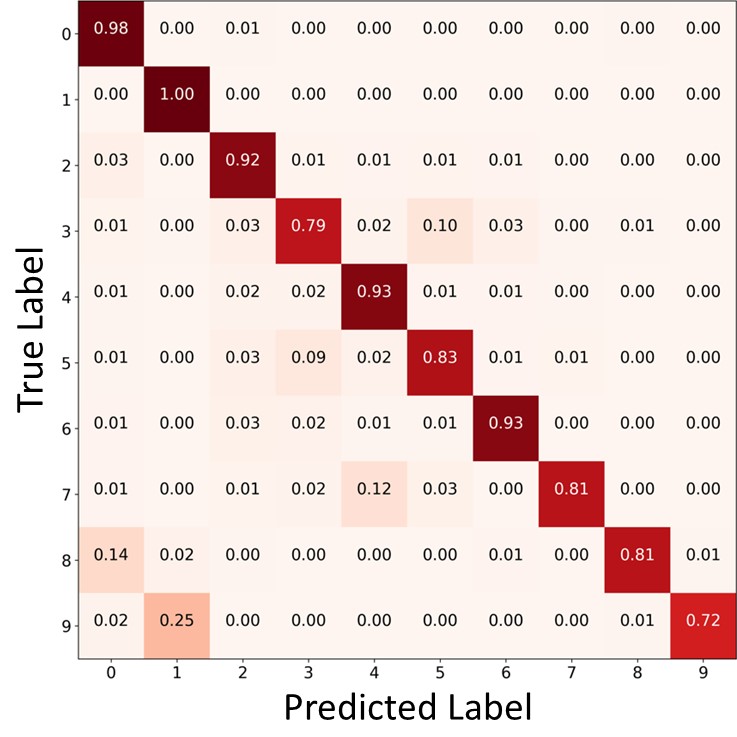} \\\
			 \footnotesize{(a) FixMatch} &\hspace{0.cm}
			\footnotesize{(b) FixMatch+CDMAD}
		\end{tabular}
	\end{center}

	\caption{Confusion matrices of the class predictions on the test set of CIFAR-10 using (a) FixMatch and (b) FixMatch+CDMAD trained on   CIFAR-10-LT under $\gamma_l=100$ and $\gamma_u=1$.}
	\label{fig:confusion1001}
\end{figure}

\section{Further comparison with LA}
\label{LACDMAD}

%Since we claimed that CDMAD can be thought as extension of LA to be aware for class distribution mismatch, 
Because CDMAD can be viewed as an extension of LA for incorporating awareness of class distribution mismatch,
we compared the classification performance of LA and CDMAD for CISSL under the settings that the class distributions of the labeled and unlabeled sets mismatch.
To use LA for CISSL, we refined pseudo-labels and class predictions on test samples by LA similar to CDMAD. Experimental results are presented in \cref{tableLACDMAD}. ReMixMatch+LA adjusts the logits on inputs by the log of the class distribution of the labeled set by assuming that the class distribution of the unlabeled set is the same as that of the labeled set. ReMixMatch+LA* adjusts the logits on inputs by the log of the class distribution of the whole training set by assuming that the class distribution of the unlabeled set is known, although it differs from that of the labeled set. From \cref{tableLACDMAD}, we can observe that ReMixMatch+CDMAD performed significantly better than both ReMixMatch+LA and ReMixMatch+LA*. This may be because CDMAD refined the biased pseudo-labels and class predictions on test samples more effectively than ReMixMatch+LA and ReMixMatch+LA* by incorporating awareness of class distribution mismatch. It should be noted that LA* cannot re-balance the classifier to an appropriate degree even if the class distribution of the unlabeled set is known under the class distribution mismatch setting. This may be because in SSL, each labeled data point is typically used more frequently and importantly than each unlabeled data point. Consequently, the classifier may become biased towards the class distribution of the labeled set to a greater degree than the class distribution of the entire training set, while still being affected by the class distribution of the unlabeled set.

\begin{table}[htbp]

  \caption{bACC/GM on CIFAR-10-LT under $\gamma_l\neq\gamma_u$.}
  \label{tableLACDMAD}
  \centering
  \resizebox{3.2in}{!}{%
  \begin{tabular}{ccccc}
    \toprule
    &&\multicolumn{3}{c}{CIFAR-10-LT ($\gamma_l=100$)}                \\
    \midrule
    \multicolumn{2}{c}{Algorithm}   &$\gamma_u=1$&$\gamma_u=50$&$\gamma_u=150$\\
    \midrule

    \multicolumn{2}{c}{ReMixMatch+LA}&$76.6$/ $66.8$&$69.9$/ $52.6$&$70.5$/ $42.7$\\
    \multicolumn{2}{c}{ReMixMatch+LA*}&$69.2$/ $54.0$&$73.7$/ $70.8$&$58.3$/ $27.4$\\
    \rowcolor{Gray}
    \multicolumn{2}{c}{ReMixMatch+CDMAD}&\textbf{89.9}/ \textbf{89.6}&\textbf{86.9}/ \textbf{86.7}&\textbf{83.1}/ \textbf{82.7}\\
    \bottomrule
  \end{tabular}}
  \vspace{-0.1 in}
\end{table}

\section{Fine grained experimental results}
\label{finegrained}
To verify that CDMAD improves classification performance for minority classes, we performed experiments using FixMatch/ ReMixMatch and FixMatch/ReMixMatch+CDMAD on CIFAR-10-LT and measured the accuracy for Many/Medium/Few groups separately (for CIFAR-10-LT, we set the first three classes as many shot groups, then next four classes as medium shot groups, and then last three classes as few shot groups). We also measured the fine grained classification performance of FixMatch/ReMixMatch+CoSSL \citep{fan2022CoSSL} on CIFAR-10-LT and compared them with those of CDMAD for comparison with a recent CISSL algorithm. The results are summarized in \cref{finegrained1}, \cref{finegrained2}, and \cref{finegrained3}. We can observe that FixMatch+CDMAD and ReMixMatch+CDMAD greatly improved accuracy for few shot groups with only slightly decreased accuracy for many shot groups compared to FixMatch and ReMixMatch. We can also observe that FixMatch/ ReMixMatch+CDMAD achieved better medium and few shot classification accuracies than FixMatch/ ReMixMatch+COSSL. These results demonstrate that CDMAD effectively relieves class imbalance.

\begin{table}[htbp]
%\vspace{-0.1 in}
  \caption{Fine grained experimental results under $\gamma_l=\gamma_u=100$.
  }
  \label{finegrained1}
  \centering
  \resizebox{3.2in}{!}{%
  \begin{tabular}{cccccc}
    \toprule
    \multicolumn{6}{c}{CIFAR-10-LT ($\gamma_l=\gamma_u=100$)}                   \\
    \midrule
    \multicolumn{2}{c}{Algorithm}   &Overall&Many&Medium&Few\\
    \midrule
    \multicolumn{2}{c}{FixMatch}&$72.5$&$95.0$&$74.6$&$47.3$\\
    \multicolumn{2}{c}{FixMatch+CDMAD}&$83.6$&$91.9$&$82.2$&$77.2$\\
    \midrule
    \multicolumn{2}{c}{ReMixMatch}&$74.3$&$96.7$&$77.8$&$47.2$\\
    \multicolumn{2}{c}{ReMixMatch+CDMAD}&$85.5$&$90.1$&$84.8$&$81.8$\\
    \bottomrule
  \end{tabular}}
%  \vspace{-0.1 in}
\end{table}
\begin{table}[htbp]
%\vspace{-0.1 in}
  \caption{Fine grained experimental results under $\gamma_l=100$, and $\gamma_u=1$.
  }
  \label{finegrained2}
  \centering
  \resizebox{3.2in}{!}{%
  \begin{tabular}{cccccc}
    \toprule
    \multicolumn{6}{c}{CIFAR-10-LT ($\gamma_l=100$, $\gamma_u=1$)}                   \\
    \midrule
    \multicolumn{2}{c}{Algorithm}   &Overall&Many&Medium&Few\\
    \midrule
    \multicolumn{2}{c}{FixMatch}&$70.2$&$96.3$&$77.7$&$34.0$\\
    \multicolumn{2}{c}{FixMatch+CDMAD}&$87.5$&$95.6$&$86.4$&$80.9$\\
    \midrule
    \multicolumn{2}{c}{ReMixMatch}&$65.4$&$96.6$&$70.8$&$27.0$\\
    \multicolumn{2}{c}{ReMixMatch+CDMAD}&$89.9$&$96.5$&$87.8$&$86.0$\\
    \bottomrule
  \end{tabular}}
%  \vspace{-0.1 in}
\end{table}
\begin{table}[htbp]
\vspace{-0.1 in}
  \caption{Fine grained experimental results under $\gamma_l=\gamma_u=100$.
  }
  \label{finegrained3}
  \centering
  \resizebox{3.2in}{!}{%
  \begin{tabular}{cccccc}
    \toprule
    \multicolumn{6}{c}{CIFAR-10-LT ($\gamma_l=\gamma_u=100$)}                   \\
    \midrule
    \multicolumn{2}{c}{Algorithm}   &Overall&Many&Medium&Few\\
    \midrule
    \multicolumn{2}{c}{FixMatch+CoSSL}&$83.2$&$93.4$&$81.1$&$75.8$\\
    \multicolumn{2}{c}{FixMatch+CDMAD}&$83.6$&$91.9$&$82.2$&$77.2$\\
    \midrule
    \multicolumn{2}{c}{ReMixMatch+CoSSL}&$84.1$&$91.7$&$82.1$&$79.1$\\
    \multicolumn{2}{c}{ReMixMatch+CDMAD}&$85.5$&$90.1$&$84.8$&$81.8$\\
    \bottomrule
  \end{tabular}}
%  \vspace{-0.1 in}
\end{table}

\section{Comparing CDMAD with DASO}
\label{furthercomparison}

\begin{table}[htbp]
%\vspace{-0.1 in}
  \caption{bACC/GM on CIFAR-10-LT under $\gamma=\gamma_{l}=\gamma_{u}$. 
  }
  \label{daso1}
  \centering
  \resizebox{3.2in}{!}{%
  \begin{tabular}{ccccc}
    \toprule
    \multicolumn{5}{c}{CIFAR-10-LT ($\gamma=\gamma_l=\gamma_u$)}                   \\
    \midrule
    \multicolumn{2}{c}{Algorithm}   &$\gamma=50$&$\gamma=100$&$\gamma=150$\\
    \midrule
    \multicolumn{2}{c}{FixMatch+DASO}&$81.8$/ $81.0$&$75.7$/ $74.0$&$72.0$/ $68.9$ \\
    \multicolumn{2}{c}{FixMatch+DASO+LA}&$84.1$/ $83.7$&$79.4$/ $78.8$&$76.5$/ $75.5$ \\
    \rowcolor{Gray}
    \multicolumn{2}{c}{FixMatch+CDMAD}&\textbf{87.3}/\textbf{87.0}&\textbf{83.6}/\textbf{83.1}&\textbf{80.8}/\textbf{79.9} \\
    \midrule
    \multicolumn{2}{c}{ReMixMatch+DASO}&$82.5$/ $81.9$&$76.0$/ $73.9$&$70.8$/ $66.5$ \\
    \multicolumn{2}{c}{ReMixMatch+DASO+LA}&$85.9$/ $85.7$&$82.8$/ $82.4$&$79.0$/ $78.4$ \\
    \rowcolor{Gray}
    \multicolumn{2}{c}{ReMixMatch+CDMAD}&\textbf{88.3}/ \textbf{88.1}&85.5/ \textbf{85.3}&\textbf{82.5}/ \textbf{82.0}\\
    \bottomrule
  \end{tabular}}
  \vspace{-0.1 in}
\end{table}

\begin{table*}[htbp]

  \caption{Comparison of bACC/GM on CIFAR-10-LT and STL-10-LT under $\gamma_l\neq\gamma_u$.}
  \label{daso2}
  \centering
  \resizebox{5in}{!}{%
  \begin{tabular}{ccccccc}
    \toprule
    &&\multicolumn{3}{c}{CIFAR-10-LT ($\gamma_l=100$)} &\multicolumn{2}{c}{STL-10-LT ($\gamma_u=$Unknown)}                  \\
    \midrule
    \multicolumn{2}{c}{Algorithm}   &$\gamma_u=1$&$\gamma_u=50$&$\gamma_u=150$&$\gamma_{l}=10$ &$\gamma_{l}=20$\\
    \midrule
    \multicolumn{2}{c}{FixMatch+DASO}&$86.4$/ $86.0$&$79.1$/ $78.2$&$74.2$/ $71.6$&$68.4$/ $65.3$&$62.1$/ $58.9$\\
    \multicolumn{2}{c}{FixMatch+DASO+LA}&$86.2$/ $85.8$&$81.7$/ $81.2$&$78.0$/ $77.0$&$68.9$/ $66.3$&$66.0$/ $64.6$\\
    
    \rowcolor{Gray}
    \multicolumn{2}{c}{FixMatch+CDMAD}&\textbf{87.5}/ \textbf{87.1}&\textbf{85.7}/ \textbf{85.3}&\textbf{82.3}/ \textbf{81.8}&\textbf{79.9}/ \textbf{78.9}&\textbf{75.2}/ \textbf{73.5}\\
    \midrule
    \multicolumn{2}{c}{ReMixMatch+DASO}&$89.6$/ $89.3$&$79.6$/ $77.8$&$72.3$/ $69.0$&$75.1$/ $73.6$&$66.8$/ $61.8$\\
    \multicolumn{2}{c}{ReMixMatch+DASO+LA}&$80.6$/ $77.7$&$84.8$/ $84.5$&$79.7$/ $79.2$&$78.1$/ $77.3$&$75.3$/ $74.0$\\
    
    \rowcolor{Gray}
    \multicolumn{2}{c}{ReMixMatch+CDMAD}&\textbf{89.9}/ \textbf{89.6}&\textbf{86.9}/ \textbf{86.7}&\textbf{83.1}/ \textbf{82.7}&\textbf{83.0}/\textbf{82.1}&\textbf{81.9}/\textbf{80.9}\\
    \bottomrule
  \end{tabular}}
  \vspace{-0.1 in}
\end{table*}

\begin{table}[htbp]
\vspace{-0.05 in}
  \caption{Comparison of bACC on CIFAR-100-LT. %Experiment results are copied from CoSSL \citep{fan2022CoSSL}.
  }
  \label{daso3}
  \centering
  \resizebox{3.2in}{!}{%
  \begin{tabular}{ccccc}
    \toprule
    \multicolumn{5}{c}{CIFAR-100-LT ($\gamma=\gamma_l=\gamma_u$)}                   \\
    \midrule
    \multicolumn{2}{c}{Algorithm}   &$\gamma=20$&$\gamma=50$&$\gamma=100$\\
    \midrule
    \multicolumn{2}{c}{FixMatch+DASO}&$45.8$&$39.2$&$33.9$\\
    \multicolumn{2}{c}{FixMatch+DASO+LA}&$46.2$&$39.9$&$34.5$ \\
    \rowcolor{Gray}
    \multicolumn{2}{c}{FixMatch+CDMAD}&\textbf{54.3}&\textbf{48.8}&\textbf{44.1}\\
    \midrule
    \multicolumn{2}{c}{ReMixMatch+DASO}&$51.5$&$43.0$&$38.2$\\
    \multicolumn{2}{c}{ReMixMatch+DASO+LA}&$52.8$&$45.5$&$40.3$ \\
    
    \rowcolor{Gray}
    \multicolumn{2}{c}{ReMixMatch+CDMAD}&\textbf{57.0}&\textbf{51.1}&\textbf{44.9} \\
    \bottomrule
  \end{tabular}}
  \vspace{-0.05 in}
\end{table}

Because classification performance of DASO were measured under slightly different settings from ours, it was difficult to fairly compare their classification performance with that of CDMAD in the main paper. Nevertheless, in the case of DASO, we conducted experiments in the same setting as ours using the official code in github. The classification performance of DASO is summarized in \cref{daso1}, \cref{daso2}, and \cref{daso3}. From \cref{daso1}, \cref{daso2}, and \cref{daso3}, we can observe that the proposed algorithm outperforms DASO. From \cref{daso2}, we can also observe that combining DASO with LA degrades the classification performance when the class distributions of the labeled and unlabeled sets severely differ. This may be because the LA considers only the class distribution of the labeled set when the class distribution of the unlabeled set is unknown. These results show the importance of re-balancing the classifier by considering the class distribution of the unlabeled set. These results demonstrate the effectiveness of CDMAD.

\end{document}